\documentclass{article}

\usepackage[utf8]{inputenc} % allow utf-8 input
\usepackage[T1]{fontenc}    % use 8-bit T1 fonts
\usepackage{hyperref}       % hyperlinks
\usepackage{url}            % simple URL typesetting
\usepackage{booktabs}       % professional-quality tables
\usepackage{amsfonts}       % blackboard math symbols
\usepackage{nicefrac}       % compact symbols for 1/2, etc.
\usepackage{microtype}      % microtypography
\usepackage{xcolor}         % colors
\usepackage{amsmath, amssymb, amsfonts, amsbsy, amsthm, latexsym, epsfig,bm}
\newcommand{\mathbbm}[1]{\text{\usefont{U}{bbm}{m}{n}#1}} % from mathbbm.sty
\DeclareMathOperator{\Tr}{tr}
\usepackage{MnSymbol}
\usepackage{hyperref}
\usepackage{xcolor}
\usepackage{algorithm} % Algorithm
\usepackage{algorithmic}
\usepackage{wrapfig}
\usepackage{bbding}
\usepackage{enumitem}
\usepackage{longtable}

    \PassOptionsToPackage{numbers, compress}{natbib}
    \usepackage[final]{neurips_2024}

\newtheorem{lemma}{Lemma}

\newtheorem{theorem}{Theorem}
\newtheorem{remark}{Remark}
\newtheorem{assumption}{Assumption}

\newtheorem{corollary}{Corollary}
\newtheorem{proposition}{Proposition}

\newcommand{\ip}[2]{\left\langle #1, #2\right\rangle }
\newcommand{\sumij}{\sum_{(i,j)\ne (i^*,j^*)}}
\newcommand{\lem}{\bar{\mathcal{L}}_{em}}

\newcommand{\KL}{\text{KL}}

\newcommand{\backt}{t^{\leftarrow}}
\newcommand{\F}{F}
\newcommand{\per}{{\rm per}}

\title{Evaluating the design space of diffusion-based generative models}

\author{%
  Yuqing Wang \\ Simons Institute\\ University of California, Berkeley\\ \texttt{yq.wang@berkeley.edu} \And Ye He \\ School of Mathematics\\ Georgia Institute of Technology\\ \texttt{yhe367@gatech.edu} \And Molei Tao \\School of Mathematics\\ Georgia Institute of Technology\\ \texttt{mtao@gatech.edu}
}

\begin{document}

\maketitle

\begin{abstract}
Most existing theoretical investigations of the accuracy of diffusion models, albeit significant, assume the score function has been approximated to a certain accuracy, and then use this a priori bound to control the error of generation. This article instead provides a first quantitative understanding of the whole generation process, i.e., both training and sampling. More precisely, it conducts a non-asymptotic convergence analysis of denoising score matching under gradient descent. In addition, a refined sampling error analysis for variance exploding models is also provided. The combination of these two results yields a full error analysis, which elucidates (again, but this time theoretically) how to design the training and sampling processes for effective generation. For instance, our theory implies a preference toward noise distribution and loss weighting in training that qualitatively agree with the ones used in \citet{karras2022elucidating}. It also provides perspectives on the choices of time and variance schedules in sampling: when the score is well trained, the design in \citet{song2020score} is more preferable, but when it is less trained, the design in \citet{karras2022elucidating} becomes more preferable.
\end{abstract}

\vskip -0.2in
\section{Introduction}
\vskip -0.02in

Diffusion models became a very popular generative modeling approach in various domains, including computer vision \citep{dhariwal2021diffusion,baranchuk2022labelefficient,ho2022cascaded,ho2022video,meng2022sdedit,wu2023medsegdiff}, natural language processing \citep{austin2021structured,li2022diffusion,lou2023discrete}, various modeling tasks \citep{chen2021wavegrad,ramesh2022hierarchical,yoon2021adversarial}, and medical, biological, chemical and physical applications \citep{anand2022protein,chung2022score,schneuing2022structure,watson2023novo,duan2023accurate,zhu2024quantum} (see more surveys in \citep{yang2023diffusion,cao2024survey,chen2024overview}). 
\citet{karras2022elucidating} provided a unified empirical understanding of the derivations of model parameters, leading to 
new state-of-the-art performance. \citet{karras2023analyzing} further upgraded the model design by revamping the network architectures and replacing the weights of the network with an exponential moving average. As diffusion models gain wider usage, efforts to understand and enhance their generation capability 
become increasingly meaningful.

In fact, a rapidly increasing number of theoretical works already analyzed various aspects of diffusion models \citep{lee2022convergence, de2022convergence,yang2022convergence, chen2022sampling,chen2023improved,benton2024nearly,conforti2023score,block2020generative,chen2023score,oko2023diffusion,shah2023learning,han2024neural}.  
Among them, a majority \citep{lee2022convergence, de2022convergence,yang2022convergence, chen2022sampling,chen2023improved,benton2024nearly,conforti2023score} focus on sampling/inference; more precisely, they assume the score error is within a certain accuracy threshold (i.e. the score function is well trained in some sense), and analyze the discrepancy between the distribution of the generated samples and the true one. 
Meanwhile, there are a handful of results \citep{shah2023learning,han2024neural} that aim at understanding different facets of the training process.  
See more detailed discussions of existing theoretical works in Section~\ref{subsec:related_works}.

However, as indicated in \citet{karras2022elucidating}, the performance of diffusion models also relies on the interaction between design components in both training and sampling, such as the noise distribution, weighting, time and variance schedules, etc. 
While focusing individually on either the training or generation process provides valuable insights, a holistic quantification of the actual generation capability can only be obtained when both processes are considered altogether. 
Therefore, motivated by obtaining {\it deeper theoretical understanding of how to maximize the performance of diffusion models}, this paper aims at establishing a full generation error analysis, combining both the optimization and sampling processes, to partially investigate the design space of diffusion models.

More precisely, we focus on the variance exploding setting~\citep{song2020score}, which is also the foundation of continuous forward dynamics in~\citet{karras2022elucidating}. Our main contributions are summarized as follows:
\begin{itemize}[topsep=0pt,parsep=0pt,partopsep=0pt, align=left,leftmargin=*,widest={3}]
\vspace{-0.05in}
    \item  For denoising score matching objective, we establish the exponential convergence of its gradient descent training dynamics (Theorem~\ref{thm:convergence_GD_no_interpolation}). We  develop a new method for proving a key lower bound of gradient under the semi-smoothness framework~\citep{allen2019convergence,li2018learning,zou2019improved,zou2020gradient}. 
    \item We extend the sampling error analysis in \citep{benton2024nearly} to the variance exploding case (Theorem~\ref{thm:VESDE_generation}), under only the finite second moment assumption (Assumption~\ref{assump:finite_second_moment}) of the data distribution. Our result applies to various variance and time schedules, and implies a sharp almost linear complexity in terms of data dimension under optimal time schedule. 
    \item We conduct a full error analysis of diffusion models, combining training and sampling (Theorem~\ref{cor:full_error_analysis}).
    \item We qualitatively derive the theory for choosing the noise distribution and weighting in the training objective, which coincides with \citet{karras2022elucidating} (Section~\ref{subsec:weighting}). More precisely, our theory implies that the optimal rate is obtained when the total weighting exhibits a similar ``bell-shaped'' pattern used in~\citet{karras2022elucidating}.
    \item We develop a theory of choosing time and variance schedules based on both training and sampling  (Section~\ref{subsec:time variance schedule}). Indeed, when the score error dominates, i.e., the neural network is less trained and not very close to the true score, polynomial schedule~\citep{karras2022elucidating} ensures smaller error; when sampling error dominates, i.e., the score function is well approximated, exponential schedule~\citep{song2020score} is preferred.                                             
\end{itemize}
\begin{wrapfigure}{r}{0.5\textwidth}
\vskip -0.15in
    \centering
    \includegraphics[width=0.55\textwidth]{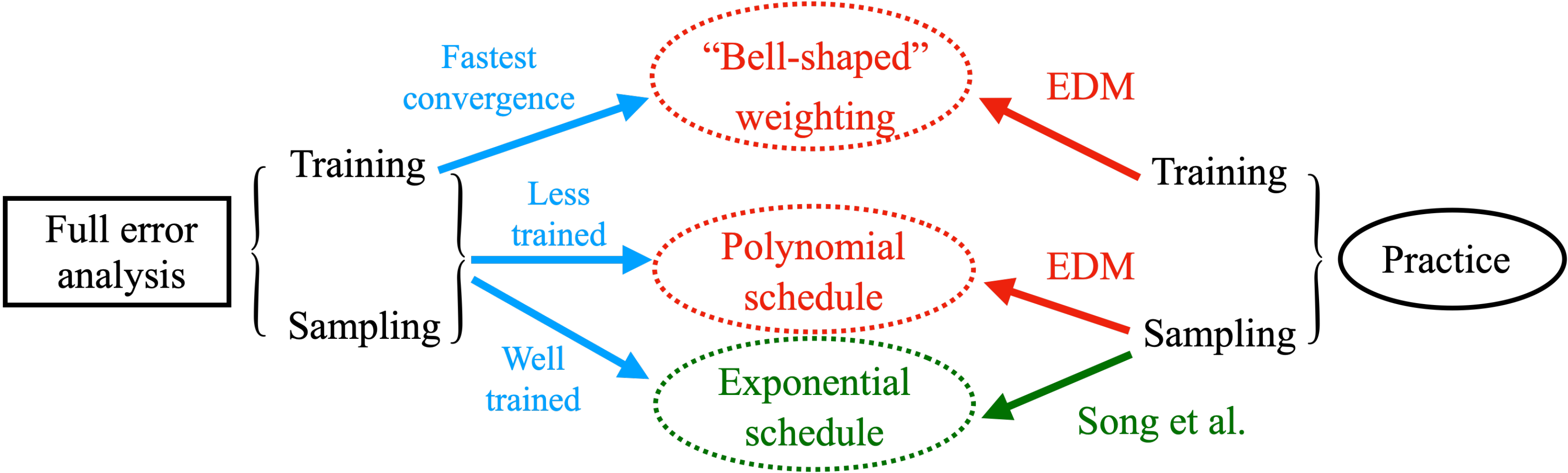}    
    \vspace{-0.25in}
    \caption{Structure of this paper.}
    \label{fig:paper_structure}
    \vspace{-0.15in}
\end{wrapfigure}

Conclusions and limitations are in Appendix~\ref{app:conclusions_limitations}.

\vspace{-0.1in}

\subsection{Related works}
\label{subsec:related_works}
\vspace{-0.1in}
\textbf{Sampling.} There has been significant progress in quantifying the sampling error of the generation process of diffusion models, assuming the score function is already approximated within certain accuracy. Most existing works \citep[e.g.,][]{chen2022sampling,chen2023improved,benton2024nearly} focused on the variance preserving (VP) SDEs, whose discretizations correspond to DDPM. For example, \citet{benton2024nearly} is one of the latest results for the VPSDE-based diffusion models, and it only needs a very mild assumption: the data distribution has finite second moment. The iteration complexity is shown to be almost linear in the data dimension and polynomial in the inverse accuracy, under exponential time schedule. However, a limited amount of works~\citep{lee2022convergence,gao2024convergence,yangconvergence}  analyzed the variance exploding (VE) SDEs, whose discretizations correspond to Score Matching with Langevin dynamics (SMLD)~\citep{song2019generative,song2020score}. To our best knowledge, \citet{yangconvergence} obtained the best result so far for VE assuming the data distribution has bounded support: the iteration complexity is polynomial in the data dimension and the inverse accuracy, under the uniform time schedule. In contrast, our work only assumed that the data distribution has finite second moment, and by extending the stochastic localization approach in \citep{benton2024nearly} to VESDE, we obtain an iteration complexity that is polynomial in the data dimension and the inverse accuracy, under more general time schedules as well. Note the improved complexity in terms of the inverse accuracy and the data dimension dependencies; in fact, under the exponential time schedule, our complexity is almost linear in the data dimension, which recovers the state-of-the-art result for VPSDE-based diffusion models. 

\textbf{Training.} To our best knowledge, the only works that quantify the training process of the diffusion models are \citet{shah2023learning} and~\citet{han2024neural}. \citet{shah2023learning} employed the DDPM formulation and considered data distributions as mixtures of two spherical Gaussians with various scales of separation, together with $K$ spherical Gaussians with a warm start. 
Then the score function can be analytically solved, and they modeled it in a teacher-student framework solved by gradient descent. They also provided the sample complexity bound under these specific settings. In contrast, our results work for general data distributions for which the true score is unknown, and training analysis is combined with sampling analysis. \citet{han2024neural} considered the GD training of a two-layer ReLU neural network with the last layer fixed, and used the neural tangent kernel (NTK) approach to establish a first result on generalization error. 
They uniformly sampled the time points in the training objective, assumed that the Gram matrix of the kernel is away from 0 (implying a lower bound on the gradient), and lacked a detailed non-asymptotic characterization of the training process. In contrast, we use the deep ReLU network with $L$ layers  trained by GD and prove instead of assuming that the gradient is lower bounded by the objective function.
Moreover, we obtain a non-asymptotic 
bound for the optimization error, and our bound is valid for general time and variance schedules, which allows us to obtain a full error analysis.

\vskip -0.05in

\textbf{Convergence of neural networks training.} The convergence analysis of neural networks under gradient descent has been a longstanding challenge and has been developed into an extensive field. Here we will only focus on results mostly related to the techniques used in this paper. 
One line of them is approaches directly based on neural tangent kernel (NTK) \citep{du2018gradient,du2019gradient,arora2019fine,song2019quadratic,liu2022provable}. 
However, existing works in this direction focus more on either scalar output, or vector output but with only one layer trained under two-layer networks, 
which is insufficient for diffusion models. Another line of research also considers overparameterized models in a regime analogous to NTK, though not necessarily explicitly resorting to kernels. Instead, it directly quantifies the lower bound of the gradient~\citep{allen2019convergence,li2018learning,allen2019convergencerecurr,zou2019improved,zou2020gradient} and uses a semi-smoothness property to prove exponential convergence. 
Our results align with the latter line, but we develop a new method for proving the lower bound of the gradient and adopt assumptions that are closer to the setting of diffusion models. See more discussions in Section~\ref{subsec:training}.

\vspace{-0.1in}

\subsection{Notations} 
\vspace{-0.1in}

% total weighting; training weighting function; loss weighting; network weights

We denote $\|\cdot\|$ to be the $\ell^2$ norm for both vectors and matrices, and $\|\cdot\|_F$ to be the Frobenius norm. For the discrete time points, we use $t_i$ to denote the time point for forward dynamics and $\backt_i$ for backward dynamics. For the order of terms, we follow the theoretical computer science convention to use $\mathcal{O}(\cdot),\Theta(\cdot),\Omega(\cdot)$. We also denote $f \lesssim g$ if $f\le C g$ for some universal constant $C$.

\vspace{-0.1in}

\section{Basics of diffusion-based generative models}
\label{sec:basics}
\vspace{-0.1in}

In this section, we will introduce the basic forward and backward dynamics of diffusion models, as well as the denoising score matching setting under which a model is trained.

\vspace{-0.1in}

\subsection{Forward and backward processes}
Consider a forward diffusion process that pushes an initial distribution $P_0$ to Gaussian
\vspace{-0.1in}
\begin{align}
\label{eqn:general_forward_SDE}
    dX_t=-f_t\, X_t\,dt+\sqrt{2\sigma_t^2}\, dW_t,
\end{align}
\vskip -0.1in
where $dW_t$ is the Brownian motion, $X_t$ is a $d$-dim. random variable, and $X_t\sim P_t$. Under mild assumptions, the process can be reversed and the backward process is defined as follows
\vskip -0.2in
\begin{align}
\label{eqn:general_backward_SDE}
    dY_t=(f_{T-t}\,Y_t+2\sigma_{T-t}^2\nabla\log p_{T-t}(Y_t))\,dt+\sqrt{2\sigma_{T-t}^2}\,d\tilde{W}_t,
\end{align} 
\vskip -0.1in
where $Y_0\sim P_T$, and $p_t$ is the density of $P_t$. Then $Y_{T-t}$ and $X_t$ have the same distribution with density $p_t$ \citep{anderson1982reverse}, which means the dynamics~\eqref{eqn:general_backward_SDE} will push (near) Gaussian distribution back to (nearly) the initial distribution $P_0$. To apply the backward dynamics for generative modeling, the main challenge lies in approximating the term $\nabla\log p_{T-t}(Y_t)$ which is called {\it score function}. It is common to use a neural network to approximate this score function and learn it via the forward dynamics~\eqref{eqn:general_forward_SDE}; then, samples can be generated by simulating the backward dynamics~\eqref{eqn:general_backward_SDE}.

\vspace{-0.1in}

\subsection{The training of score function via denoising score matching}
\vspace{-0.1in}

In order to learn the score function, a natural starting point is to consider the following score matching objective \citep[e.g.,][]{hyvarinen2005estimation}
\vspace{-0.15in}
\begin{align}
\label{fct:L_conti}
    \mathcal{L}_{\rm conti}(\theta)&=\frac{1}{2}\int_{t_0}^T w(t)\mathbb{E}_{X_t\sim P_t}\|S(\theta;t,X_t)-\nabla_x \log p_t(X_t)\|^2\,dt
\end{align}
where $S(\theta;t,X_t)$ is a $\theta$-parametrized neural network, $w(t)$ is some {\it weighting function}, and the subscript means this is the continuous setup. Ideally one would like to minimize this objective function to obtain $\theta$; however, $p_t$ in general is unknown, and so is the true score function $\nabla_x \log p_t$. One of the solutions is denoising score matching proposed by \citet{vincent2011connection}, where one, instead of directly matching the true score, leverages conditional score for which initial condition is fixed so that $p_{t|0}$ is analytically known. 

More precisely, given the linearity of forward dynamics \eqref{eqn:general_forward_SDE}, its exact solution is explicitly known: Let $\mu_t=\int_0^t f_s\,ds$, and $\bar{\sigma}_t^2=2\int_0^t e^{2\mu_s-2\mu_t}\sigma_s^2\,ds$. Then the solution is $X_t=e^{-\mu_t}X_0+\bar{\sigma}_t\xi,$
where $\xi\sim\mathcal{N}\left(0, I\right)$. We also have $X_t|X_0\sim\mathcal{N}\left(e^{-\mu_t}X_0,\bar{\sigma}_t^2I\right)$ and $g_t(x|y)=(2\pi\bar{\sigma}_t^2)^{-d/2}\exp(-\|x-e^{-\mu_t}y\|^2/(2\bar{\sigma}_t^2))$, which is the density of $X_t|X_0$. Then the objective can be rewritten as 
\vskip -0.25in
\begin{align}
\label{fct:conti_denoising}
    \mathcal{L}_{\rm conti}(\theta)&=\frac{1}{2}\int_{t_0}^T w(t)\mathbb{E}_{X_0}\mathbb{E}_{X_t|X_0}\| S(\theta;t,X_t)-\nabla\log g_t(X_t|X_0)\|^2dt+\frac{1}{2}\int_{t_0}^T w(t){C_t}dt\notag\\
    &=\frac{1}{2}\int_{t_0}^T w(t)\frac{1}{\bar{\sigma}_t}\mathbb{E}_{X_0}\mathbb{E}_{\xi}\|\bar{\sigma}_t S(\theta;t,X_t)+\xi\|^2dt+\frac{1}{2}\int_{t_0}^T w(t){C_t}dt
\end{align}
where $C_t=\mathbb{E}_{X_t}\|\nabla\log p_t\|^2-\mathbb{E}_{X_0}\mathbb{E}_{X_t|X_0}\|\nabla\log g_t(X_t|X_0)\|^2$. For completeness, we will provide a detailed derivation of these results in Appendix~\ref{app:denoising_score_matching_objective} and emphasize that it is just a review of existing results in our notation.

Throughout this paper, we adopt the variance exploding (VESDE) setting \citep{song2020score}, where $f_t=0$ and hence $\mu_t=0$, which also aligns with the setup in the classic of EDM~\citep{karras2022elucidating}.

\vspace{-0.1in}

\section{Error analysis for diffusion-based generative models}
\label{sec:error_analysis}
\vspace{-0.1in}

In this section, we will quantify both the training and sampling processes, and then integrate them into a more comprehensive generation error analysis.

\vspace{-0.1in}

\subsection{Training}
\vspace{-0.1in}

\label{subsec:training}
In this section, we consider a practical implementation of denoising score matching objective, represent the score by a deep ReLU network, and establish the exponential convergence of GD training dynamics.

\textbf{Training objective function. } Consider a quadrature discretization of the time integral in \eqref{fct:conti_denoising} 
based on deterministic\footnote{Otherwise it is no longer GD training but stochastic GD.} collocation points $0<t_0<t_1<t_2<\cdots<t_N=T$. Then
\vskip -0.25in
\begin{align}
\label{eqn:barL+barC}
    \mathcal{L}_{\rm conti}(\theta)\approx \bar{\mathcal{L}}(\theta)+\bar{C}, 
\end{align}
\vskip -0.2in
where $\bar{C}=\frac{1}{2}\sum_{j=1}^{N}w(t_j)(t_j-t_{j-1})C_{t_j}$, and 
\vskip -0.25in
\begin{align}
\label{fct:bar_L}
     \bar{\mathcal{L}}(\theta)=\frac{1}{2}\sum_{j=1}^{N}w(t_j)(t_j-t_{j-1})\frac{1}{\bar{\sigma}_{t_j}}\mathbb{E}_{X_0}\mathbb{E}_{\xi}\|\bar{\sigma}_{t_j} S(\theta;t_j,X_{t_j})+\xi\|^2.
\end{align}
\vspace{-0.1in}

Define $\beta_j=w(t_j)(t_j-t_{j-1})\frac{1}{\bar{\sigma}_{t_j}}$ to be the {\it total weighting}. Consider the empirical version of $\bar{\mathcal{L}}$~\eqref{fct:bar_L} . Denote the initial data to be $\{x_i\}_{i=1}^n$ with $x_i\sim P_0$, and the noise to be $\{\xi_{ij}\}_{j=1}^{N}$ with $\xi_{ij}\sim\mathcal{N}(0,I_d)$. Then the input data of the neural network is $\{t_j, X_{ij}\}_{i=1,j=1}^{n,N}=\{t_j, x_i+\bar{\sigma}_{t_j}\xi_{ij}\}_{i=1,j=1}^{n,N}$ and the output data is $\{\xi_{ij}/\bar{\sigma}_{t_j}\}_{i=1,j=1}^{n,N}$ if $\bar{\sigma}_{t_j}\ne 0$. Consequently, $\bar{\mathcal{L}}(\theta)$~\eqref{fct:bar_L} can be approximated by the following
\vskip -0.25in
\begin{align}
\label{fct:lem}
    \lem(\theta)=\frac{1}{2n}\sum_{i=1}^n\sum_{j=1}^{N}\beta_j\|\bar{\sigma}_{t_j} S(\theta;t_j,x_i+\bar{\sigma}_{t_j}\xi_{ij})+\xi_{ij}\|^2.
\end{align}
We will use~\eqref{fct:lem} as the training objective function in our analysis. For simplicity, we also denote $f(\theta;i,j)=\beta_j\|\bar{\sigma}_{t_j} S(\theta;t_j,x_i+\bar{\sigma}_{t_j}\xi_{ij})+\xi_{ij}\|^2$ and then $\lem(\theta)=\frac{1}{2n}\sum_{i=1}^n\sum_{j=1}^{N}f(\theta;i,j)$. Note the time dependence can be absorbed into the $X$ dependence. More precisely, because $\bar{\sigma}_{t}$ is a monotonically increasing function of $t$, we can replace $t_j$ in the inputs by $\bar{\sigma}_{t_j}$ to indicate the time dependence. This is then equivalent to augmenting $X_{ij}$ to be $d+1$ dimensional with $(x_i)_{d+1}:=0$ and $(\xi_{ij})_{d+1}:=1$. For simplified presentation, we will slightly abuse notation and still use $d$ as the input dimension rather than $d+1$.

\textbf{Architecture.} The analysis of diffusion model training is in general very challenging. One obvious factor is the complex score parameterizations used in practice such as U-Net \citep{ronneberger2015u} and transformers \citep{peebles2023scalable,li2022diffusion}. In this paper, we simplify the architecture and consider deep feedforward networks. 
Although it is still far from practical usage, note this simple structure can already provide insights about the design space, as shown in later sections, and is more complicated than existing works~\citep{han2024neural,shah2023learning} related to the training of diffusion models (see Section~\ref{subsec:related_works}). More precisely, we consider the standard deep ReLU network with bias absorbed:
\vskip -0.3in
\begin{align}
\label{fct:neural network}
S(\theta;t_j,X_{ij})=W_{L+1}\sigma(W_L\cdots W_1\sigma(W_0X_{ij})),
\end{align}
where $\theta=(W_0,\cdots,W_{L+1})$, $W_0\in\mathbb{R}^{m\times d},W_{L+1}\in\mathbb{R}^{d\times m}$, $W_\ell\in\mathbb{R}^{m\times m}$ for $\ell=1,\cdots,L$, and $\sigma(\cdot)$ is the ReLU activation. 
\vskip -0.05in
\textbf{Algorithm.} Let $\theta^{(k)}=(W_0,W_1^{(k)},\cdots,W_L^{(k)},W_{L+1})$. We consider the gradient descent (GD) algorithm as follows
\vskip -0.25in
\begin{align}
    \theta^{(k+1)}=\theta^{(k)}-h\nabla \lem(\theta^{(k)}),
    \label{eq:agjvdouiy2g314bori4uybgo1afds}
\end{align}
where $h>0$ is the learning rate. We fix $W_0$ and $W_{L+1}$ throughout the training process and only update $W_1,\cdots,W_L$, which is a commonly used setting in the convergence analysis of neural networks~\citep{allen2019convergence,cai2019neural,han2024neural}.
\vskip -0.05in
\textbf{Initialization.} We employ the same initialization as in~\citet{allen2019convergence}, which is to set $(W_\ell^{(0)})_{ij}\sim\mathcal{N}(0,\frac{2}{m})$ for $\ell=0,\cdots,L$, $i,j=1,\cdots,m$, and $(W_{L+1}^{(0)})_{ij}\sim\mathcal{N}(0,\frac{1}{d})$ for $i=1\cdots,d$, $j=1\cdots,m$.

For this setup, the main challenge in our convergence analysis for denoising score matching lies in the nature of the data. 1) The output data that neural network tries to match 
is an unbounded Gaussian random vector, and cannot be rescaled as assumed in many theoretical works (for example, \citet{allen2019convergence} assumed the output data to be of order $o(1)$). 2) The input data $X_{ij}$ is the sum of two parts: $x_i$ which follows the initial distribution $P_0$, and a Gaussian noise $\bar{\sigma}_{t_j}\xi_{ij}$. Therefore, any assumption on the input data needs to agree with this noisy and unbounded nature, and commonly used assumptions like data separability~\citep{allen2019convergence,li2018learning} can no longer be used.

To deal with the above issues, we instead make the following assumptions.
\begin{assumption}[On network hyperparameters and initial data of the forward dynamics]
\label{assump:network_hyperparameters}We assume the following holds:
    \begin{enumerate}[topsep=0pt,parsep=0pt,partopsep=0pt, align=left,leftmargin=*,widest={3}]
    \vspace{-0.05in}
        \item Data scaling: $\|x_i\|=\Theta(d^{1/2})$ for all $i$. 
        \item Input dimension: $d=\Omega({\text{poly}}(\log (nN)))$.
    \end{enumerate}
\end{assumption}
\vskip -0.05in
We remark that the first assumption focuses only on the initial data $x_i$ instead of the whole solution of the forward dynamics $X_{ij}$ which incorporates the Gaussian noise. Also, this assumption is indeed not far away from reality; 
for example, it holds with high (at least $1-\mathcal{O}(\exp(-\Omega(d))$) probability for standard Gaussian random vectors.  The requirement for input dimension $d$ is to ensure that $d$ is not too small, or equivalently 
the number of data points is not exponential in $d$.

We also make the following assumptions on the hyperparameters of the denoising score matching.
\vskip -0.2in
\begin{assumption}[On the design of diffusion models]
\vskip -0.05in
\label{assump:normalize_  weighting}
We assume the following holds:
\begin{enumerate}[topsep=0pt,parsep=0pt,partopsep=0pt, align=left,leftmargin=*,widest={3}]\vspace{-0.02in}

    \item Weighting: $\sum_{j=1}^N w(t_j)(t_{j}-t_{j-1})\bar{\sigma}_{t_j}= \mathcal{O}(N)$. 
    \item Variance: $\bar{\sigma}_{t_0}>0$ and $\bar{\sigma}_{t_N}=\Theta(1)$.
\end{enumerate}
\end{assumption}
\vskip -0.1in
The first assumption is to guarantee that the weighting function $w(t)$ is properly scaled. This expression $w(t_j)(t_{j}-t_{j-1})\bar{\sigma}_{t_j}$ is obtained from proving the upper and lower bounds of the gradient of~\eqref{fct:lem}, and is different from the total weighting $\beta_i$ defined above. In the second assumption, $\bar{\sigma}_{t_0}>0$ ensures the output $\xi_{ij}/\bar{\sigma}_{t_j}$ is well-defined. The $\bar{\sigma}_{t_N}=\Theta(1)$ guarantees that the scales of the noise $\bar{\sigma}_{t_j}\xi_{ij}$ and the initial data $x_i$ are of the same order at the end of the forward process, namely, the initial data $x_i$ is eventually push-forwarded to near Gaussian with the proper variance. Therefore, Assumption~\ref{assump:normalize_  weighting} aligns with what has been used in practice (see Section~\ref{sec:design_space_explore} and \citet{karras2022elucidating,song2020score} for examples).

The following theorem summarizes our convergence result for the training of the score function. 

\begin{theorem}[Convergence of GD]
\label{thm:convergence_GD_no_interpolation}
 Define a set of indices to be $\mathcal{G}^{(s)}=\{(i,j)|f(\theta^{(s)};i,j)\ge f(\theta^{(s)};i',j')\text{ for all }i',j'\}$.
 Then given Assumption~\ref{assump:network_hyperparameters} and \ref{assump:normalize_  weighting}, for any $\epsilon_{\rm train}>0$, there exists some $M(\epsilon_{\rm train})=\Omega\left(\text{poly}\big(n,N,d,L,T/t_0,\log(\frac{1}{\epsilon_{\rm train}})\big)\right)$, s.t., when $m\ge M(\epsilon_{\rm train})$,  $h =\Theta(\frac{nN}{m\min_j w(t_j)(t_j-t_{j-1})\bar{\sigma}_{t_j}})$, and $k=\mathcal{O}(d^{\frac{1-a_0}{2}}n^2N\log(\frac{d}{\epsilon_{\rm train}}))$, with probability at least $1-\mathcal{O}(nN)\exp(-\Omega(d^{2a_0-1}))$, we have
 \vskip -0.15in
    \begin{align*}
         &\lem(\theta^{(k)})\le\prod_{s=0}^{k-1}\left(1-C_5 h \ w(t_{j^*(s)})(t_{j^*(s)}-t_{j^*(s)-1})\bar{\sigma}_{t_{j^*(s)}} \left(\frac{md^{\frac{a_0-1}{2}}}{n^3N^2}\right)\right)\lem(\theta^{(0)})
        %  \\
        % &\le \left(1-h \ w(t_{j^*(s^*)})(t_{j^*(s^*)+1}-t_{j^*(s^*)})\bar{\sigma}_{t_{j^*(s^*)}} \cdot\Omega\left(\frac{m\sqrt{\log d}d^{2a-1.5-4c}}{n^2N}\right)\right)^k\cdot(\lem(W^{(0)})-\lem^*)+\mathcal{O}(\frac{1}{nN})
    \end{align*}
             \vskip -0.1in
    where the universal constant $C_5>0$, $a_0\in(\frac{1}{2},1)$,  and $(i^*(s),j^*(s))=\arg\max_{(i,j)\in\mathcal{G}^{(s)}}w(t_{j})(t_j-t_{j-1})\bar{\sigma}_{t_{j}}$. Moreover, when $K=\Theta(d^{\frac{1-a_0}{2}}n^2N\log(\frac{d}{\epsilon_{\rm train}}))$, 
    \begin{align*}
        \lem(\theta^{(K)})\le \epsilon_{\rm{train}}.
    \end{align*}
\end{theorem}
The above theorem implies that for denoising score matching objective $\lem(\theta)$, GD has exponential convergence. For example, if we simply take $j^*=\min_jw(t_{j})(t_j-t_{j-1})\bar{\sigma}_{t_{j}}$, then $\lem(\theta^{(k+1)})$ is further upper bounded by $\left(1-C_6 h \ w(t_{j^*})(t_{j^*}-t_{j^*-1})\bar{\sigma}_{t_{j^*}} \left(\frac{md^{\frac{a_0-1}{2}}}{n^3N^2}\right)\right)^{k+1} \lem(\theta^{(0)})$.
The rate of convergence can be interpreted in the following way: 1) at the $k$th iteration, we collect all the indices of the time points into $\mathcal{G}^{(k)}$ where $f(\theta; i,j)$ has the maximum value; 2) we then choose the maximum of $w(t_j)(t_{j}-t_{j-1})\bar{\sigma}_{t_j}$ among all such indices and denote the index to be $j^*(k)$, and obtain the decay ratio bound  for the next iteration as $1-C_6 h \ w(t_{j^*(k)})(t_{j^*(k)}-t_{j^*(k)-1})\bar{\sigma}_{t_{j^*(k)}} \left(\frac{md^{\frac{a_0-1}{2}}}{n^3N^2}\right)$.

\vspace{-0.05in}

\begin{remark}[Can $\epsilon_{\rm train}$ be arbitrarily small? Some ramifications of the denoising setting]
\label{rmk:epsilon_train}
    Let us first see some facts about $\lem$ and $\bar{\mathcal{L}}$. Under minimal assumption of the existence of score function and in the zero-time-discretization-error limit, the score matching objective can be made zero and therefore the denoising score matching objective is bounded below by $-\bar{C}$, which is nonnegative and zero only when the data distribution is extremely special (we thus write $-\bar{C}>0$ from hereon unless confusion arises). That is, $\min_\theta\bar{\mathcal{L}}(\theta)\ge\min_{\text{any function }S}\bar{\mathcal{L}}=-\bar{C} >0$     according to \eqref{fct:conti_denoising}. Since $\lem\to\bar{\mathcal{L}}$ as the sample size of the training data set $n\to\infty$, we have ${\lem}\ge -\bar{C}-c_n > 0$ for some constant $c_n>0$ and $c_n\to 0$ as $n\to\infty$.

    However, Theorem~\ref{thm:convergence_GD_no_interpolation} seems to imply $\lem(\theta^{(k)})\to 0$ as $k\to\infty$ since $\lem(\theta^{(K)})\le \epsilon_{\rm train}$ and $\epsilon_{\rm train}$ is arbitrary, and it seems to contradict the $-\bar{C} > 0$ lower bound. However, there is no contradiction due to the combination of two facts. 
    First, the theorem states that for arbitrary $\epsilon_{\rm train}>0$, there exists a critical size, such that for overparameterized network beyond this size, GD can render the loss $\lem(\theta)$ eventually no greater than $\epsilon_{\rm train}$. If we fix the network size, i.e., with $m,L,d$ given, then $K$ is given, and Theorem~\ref{thm:convergence_GD_no_interpolation} says nothing about GD's behavior after $K$ iterations. %Namely, we do not know $\lim\sup_{k\to\infty}{\theta^{(k)}}$ and $\lim\inf_{k\to\infty}{\theta^{(k)}}$; 
    That is, we do not know whether $\lim\sup_{k\to \infty}\lem(\theta^{(k)})=0$. 
    Second, our optimization setting requires the sample size $n$ to be smaller than the network width $m$ (Assumption~\ref{assump:network_hyperparameters}). Thus, when $m$ is fixed, the sample size $n$ is upper bounded. 

The above discussion implies, within the validity of our theory, for any fixed network width $m$, if $\epsilon_{\rm train}$ is small, the sample size $n$ cannot be too large, meaning $\lem(\theta)-\bar{\mathcal{L}}(\theta)$ may not be small. Therefore, we can simultaneously have $\lem(\theta)$ close to 0 and $\bar{\mathcal{L}}(\theta)$ close to $-\bar{C}>0$.

\end{remark}

\textbf{Main technical steps for proving Theorem~\ref{thm:convergence_GD_no_interpolation}.} The proof of Theorem~\ref{thm:convergence_GD_no_interpolation} is in Appendix~\ref{app:training_proof}, where the analysis framework is adapted from~\citet{allen2019convergence}. Roughly speaking, the key proof in this framework is to establish the lower bound of the gradient. Then by integrating it into the semi-smoothness property of the neural network, we can obtain the exponential rate of convergence of gradient descent. For the lower bound of gradient, we develop a new method to deal with the difficulties in the denoising score matching setting (see the discussions earlier in this section). 
% The lemma is shown in the following.

% \begin{lemma}
% \label{lem:lower_bound_mainbody}
%     With probability $1-\mathcal{O}(nN)\exp(-\Omega(d^{2a_0-1}))$, we have
%     \begin{align*}
%         \|\nabla\lem({\theta^{(0)}})\|^2\ge C_6\left(\frac{md^{\frac{a_0-1}{2}}}{n^3N^2} \ w(t_{j^*})(t_{j^*}-t_{j^*-1})\bar{\sigma}_{t_{j^*}}\right)\lem(\theta^{(0)})
%     \end{align*}
%     where $(i^*,j^*)=\arg\max \left\|\sqrt{\frac{w(t_j)(t_j-t_{j-1})}{\bar{\sigma}_{t_j}}}(\bar{\sigma}_{t_j}W_{L+1}q_{ij,L}+\xi_{ij})\right\|$, $\frac{1}{2}<a_0<1$, and $C_6>0$ is some universal constant.
% \end{lemma}
Our new proof technique adopts a different decoupling of the gradient
and leverages a high probability bound based on a high-dimensional geometric idea. See Appendix~\ref{app:subsec:lower_bound_grad} for a proof sketch and more details.

\vspace{-0.1in}

\subsection{Sampling}
\label{subsec:sampling}
\vspace{-0.1in}

In this section, we establish a nonasymptotic error bound of the backward process in the variance exploding setting, which is an extension to \citet{benton2024nearly}. For simplified notations, denote the backward time schedule as $\{\backt_j\}_{0\le j\le N}$ such that $0=\backt_0<\backt_1<\cdots<\backt_N=T-\delta$.

         \vskip -0.05in
\textbf{Generation algorithm.} We consider the exponential integrator scheme for simulating the backward SDE \eqref{eqn:general_backward_SDE} with $f_t\equiv 0$\footnote{The exponential integrator scheme is degenerate since $f_t\equiv 0$. Time discretization is applied when we evaluate the score approximations $\{S(\theta;t,\Bar{Y}_t)\}$.}. 
 The generation algorithm can be piecewisely expressed as a continuous-time SDE: for any $t\in [\backt_j,\backt_{j+1})$,
          \vskip -0.3in
\begin{align}\label{eqn:EI_backward_SDE}
    d \Bar{Y}_t =  2\sigma_{T-t}^2 S(\theta;T-\backt_j, \Bar{Y}_{\backt_j})  dt +\sqrt{2\sigma_{T-t}^2} d\Bar{W}_t.
\end{align}
         \vskip -0.1in
\textbf{Initialization.} Denote $q_t:= \text{Law}(\Bar{Y}_t)$ for all $t\in [0,T-\delta]$. We choose the Gaussian initialization, $q_0=\mathcal{N}(0,\bar{\sigma}_T^2)$.
         \vskip -0.05in
Our convergence result relies on the following assumption.
\begin{assumption}
         \vskip -0.05in
\label{assump:finite_second_moment} The distribution $P_0$ has a finite second moment: $\mathbb{E}_{x\sim P_0}[ \lVert x \rVert^2] = \mathrm{m}_2^2<\infty$. 
\end{assumption}
         \vskip -0.1in
Next we state the main convergence result, whose proof is provided in Appendix~\ref{app:sampling_proof}.
         \vskip -0.2in
\begin{theorem}
         \vskip -0.05in\label{thm:VESDE_generation} Under Assumption~\ref{assump:finite_second_moment}, for any $\delta\in (0,1)$ and $T>1$, we have
                  \vskip -0.3in
\begin{align}\label{eqn:VESDE_generation}
    &\KL(p_\delta|q_{T-\delta}) \lesssim  \underbrace{\frac{\mathrm{m}_2^2}{\bar{\sigma}_T^2}}_{E_I} + \underbrace{\sum_{j=0}^{N-1} \gamma_j \sigma_{T-\backt_j}^2 \mathbb{E}_{Y_{\backt_j}\sim p_{T-\backt_j}}[\lVert S(\theta;T-\backt_j,Y_{\backt_j})-\nabla \log p_{T-\backt_j}(Y_{\backt_j}) \rVert^2]}_{E_S} \nonumber \\
    &+ \underbrace{d\sum_{j=0}^{N-1} \gamma_j \int_{\backt_j}^{\backt_{j+1}} \frac{\sigma_{T-t}^4}{\bar{\sigma}_{T-t}^4} dt + \mathrm{m}_2^2\frac{\int_0^{\backt_1}\sigma_{T-t}^2 d t }{\bar{\sigma}_T^{4}}  +(\mathrm{m}_2^2+d)\sum_{j=1}^{N-1} ( 1-e^{-\bar{\sigma}_{T-\backt_j}^2} ) \frac{\bar{\sigma}_{T-\backt_j}^4-\bar{\sigma}_{T-\backt_{j+1}}^2\bar{\sigma}_{T-\backt_{j-1}}^2}{\bar{\sigma}_{T-\backt_{j-1}}^2\bar{\sigma}_{T-\backt_j}^4}}_{E_D}.
\end{align}
             \vskip -0.2in
where $\gamma_j:=\backt_{j+1}-\backt_j$ for all $j=0,1,\cdots,N-1$ is the stepsize of the generation algorithm in \eqref{eqn:EI_backward_SDE}.
             \vskip -0.1in
\end{theorem}
             \vskip -0.1in
Theorem \ref{thm:VESDE_generation} is a VESDE-based diffusion model's analogy of what's proved in \citet{benton2024nearly} for VPSDE-based diffusion model, only requiring the data distributions to have finite second moments, and it achieves the sharp almost linear data dimension dependence under the exponential time schedule. The major differences from \citep{benton2024nearly} are (1) the initialization error in the VESDE case is handled differently (see Lemma~\ref{lem:initialization error VE}); (2) Theorem \ref{thm:VESDE_generation} applies to varies choices of time schedules, which enables to investigate the design space of the diffusion model, as we will discuss in Section \ref{sec:design_space_explore}. Worth mentioning is, \citet{yangconvergence} also obtained 
 polynomial complexity results for VESDE-based diffusion models with uniform stepsize, but under stronger data assumption (assuming compact support). Compared to their result, 
 complexity implied by Theorem~\ref{thm:VESDE_generation} has better accuracy and data dimension dependencies. A detailed discussion on complexities is given in Appendix~\ref{append:EI+ED dominates}.

Terms $E_I,E_D, E_S$ in \eqref{eqn:VESDE_generation} represent the three types of errors: initialization error, discretization error, and score estimation error, respectively. Term $E_I$ quantifies the error between the initial density of the sampling algorithm $q_0$ and the ideal initialization $p_T$, which is the density when the forward process stops at time $T$. Term $E_D$ is the error stemming from the discretization of the backward dynamics. Term $E_S$ characterizes the error of the estimated score function and the true score, and is related to the optimization error of $\lem$. Important to note is, in Theorem~\ref{thm:VESDE_generation}, population loss is needed instead of the empirical version $\lem$ ~\eqref{fct:lem}. Besides this, the weighting $\gamma_j \sigma_{T-\backt_j}^2$ is not necessarily the same as the total weighting in $\lem$ \eqref{fct:lem} $\beta_j$, depending on choices of $w(t_j)$ and time and variance schedules (see more discussion in Section~\ref{sec:design_space_explore}). We will later on integrate the optimization error (Theorem~\ref{thm:convergence_GD_no_interpolation}) into this score error $E_S$ to obtain a full error analysis in Section~\ref{subsec:full_error_analysis}. 

\begin{remark}[sharpness of dependence in $d$ and $\mathrm{m}_2^2$] In one of the simplest cases, when the data distribution is Gaussian, the score function is explicitly known. Hence $\KL(p_\delta|q_{T-\delta})$ can be explicitly computed as well, which verifies that the dependence of parameters $d$ and $\mathrm{m}_2^2$ is sharp in $E_I$ and $E_D$.  
\end{remark}

\vspace{-0.1in}

\subsection{Full error analysis}
\label{subsec:full_error_analysis}
\vspace{-0.1in}

In this section, we combine the analyses from the previous two sections to obtain an end-to-end generation error bound.

Before providing the main result of this section, let us first clarify some terminologies.

\textbf{Time schedule, variance schedule, and total weighting.}
The terms {\it time schedule} and {\it variance schedule} respectively refer to the choice of $\backt_j$ and $\bar{\sigma}_{t_j}$ in sampling. Meanwhile, note both the training and sampling processes require the proper choices of time and variance, and these choices are not necessarily the same for both processes. For training, the effect of these two is integrated into the {\it total weighting} $\beta_j$, which is also influenced by an additional weighting parameter $w(t_j)$. In this theoretical paper, when studying the generation error, we aim to apply the optimization result to better understand the effect of optimization on sampling. Therefore, to simplify the analysis and discussions in Section~\ref{sec:design_space_explore}, we choose the same time and variance schedules for both training and sampling. 

\vspace{-0.05in}

The main result is stated in the following.

\begin{theorem}
\label{cor:full_error_analysis}

    Under the same conditions as Theorem~\ref{thm:convergence_GD_no_interpolation},\ref{thm:VESDE_generation}, and that $K$ is such that GD reaches $\epsilon_{\text{train}}$ in at most $K$th iterations,
    we have
    \begin{align*}
        &\KL(p_\delta|q_{T-\delta}) \lesssim E_I+E_D+\max_{1\leq j \leq N} \frac{\sigma^2_{t_{N-j}}}{w(t_{N-j})} \left(\epsilon_{\rm{train}}+\epsilon_{n}+\epsilon_{\rm est}+\epsilon_{\rm approx}\right)
    \end{align*}
    \vskip -0.15in
    where $E_I,E_D$ are defined in Theorem~\ref{thm:VESDE_generation}, $\epsilon_{\rm train}$ is defined in Theorem~\ref{thm:convergence_GD_no_interpolation}, $\epsilon_{n}=|\bar{\mathcal{L}}(\theta^{(K)})-\lem(\theta^{(K)})+\lem(\theta^*)-\bar{\mathcal{L}}(\theta^*)|$, $\epsilon_{\rm est}=|\bar{\mathcal{L}}(\theta^*)-\bar{\mathcal{L}}(\theta_\mathcal{F})|$, $\epsilon_{\rm approx}=|\bar{\mathcal{L}}(\theta_\mathcal{F})+\bar{C}|$. In these terms, $\bar{C}$ is defined in \eqref{eqn:barL+barC}, $\theta^*=\arg\min_{\theta, s.t., \lem(\theta)=0}\bar{\mathcal{L}}(\theta)$ and $\theta_{\mathcal{F}}=\arg\inf_{\{\theta:S(\theta)\in\mathcal{F}\}}|\bar{\mathcal{L}}(\theta)+\bar{C}|$ with 
    $\mathcal{F}=\big\{$ReLU network function defined in \eqref{fct:neural network}, with $d=\Omega({\rm{poly}}(\log (nN))), m=\Omega\left(\text{poly}\big(n,N,d,L,T/t_0\big)\right)\big\}$.
\end{theorem}

In this theorem, the discretization error $E_D$ and initialization error $E_I$ are the same as Theorem~\ref{thm:VESDE_generation}. For the score error $E_S$,
our optimization result is valid for general time schedules and therefore can directly fit into the sampling error analysis, which is in contrast to existing works \citep{han2024neural,shah2023learning} (see more discussions in Section~\ref{subsec:related_works}). The coefficient $\max_j {\sigma^2_{t_{N-j}}}/{w(t_{N-j})}$ results from different weightings in $E_S$ and $\lem$, i.e., $\gamma_j \sigma_{T-\backt_j}^2$ and $\beta_j$. We will discuss the effect of $\max_j {\sigma^2_{t_{N-j}}}/{w(t_{N-j})}$ under different time and variance schedules in Section~\ref{sec:design_space_explore}.

The way we bound $\mathbb{E}_{Y_{\backt_j}\sim p_{T-\backt_j}}[\lVert S(\theta;T-\backt_j,Y_{\backt_j})-\nabla \log p_{T-\backt_j}(Y_{\backt_j}) \rVert^2]$ in $E_S$ (see Theorem~\ref{thm:VESDE_generation}) is to decompose it into the optimization error $\epsilon_{\rm train}$, statistical error $\epsilon_n$, estimation error $\epsilon_{\rm est}$, and approximation error $\epsilon_{\rm approx}$. This gives clear intuition to results, but we also note it may not give a tight bound. In fact, we have
\begin{align*}
    \epsilon_n+\epsilon_{\rm train}&=|\bar{\mathcal{L}}(\theta^{(K)})-\lem(\theta^{(K)})+\lem(\theta^*)-\bar{\mathcal{L}}(\theta^*)|+ |\lem(\theta^{(K)})-\lem(\theta^*)|\\
    &\ge \bar{\mathcal{L}}(\theta^{(K)})+\bar{\mathcal{L}}(\theta^*)\ge 2\min_\theta\bar{\mathcal{L}}(\theta)\ge -2\bar{C} > 0.
\end{align*}
$\epsilon_n$ can still be small if we take $n\to\infty$, but that means $\epsilon_{\rm train}$ has to be large, and our generation error bound cannot be made 0. It is unclear yet whether this is due to limitation of our analysis or intrinsic, and will be left for future investigation.

Another related note is, in this paper, we focus on $\epsilon_{\text{train}}$ and the effect of optimization, but the analyses of $\epsilon_n,\ \epsilon_{\rm est},\text{ and }\epsilon_{\rm approx}$ are also important and possible \cite{chen2023score,oko2023diffusion,han2024neural,wibisono2024optimal}. On the other hand, again, whether it is optimal to decompose the full error into these four is unclear.

To better see the parameter dependence of the error bound in Theorem~\ref{cor:full_error_analysis}, the following is an example with simplified results, where we employ the schedules in EDM~\citep{karras2022elucidating}.
\begin{corollary}[Full error analysis under EDM \citep{karras2022elucidating} designs]
    Under the same conditions as Theorem~\ref{cor:full_error_analysis}, we have
    \begin{align*}
        KL(p_\delta|q_{T-\delta}) \lesssim  \frac{\mathrm{m}_2^2}{T^2} +\frac{d a^2 T^{\frac{1}{a}} }{\delta^{\frac{1}{a}}N}+(\mathrm{m}_2^2+d)\left( \frac{a^2T^{\frac{1}{a}}}{\delta^{\frac{1}{a}} N}+\frac{a^3T^{\frac{2}{a}}}{\delta^{\frac{2}{a}}N^2}\right)+\frac{1}{N}\left(C_9+\left(1-C_8 h \left(\frac{md^{\frac{a_0-1}{2}}}{n^3N^2}\right)\right)^K\right),
    \end{align*}
    where $C_8, C_9>0$ and $a=7$ in~\citep{karras2022elucidating}.
\end{corollary}

\vspace{-0.15in}

\section{Theory-based understanding of the design space and its relation to existing empirical counterparts}
\label{sec:design_space_explore}
\vspace{-0.1in}

This section theoretically explores preferable choices of parameters in both training and sampling, and shows that they agree with the ones used in EDM \citep{karras2022elucidating} and \citet{song2020score} in different circumstances.

\vspace{-0.1in}

\subsection{Choice of total weighting for training}
\label{subsec:weighting}
\vspace{-0.1in}

This section develops the optimal total weighting $\beta_j$ for training objective~\eqref{fct:lem}. We qualitatively show in two steps that ``bell-shaped'' weighting, which is the one used in EDM~\citep{karras2022elucidating}, will lead to the optimal rate of convergence:
Step 1) $\|\bar{\sigma}_{t_j} S(\theta;t_j,X_{ij})+\xi_{ij}\|$ as a function of $j$ is inversely ``bell-shaped''; Step 2) $f(\theta;i,j)=\beta_j\|\bar{\sigma}_{t_j} S(\theta;t_j,X_{ij})+\xi_{ij}\|$ should be close to each other for any~$i,j$.

\vspace{-0.1in}
\subsubsection{Inversely ``bell-shaped'' loss $\|\bar{\sigma}_{t_j} S(\theta;t_j,X_{ij})+\xi_{ij}\|$ as a function of time index $j$}
\vspace{-0.1in}

\begin{proposition}
\label{prop:inverse_bell_shaped_loss}
    Under the same assumptions as Theorem~\ref{thm:convergence_GD_no_interpolation}, for any $\theta$ and $i=1,\cdots,n$, we have 
    \begin{enumerate}[topsep=0pt,parsep=0pt,partopsep=0pt, align=left,leftmargin=*,widest={3}]
        \item $\forall \epsilon_1>0$, $\exists\ \delta_1>0$, s.t., when $0\le\bar{\sigma}_{t_j}<\delta_1$, $\|\bar{\sigma}_{t_j} S(\theta;t_j,x_i+\bar{\sigma}_{t_j}\xi_{ij})+\xi_{ij}\|>\|\xi_{ij}\|-\epsilon_1.$ \label{prop:1_inverse_bell_shaped_loss}
        \item $\forall \epsilon_2>0$, $\exists\ M>0$, s.t., when $\bar{\sigma}_{t_j}>M$, $\|\bar{\sigma}_{t_j} S(\theta;t_j,x_i+\bar{\sigma}_{t_j}\xi_{ij})+\xi_{ij}\|\ge M^2 (\|S(\theta;t_j,\xi_{ij})\|-\epsilon_2)$.
        \label{prop:2_inverse_bell_shaped_loss}
    \end{enumerate}
\end{proposition}
\vskip -0.1in

The above proposition can be interpreted in the following way. Given any network $S$, 
when $\bar{\sigma}_{t_j}$ is very small, \ref{prop:1_inverse_bell_shaped_loss} implies that $\|\bar{\sigma}_{t_j} S(\theta;t_j,x_i+\bar{\sigma}_{t_j}\xi)+\xi_{ij}\|$ is away from 0 by approximately $\|\xi_{ij}\|$ which is of order $\sqrt{d}$ with high probability, i.e., it cannot be small. When $\bar{\sigma}_{t_j}$ is large, \ref{prop:2_inverse_bell_shaped_loss} shows that as it becomes larger and larger, i.e., as $M$ increases, $\|\bar{\sigma}_{t_j} S(\theta;t_j,x_i+\bar{\sigma}_{t_j}\xi)+\xi_{ij}\|$ 
will also increase.
Therefore, the function $\|\bar{\sigma}_{t_j} S(\theta;t_j,X_{ij})+\xi_{ij}\|$
has most likely an inversely ``bell-shaped'' curve in terms of $j$ dependence.

\vspace{-0.1in}

\subsubsection{ Ensuring comparable values of $f(\theta;i,j)$ for optimal rate of convergence}
\vspace{-0.1in}

\begin{corollary}
\label{cor:optimal_rate}
    Under the same conditions as Theorem~\ref{thm:convergence_GD_no_interpolation}, for some large $K'>0$, if $|f(\theta^{(k+K')};i,j)-f(\theta^{(k+K')};l,s)|\le\epsilon$ holds for all $k>0$ and all $(i,j),(l,s)$, with some small universal constant $\epsilon>0$, then we have, for some constant $C_7>0$,
    \vspace{-0.1in}
    \begin{align*}
         &\lem(\theta^{(k+K')})\le\left(1-C_7 h \ \max_{j=1,\cdots,N} w(t_{j})(t_{j}-t_{j-1})\bar{\sigma}_{t_{j}} \left(\frac{md^{\frac{a_0-1}{2}}}{n^3N^2}\right)\right)^k\lem(\theta^{(K')}).
    \end{align*}
\end{corollary}
\vspace{-0.2in}
\begin{wrapfigure}{r}{0.25\textwidth}
    \centering  
    \vspace{-0.1in}
\includegraphics[width=0.3\textwidth]{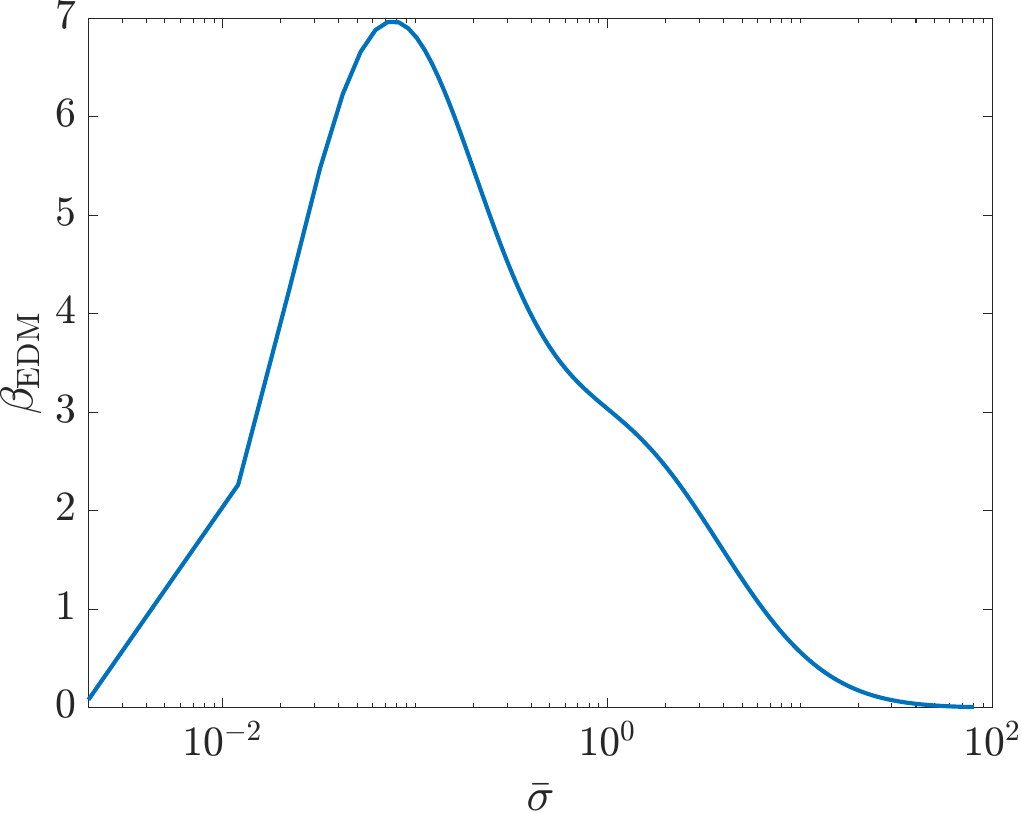}
    \vspace{-0.2in}
    \caption{Weighting choice $\beta_{\rm EDM}$ in EDM.}
    \vspace{-0.3in}
    \label{fig:EDM_weighting}
\end{wrapfigure}
The above corollary shows that
if $f(\theta^{(k)};i,j)$'s are almost the same for any $i,j$, then the decay ratio of the next iteration is minimized. More precisely, the index set $\mathcal{G}^{(k)}$ defined in Theorem~\ref{thm:convergence_GD_no_interpolation} is roughly the whole set $\{1,\cdots,N\}$, and therefore $w(t_{j^*})(t_{j^*}-t_{j^*-1})\bar{\sigma}_{t_{j^*}}$ can be taken as the maximum value over all $j$, which consequently leads to the optimal rate.

\vspace{-0.1in}

\subsubsection{``Bell-shaped'' weighting: our theory and EDM}
\vspace{-0.1in}

Combining the above two aspects, the optimal rate of convergence leads to the choice of total weighting $\beta_j$ such that $f(\theta;i,j)=\beta_j\|\bar{\sigma}_{t_j} S(\theta;t_j,X_{t_j})+\xi_{ij}\|$ is close to each other; as a result, the total weighting should be chosen as a ``bell-shaped'' curve as a function of $j$ according to the shape of the curve for $\|\bar{\sigma}_{t_j} S(\theta;t_j,X_{t_j})+\xi_{ij}\|$.

Before comparing the preferable weighting predicted by our theory and the intuition-and-empirics-based one in EDM~\citep{karras2022elucidating}, let us first recall that the EDM training objective\footnote{In EDM \citep{karras2022elucidating}, they use $P_{\rm mean}=-1.2,\ P_{\rm std}=1.2,\ \sigma_{\rm data}=0.5,\ \bar{\sigma}_{\min}=0.002,\ \bar{\sigma}_{\max}=80$.} can be written as $\mathbb{E}_{\bar{\sigma}\sim p_{\rm train}}\mathbb{E}_{y,n}\lambda(\bar{\sigma})\|D_\theta(y+n;\bar{\sigma})-y\|^2$
\begin{align}
    \label{fct:EDM_training_objective}=\frac{1}{Z_1}\int e^{-\frac{(\log \bar{\sigma}-P_{\rm mean})^2}{2P_{\rm std}^2}}\frac{\bar{\sigma}^2+\sigma_{\rm data}^2}{\bar{\sigma}\sigma_{\rm data}^2}\mathbb{E}_{X_0,\xi}\|\bar{\sigma} s(\theta;t,X_{t})+\xi\|^2\, d\bar{\sigma},
\end{align}
\vskip -0.2in
where $Z_1$ is a normalization constant, and we denote $\beta_{\rm EDM}(\bar{\sigma})=e^{-\frac{(\log \bar{\sigma}-P_{\rm mean})^2}{2P_{\rm std}^2}}\frac{\bar{\sigma}^2+\sigma_{\rm data}^2}{\bar{\sigma}\sigma_{\rm data}^2}$ to be the {\it total weighting} of EDM. Note the dependence on $\bar{\sigma}$ and time $j$ can be freely switched due to their 1-to-1 correspondence.

 Figure~\ref{fig:EDM_weighting} plots the total weighting of EDM $\beta_{\rm EDM}$ as a function of $\bar{\sigma}$. As is shown in the picture, this is a ``bell-shaped'' curve\footnote{This horizontal axis is in $\log$-scale and the plot in regular scale is a little bit skewed, not precisely a ``bell'' shape. However, we remark that the trend of the curve still matches our theory.}, which coincides with our choice of total weighting in the above theory. When $\bar{\sigma}$ is very small or very large, according to Proposition~\ref{prop:inverse_bell_shaped_loss}, the lower bound of $\|\bar{\sigma}_{t_j} S(\theta;t_j,X_{t_j})+\xi_{ij}\|$ cannot vanish and therefore needs the smallest weighting over all $\bar{\sigma}$. When $\bar{\sigma}$ takes the middle value, the scale of the output data $\xi_{ij}/\bar{\sigma}_j$ is roughly the same as the input data $X_{ij}$ and therefore makes it easier for the neural network to fit the data, which admits larger weighting.

\vspace{-0.1in}

\subsection{Choice of time and variance schedules}\label{subsec:time variance schedule}

\vspace{-0.1in}

This section will discuss the choice of time and variance schedules based on the three errors $E_S,E_D,E_I$ in the error analysis of Section~\ref{subsec:full_error_analysis}. Two situations will be considered based on how well the score function is approximated in training: when the network is less trained, $E_S$ dominates and polynomial schedule~\citep{karras2022elucidating} is preferable; when the score function is well approximated, $E_D+E_I$ dominates and exponential schedule~\citep{song2020score} is better.

\vspace{-0.05in}

\subsubsection{When score error $E_S$ dominates}
\vspace{-0.1in}

As is shown in Theorem~\ref{cor:full_error_analysis}, the main impact of different time and variance schedules on score error $E_S$ appears in the term $\max_j {\sigma^2_{t_{N-j}}}/{w(t_{N-j})}$, when the score function is approximated to a certain accuracy. It remains to compute $w(t)$ under various choices of schedules.

\textbf{General rule of constructing $w(t)$.} To ensure fair comparisons between different time and variance schedules, we maintain a fixed total weighting in the training objective. Additionally, to facilitate comparisons with practical usage, we adopt the total weighting in EDM, i.e., $    
\beta_{j}=C_3\beta_{\rm EDM}(\bar{\sigma}_{t_j}),$
for some universal constant $C_3>0$. The reason for using the EDM total weighting is that according to Section~\ref{subsec:weighting}, our total weighting $\beta_j$ should be ``bell-shaped'' as a function of $j$, which agrees qualitatively with the one used in EDM.

\textbf{Polynomial  schedule \citep{karras2022elucidating} vs exponential  schedule \citep{song2020score}.} 
We fix $\epsilon_n,\epsilon_{\rm train}$ and apply the two schedules (Table~\ref{tab:schedule}) separately to the above total weighting $\beta$ (hence $w$). Then, compute $\max_j {\sigma^2_{t_{N-j}}}/{w(t_{N-j})}$ which is a factor in score error $E_S$ (Thm.\ref{cor:full_error_analysis}) 
in Table~\ref{tab:E_S_EDM_VE}. The Exp.'s result $\frac{1}{2}\left(\bar{\sigma }_{\max }-\bar{\sigma }_{\max } \left(\frac{\bar{\sigma }_{\min }^2}{\bar{\sigma }_{\max }^2}\right){}^{1/N}\right)$ is larger\footnote{This holds under parameters used in either \citet{song2020score} or \citet{karras2022elucidating}.} than the Poly.'s result $\left( \bar{\sigma }_{\max }-\left(\bar{\sigma }_{\max }^{1/\rho }-\frac{\bar{\sigma }_{\max }^{1/\rho }-\bar{\sigma }_{\min }^{1/\rho }}{N}\right)^{\rho }\right)$ for large $N$, meaning the poly. time schedule in EDM is better than the exp. schedule in \citep{song2020score}.
Note these two terms are both of order $1/N$ as $N\to\infty$ and therefore the difference lies in their prefactors.

\begin{table}[htb!]
  \caption{Polynomial and exponential (time) schedules.}
  \label{tab:schedule}
  \centering
  \begin{tabular}{lll}
    \toprule
    & Variance schedule $\bar{\sigma}_t$ & Time schedule $t_j$ \\
    \midrule
    Poly. \citep{karras2022elucidating} & $t$ & $\left( \bar{\sigma}_{\max}^{1/\rho}-(\bar{\sigma}_{\max}^{1/\rho}-\bar{\sigma}_{\min}^{1/\rho})\frac{N-j}{N} \right)^{\rho}$ \\
    Exp. \citep{song2020score} & $\sqrt{t}$ & $\bar{\sigma}_{\max}^2\left(\frac{\bar{\sigma}_{\min}^2}{\bar{\sigma}_{\max}^2}\right)^{\frac{N-j}{N}}$ \\
    \bottomrule
  \end{tabular}
\end{table}
% \vspace{-0.1in}

% \vspace{-0.1in}

\begin{table}[htb!]
  \caption{Comparisons between different schedules.}
  \label{tab:E_S_EDM_VE}
  \centering
  \begin{tabular}{l|ll|ll}
    \toprule
    &\multicolumn{2}{c|}{$E_S$ (score error) dominates}
    &\multicolumn{2}{c}{$E_D+E_I$ (sampling error) dominates}\\
    \cmidrule(r){2-5}
    & $\max_j {\sigma^2_{t_{j}}}/{w(t_{j})}$ & Choice & $N$ & Choice\\
    \midrule
    Poly. \citep{karras2022elucidating} &  $C_4\left( \bar{\sigma }_{\max }-\left(\bar{\sigma }_{\max }^{1/\rho }-\frac{\bar{\sigma }_{\max }^{1/\rho }-\bar{\sigma }_{\min }^{1/\rho }}{N}\right)^{\rho }\right)$ & \Checkmark & $\Omega\left(\frac{\mathrm{m}_2^2 \vee d}{d}\rho^2\big(\frac{\bar{\sigma}_{\max}}{\bar{\sigma}_{\min}}\big)^{1/\rho}\bar{\sigma}_{\max}^2 \right)$ & \\
    Exp. \citep{song2020score} & $C_4\cdot\frac{1}{2}\left(\bar{\sigma }_{\max }-\bar{\sigma }_{\max } \left(\frac{\bar{\sigma }_{\min }^2}{\bar{\sigma }_{\max }^2}\right){}^{1/N}\right)$&&$\Omega\left(\frac{\mathrm{m}_2^2 \vee d}{d}\ln(\frac{\bar{\sigma}_{\max}}{\bar{\sigma}_{\min}})^2\bar{\sigma}_{\max}^2\right)$&\Checkmark \\
    \bottomrule
  \end{tabular}
\end{table}

\subsubsection{When discretization error $E_D$ and initialization error $E_I$ dominate}
\vspace{-0.1in}

In this section, we compare the two different schedules in Table~\ref{tab:schedule} by studying the iteration complexity of the sampling algorithm, i.e., number of time points $N$, when $E_D+E_I$ dominates. 

\textbf{General rules of comparison.} We consider the case when the discretization and initialization errors are bounded by the same quantity $\epsilon$, i.e., $E_I+E_D\lesssim \varepsilon.$
Then according to Theorem~\ref{thm:VESDE_generation} and Theorem~\ref{cor:full_error_analysis}, we compute the iteration complexity for achieving this error using the two schedules in Table~\ref{tab:schedule}. To make the comparison more straightforward, we adopt $T=t_N=\Theta(\text{poly}(\varepsilon^{-1}))$ and therefore $\bar{\sigma}_{\max}=\Theta(\varepsilon^{-1/2})$. More details are provided in Appendix \ref{append:EI+ED dominates}. 

\textbf{Polynomial  schedule \citep{karras2022elucidating} vs exponential  schedule \citep{song2020score}.} As is shown in the last column of Table~\ref{tab:E_S_EDM_VE}, the iteration complexity under exponential schedule~\citep{song2020score} has the poly-logarithmic dependence on the ratio between maximal and minimal variance ($\bar{\sigma}_{\max}/\bar{\sigma}_{\min}$)\footnote{The exponential time schedule under the variance schedule in \citep{karras2022elucidating} also has the poly-logarithmic dependence on $\bar{\sigma}_{\max}/\bar{\sigma}_{\min}$. Under both variance schedules in \citep{karras2022elucidating} and \citep{song2020score}, it can be shown that exponential time schedule is optimal. Details are provided in Appendix~\ref{append:EI+ED dominates}.}, which is better than the complexity under polynomial schedule~\citep{karras2022elucidating}, which is polynomially dependent on $\bar{\sigma}_{\max}/\bar{\sigma}_{\min}$. Both complexities are derived from Theorem \ref{thm:VESDE_generation} by choosing different parameters.
% \begin{itemize}[noitemsep,leftmargin=0pt]
%     \item []\textbf{polynomial schedule:} $\bar{\sigma}_t=t$, $T-\backt_j=(\delta^{1/a}+(N-k)h)^a$, $h=\tfrac{T^{1/a}-\delta^{1/a}}{N}$, $\delta=\bar{\sigma}_{\min}$, $T=\bar{\sigma}_{\max}=\Theta(\tfrac{\mathrm{m}_2}{\varepsilon^{1/2}})$ and $a=\rho$.
%     \item [] \textbf{exponential schedule:} $\bar{\sigma}_t=\sqrt{t}$, $\gamma_k=\kappa(T-\backt_j)$ with $\kappa=\tfrac{\log(T/\delta)}{N}$, $\delta=\bar{\sigma}_{\min}^2$ and $T=\bar{\sigma}_{\max}^2=\Theta(\tfrac{\mathrm{m}_2^2}{\varepsilon})$.
% \end{itemize} 

\begin{remark}[The existence of optimal $\rho$ in the polynomial schedule~\citep{karras2022elucidating}] 
For fixed $\bar{\sigma}_{\max}$ and $\bar{\sigma}_{\min}$, the optimal $\rho$ that minimizes the iteration complexity is $\rho= \frac{1}{2} \ln \big(\tfrac{\bar{\sigma}_{\max}}{\bar{\sigma}_{\min}}\big)$. In \citep{karras2022elucidating}, it was empirically observed that with fixed iteration complexity, there is an optimal value of $\rho$ that minimizes the FID. Our result indicates that, for fixed $\bar{\sigma}_{\max}$ and $\bar{\sigma}_{\min}$, hence the desired accuracy in KL divergence being fixed, there is an optimal value of $\rho$ that minimizes the iteration complexity to reach the fixed accuracy. Even though we consider a different metric/divergence instead of FID, our result still provides a quantitative support to the existence of optimal $\rho$ observed in \citep{karras2022elucidating}. 

\end{remark}

% acknowledgement
\begin{ack}
The authors are grateful for the partially support by NSF DMS-1847802, Cullen-Peck Scholarship, and GT-Emory Humanity.AI Award. We thank the anonymous reviewers for their helpful comments.
% Use unnumbered first level headings for the acknowledgments. All acknowledgments
% More information about this disclosure can be found at: \url{https://neurips.cc/Conferences/2024/PaperInformation/FundingDisclosure}.

% Do {\bf not} include this section in the anonymized submission, only in the final paper. You can use the \texttt{ack} environment provided in the style file to automatically hide this section in the anonymized submission.
\end{ack}

\bibliographystyle{plainnat}
\bibliography{ref}

\begin{thebibliography}{58}
\providecommand{\natexlab}[1]{#1}
\providecommand{\url}[1]{\texttt{#1}}
\expandafter\ifx\csname urlstyle\endcsname\relax
  \providecommand{\doi}[1]{doi: #1}\else
  \providecommand{\doi}{doi: \begingroup \urlstyle{rm}\Url}\fi

\bibitem[Allen-Zhu et~al.(2019{\natexlab{a}})Allen-Zhu, Li, and
  Song]{allen2019convergence}
Zeyuan Allen-Zhu, Yuanzhi Li, and Zhao Song.
\newblock A convergence theory for deep learning via over-parameterization.
\newblock In \emph{International conference on machine learning}, pages
  242--252. PMLR, 2019{\natexlab{a}}.

\bibitem[Allen-Zhu et~al.(2019{\natexlab{b}})Allen-Zhu, Li, and
  Song]{allen2019convergencerecurr}
Zeyuan Allen-Zhu, Yuanzhi Li, and Zhao Song.
\newblock On the convergence rate of training recurrent neural networks.
\newblock \emph{Advances in neural information processing systems}, 32,
  2019{\natexlab{b}}.

\bibitem[Anand and Achim(2022)]{anand2022protein}
Namrata Anand and Tudor Achim.
\newblock Protein structure and sequence generation with equivariant denoising
  diffusion probabilistic models.
\newblock \emph{arXiv preprint arXiv:2205.15019}, 2022.

\bibitem[Anderson(1982)]{anderson1982reverse}
Brian~DO Anderson.
\newblock Reverse-time diffusion equation models.
\newblock \emph{Stochastic Processes and their Applications}, 12\penalty0
  (3):\penalty0 313--326, 1982.

\bibitem[Arora et~al.(2019)Arora, Du, Hu, Li, and Wang]{arora2019fine}
Sanjeev Arora, Simon Du, Wei Hu, Zhiyuan Li, and Ruosong Wang.
\newblock Fine-grained analysis of optimization and generalization for
  overparameterized two-layer neural networks.
\newblock In \emph{International Conference on Machine Learning}, pages
  322--332. PMLR, 2019.

\bibitem[Austin et~al.(2021)Austin, Johnson, Ho, Tarlow, and Van
  Den~Berg]{austin2021structured}
Jacob Austin, Daniel~D Johnson, Jonathan Ho, Daniel Tarlow, and Rianne Van
  Den~Berg.
\newblock Structured denoising diffusion models in discrete state-spaces.
\newblock \emph{Advances in Neural Information Processing Systems},
  34:\penalty0 17981--17993, 2021.

\bibitem[Baranchuk et~al.(2022)Baranchuk, Voynov, Rubachev, Khrulkov, and
  Babenko]{baranchuk2022labelefficient}
Dmitry Baranchuk, Andrey Voynov, Ivan Rubachev, Valentin Khrulkov, and Artem
  Babenko.
\newblock Label-efficient semantic segmentation with diffusion models.
\newblock In \emph{International Conference on Learning Representations}, 2022.
\newblock URL \url{https://openreview.net/forum?id=SlxSY2UZQT}.

\bibitem[Benton et~al.(2024)Benton, De~Bortoli, Doucet, and
  Deligiannidis]{benton2024nearly}
Joe Benton, Valentin De~Bortoli, Arnaud Doucet, and George Deligiannidis.
\newblock Nearly d-linear convergence bounds for diffusion models via
  stochastic localization.
\newblock In \emph{The Twelfth International Conference on Learning
  Representations}, 2024.

\bibitem[Block et~al.(2020)Block, Mroueh, and Rakhlin]{block2020generative}
Adam Block, Youssef Mroueh, and Alexander Rakhlin.
\newblock Generative modeling with denoising auto-encoders and langevin
  sampling.
\newblock \emph{arXiv preprint arXiv:2002.00107}, 2020.

\bibitem[Cai et~al.(2019)Cai, Yang, Lee, and Wang]{cai2019neural}
Qi~Cai, Zhuoran Yang, Jason~D Lee, and Zhaoran Wang.
\newblock Neural temporal-difference learning converges to global optima.
\newblock \emph{Advances in Neural Information Processing Systems}, 32, 2019.

\bibitem[Cao et~al.(2024)Cao, Tan, Gao, Xu, Chen, Heng, and Li]{cao2024survey}
Hanqun Cao, Cheng Tan, Zhangyang Gao, Yilun Xu, Guangyong Chen, Pheng-Ann Heng,
  and Stan~Z Li.
\newblock A survey on generative diffusion models.
\newblock \emph{IEEE Transactions on Knowledge and Data Engineering}, 2024.

\bibitem[Chen et~al.(2023{\natexlab{a}})Chen, Lee, and Lu]{chen2023improved}
Hongrui Chen, Holden Lee, and Jianfeng Lu.
\newblock Improved analysis of score-based generative modeling: User-friendly
  bounds under minimal smoothness assumptions.
\newblock In \emph{International Conference on Machine Learning}, pages
  4735--4763. PMLR, 2023{\natexlab{a}}.

\bibitem[Chen et~al.(2023{\natexlab{b}})Chen, Huang, Zhao, and
  Wang]{chen2023score}
Minshuo Chen, Kaixuan Huang, Tuo Zhao, and Mengdi Wang.
\newblock Score approximation, estimation and distribution recovery of
  diffusion models on low-dimensional data.
\newblock In \emph{International Conference on Machine Learning}, pages
  4672--4712. PMLR, 2023{\natexlab{b}}.

\bibitem[Chen et~al.(2024)Chen, Mei, Fan, and Wang]{chen2024overview}
Minshuo Chen, Song Mei, Jianqing Fan, and Mengdi Wang.
\newblock An overview of diffusion models: Applications, guided generation,
  statistical rates and optimization.
\newblock \emph{arXiv preprint arXiv:2404.07771}, 2024.

\bibitem[Chen et~al.(2021)Chen, Zhang, Zen, Weiss, Norouzi, and
  Chan]{chen2021wavegrad}
Nanxin Chen, Yu~Zhang, Heiga Zen, Ron~J Weiss, Mohammad Norouzi, and William
  Chan.
\newblock Wavegrad: Estimating gradients for waveform generation.
\newblock In \emph{International Conference on Learning Representations}, 2021.
\newblock URL \url{https://openreview.net/forum?id=NsMLjcFaO8O}.

\bibitem[Chen et~al.(2023{\natexlab{c}})Chen, Chewi, Li, Li, Salim, and
  Zhang]{chen2022sampling}
Sitan Chen, Sinho Chewi, Jerry Li, Yuanzhi Li, Adil Salim, and Anru~R Zhang.
\newblock Sampling is as easy as learning the score: theory for diffusion
  models with minimal data assumptions.
\newblock \emph{ICLR}, 2023{\natexlab{c}}.

\bibitem[Chung and Ye(2022)]{chung2022score}
Hyungjin Chung and Jong~Chul Ye.
\newblock Score-based diffusion models for accelerated mri.
\newblock \emph{Medical image analysis}, 80:\penalty0 102479, 2022.

\bibitem[Conforti et~al.(2023)Conforti, Durmus, and Silveri]{conforti2023score}
Giovanni Conforti, Alain Durmus, and Marta~Gentiloni Silveri.
\newblock Score diffusion models without early stopping: finite fisher
  information is all you need.
\newblock \emph{arXiv preprint arXiv:2308.12240}, 2023.

\bibitem[De~Bortoli(2022)]{de2022convergence}
Valentin De~Bortoli.
\newblock Convergence of denoising diffusion models under the manifold
  hypothesis.
\newblock \emph{TMLR}, 2022.

\bibitem[Dhariwal and Nichol(2021)]{dhariwal2021diffusion}
Prafulla Dhariwal and Alexander Nichol.
\newblock Diffusion models beat gans on image synthesis.
\newblock \emph{Advances in neural information processing systems},
  34:\penalty0 8780--8794, 2021.

\bibitem[Du et~al.(2019)Du, Lee, Li, Wang, and Zhai]{du2019gradient}
Simon Du, Jason Lee, Haochuan Li, Liwei Wang, and Xiyu Zhai.
\newblock Gradient descent finds global minima of deep neural networks.
\newblock In \emph{International conference on machine learning}, pages
  1675--1685. PMLR, 2019.

\bibitem[Du et~al.(2018)Du, Zhai, Poczos, and Singh]{du2018gradient}
Simon~S Du, Xiyu Zhai, Barnabas Poczos, and Aarti Singh.
\newblock Gradient descent provably optimizes over-parameterized neural
  networks.
\newblock \emph{arXiv preprint arXiv:1810.02054}, 2018.

\bibitem[Duan et~al.(2023)Duan, Du, Jia, and Kulik]{duan2023accurate}
Chenru Duan, Yuanqi Du, Haojun Jia, and Heather~J Kulik.
\newblock Accurate transition state generation with an object-aware equivariant
  elementary reaction diffusion model.
\newblock \emph{Nature Computational Science}, 3\penalty0 (12):\penalty0
  1045--1055, 2023.

\bibitem[Gao and Zhu(2024)]{gao2024convergence}
Xuefeng Gao and Lingjiong Zhu.
\newblock Convergence analysis for general probability flow odes of diffusion
  models in wasserstein distances.
\newblock \emph{arXiv preprint arXiv:2401.17958}, 2024.

\bibitem[Han et~al.(2024)Han, Razaviyayn, and Xu]{han2024neural}
Yinbin Han, Meisam Razaviyayn, and Renyuan Xu.
\newblock Neural network-based score estimation in diffusion models:
  Optimization and generalization.
\newblock In \emph{The Twelfth International Conference on Learning
  Representations}, 2024.
\newblock URL \url{https://openreview.net/forum?id=h8GeqOxtd4}.

\bibitem[He et~al.(2024)He, Rojas, and Tao]{he2024zeroth}
Ye~He, Kevin Rojas, and Molei Tao.
\newblock Zeroth-order sampling methods for non-log-concave distributions:
  Alleviating metastability by denoising diffusion.
\newblock \emph{arXiv preprint arXiv:2402.17886}, 2024.

\bibitem[Ho et~al.(2022{\natexlab{a}})Ho, Saharia, Chan, Fleet, Norouzi, and
  Salimans]{ho2022cascaded}
Jonathan Ho, Chitwan Saharia, William Chan, David~J Fleet, Mohammad Norouzi,
  and Tim Salimans.
\newblock Cascaded diffusion models for high fidelity image generation.
\newblock \emph{Journal of Machine Learning Research}, 23\penalty0
  (47):\penalty0 1--33, 2022{\natexlab{a}}.

\bibitem[Ho et~al.(2022{\natexlab{b}})Ho, Salimans, Gritsenko, Chan, Norouzi,
  and Fleet]{ho2022video}
Jonathan Ho, Tim Salimans, Alexey Gritsenko, William Chan, Mohammad Norouzi,
  and David~J Fleet.
\newblock Video diffusion models.
\newblock \emph{Advances in Neural Information Processing Systems},
  35:\penalty0 8633--8646, 2022{\natexlab{b}}.

\bibitem[Hyv{\"a}rinen and Dayan(2005)]{hyvarinen2005estimation}
Aapo Hyv{\"a}rinen and Peter Dayan.
\newblock Estimation of non-normalized statistical models by score matching.
\newblock \emph{Journal of Machine Learning Research}, 6\penalty0 (4), 2005.

\bibitem[Karras et~al.(2022)Karras, Aittala, Aila, and
  Laine]{karras2022elucidating}
Tero Karras, Miika Aittala, Timo Aila, and Samuli Laine.
\newblock Elucidating the design space of diffusion-based generative models.
\newblock \emph{Advances in Neural Information Processing Systems},
  35:\penalty0 26565--26577, 2022.

\bibitem[Karras et~al.(2023)Karras, Aittala, Lehtinen, Hellsten, Aila, and
  Laine]{karras2023analyzing}
Tero Karras, Miika Aittala, Jaakko Lehtinen, Janne Hellsten, Timo Aila, and
  Samuli Laine.
\newblock Analyzing and improving the training dynamics of diffusion models.
\newblock \emph{arXiv preprint arXiv:2312.02696}, 2023.

\bibitem[Lee et~al.(2022)Lee, Lu, and Tan]{lee2022convergence}
Holden Lee, Jianfeng Lu, and Yixin Tan.
\newblock Convergence for score-based generative modeling with polynomial
  complexity.
\newblock \emph{Advances in Neural Information Processing Systems},
  35:\penalty0 22870--22882, 2022.

\bibitem[Lee and Kim(2014)]{lee2014concise}
Yongjae Lee and Woo~Chang Kim.
\newblock Concise formulas for the surface area of the intersection of two
  hyperspherical caps.
\newblock \emph{KAIST Technical Report}, 2014.

\bibitem[Li et~al.(2022)Li, Thickstun, Gulrajani, Liang, and
  Hashimoto]{li2022diffusion}
Xiang Li, John Thickstun, Ishaan Gulrajani, Percy~S Liang, and Tatsunori~B
  Hashimoto.
\newblock Diffusion-lm improves controllable text generation.
\newblock \emph{Advances in Neural Information Processing Systems},
  35:\penalty0 4328--4343, 2022.

\bibitem[Li and Liang(2018)]{li2018learning}
Yuanzhi Li and Yingyu Liang.
\newblock Learning overparameterized neural networks via stochastic gradient
  descent on structured data.
\newblock \emph{Advances in neural information processing systems}, 31, 2018.

\bibitem[Liu et~al.(2022)Liu, Pan, and Tao]{liu2022provable}
Xin Liu, Zhisong Pan, and Wei Tao.
\newblock Provable convergence of nesterov’s accelerated gradient method for
  over-parameterized neural networks.
\newblock \emph{Knowledge-Based Systems}, 251:\penalty0 109277, 2022.

\bibitem[Lou et~al.(2023)Lou, Meng, and Ermon]{lou2023discrete}
Aaron Lou, Chenlin Meng, and Stefano Ermon.
\newblock Discrete diffusion language modeling by estimating the ratios of the
  data distribution.
\newblock \emph{arXiv preprint arXiv:2310.16834}, 2023.

\bibitem[Meng et~al.(2022)Meng, He, Song, Song, Wu, Zhu, and
  Ermon]{meng2022sdedit}
Chenlin Meng, Yutong He, Yang Song, Jiaming Song, Jiajun Wu, Jun-Yan Zhu, and
  Stefano Ermon.
\newblock {SDE}dit: Guided image synthesis and editing with stochastic
  differential equations.
\newblock In \emph{International Conference on Learning Representations}, 2022.
\newblock URL \url{https://openreview.net/forum?id=aBsCjcPu_tE}.

\bibitem[Oko et~al.(2023)Oko, Akiyama, and Suzuki]{oko2023diffusion}
Kazusato Oko, Shunta Akiyama, and Taiji Suzuki.
\newblock Diffusion models are minimax optimal distribution estimators.
\newblock In \emph{International Conference on Machine Learning}, pages
  26517--26582. PMLR, 2023.

\bibitem[Peebles and Xie(2023)]{peebles2023scalable}
William Peebles and Saining Xie.
\newblock Scalable diffusion models with transformers.
\newblock In \emph{Proceedings of the IEEE/CVF International Conference on
  Computer Vision}, pages 4195--4205, 2023.

\bibitem[Ramesh et~al.(2022)Ramesh, Dhariwal, Nichol, Chu, and
  Chen]{ramesh2022hierarchical}
Aditya Ramesh, Prafulla Dhariwal, Alex Nichol, Casey Chu, and Mark Chen.
\newblock Hierarchical text-conditional image generation with clip latents.
\newblock \emph{arXiv preprint arXiv:2204.06125}, 1\penalty0 (2):\penalty0 3,
  2022.

\bibitem[Ronneberger et~al.(2015)Ronneberger, Fischer, and
  Brox]{ronneberger2015u}
Olaf Ronneberger, Philipp Fischer, and Thomas Brox.
\newblock U-net: Convolutional networks for biomedical image segmentation.
\newblock In \emph{Medical image computing and computer-assisted
  intervention--MICCAI 2015: 18th international conference, Munich, Germany,
  October 5-9, 2015, proceedings, part III 18}, pages 234--241. Springer, 2015.

\bibitem[Schneuing et~al.(2022)Schneuing, Du, Harris, Jamasb, Igashov, Du,
  Blundell, Li{\'o}, Gomes, Welling, et~al.]{schneuing2022structure}
Arne Schneuing, Yuanqi Du, Charles Harris, Arian Jamasb, Ilia Igashov, Weitao
  Du, Tom Blundell, Pietro Li{\'o}, Carla Gomes, Max Welling, et~al.
\newblock Structure-based drug design with equivariant diffusion models.
\newblock \emph{arXiv preprint arXiv:2210.13695}, 2022.

\bibitem[Shah et~al.(2023)Shah, Chen, and Klivans]{shah2023learning}
Kulin Shah, Sitan Chen, and Adam Klivans.
\newblock Learning mixtures of gaussians using the ddpm objective.
\newblock \emph{Advances in Neural Information Processing Systems},
  36:\penalty0 19636--19649, 2023.

\bibitem[Song and Ermon(2019)]{song2019generative}
Yang Song and Stefano Ermon.
\newblock Generative modeling by estimating gradients of the data distribution.
\newblock \emph{Advances in neural information processing systems}, 32, 2019.

\bibitem[Song et~al.(2021)Song, Sohl-Dickstein, Kingma, Kumar, Ermon, and
  Poole]{song2020score}
Yang Song, Jascha Sohl-Dickstein, Diederik~P Kingma, Abhishek Kumar, Stefano
  Ermon, and Ben Poole.
\newblock Score-based generative modeling through stochastic differential
  equations.
\newblock \emph{International Conference on Learning Representations}, 2021.

\bibitem[Song and Yang(2019)]{song2019quadratic}
Zhao Song and Xin Yang.
\newblock Quadratic suffices for over-parametrization via matrix chernoff
  bound.
\newblock \emph{arXiv preprint arXiv:1906.03593}, 2019.

\bibitem[Vincent(2011)]{vincent2011connection}
Pascal Vincent.
\newblock A connection between score matching and denoising autoencoders.
\newblock \emph{Neural computation}, 23\penalty0 (7):\penalty0 1661--1674,
  2011.

\bibitem[Watson et~al.(2023)Watson, Juergens, Bennett, Trippe, Yim, Eisenach,
  Ahern, Borst, Ragotte, Milles, et~al.]{watson2023novo}
Joseph~L Watson, David Juergens, Nathaniel~R Bennett, Brian~L Trippe, Jason
  Yim, Helen~E Eisenach, Woody Ahern, Andrew~J Borst, Robert~J Ragotte, Lukas~F
  Milles, et~al.
\newblock De novo design of protein structure and function with rfdiffusion.
\newblock \emph{Nature}, 620\penalty0 (7976):\penalty0 1089--1100, 2023.

\bibitem[Wibisono et~al.(2024)Wibisono, Wu, and Yang]{wibisono2024optimal}
Andre Wibisono, Yihong Wu, and Kaylee~Yingxi Yang.
\newblock Optimal score estimation via empirical bayes smoothing.
\newblock \emph{arXiv preprint arXiv:2402.07747}, 2024.

\bibitem[Wu et~al.(2023)Wu, FU, Fang, Zhang, Yang, Xiong, Liu, and
  Xu]{wu2023medsegdiff}
Junde Wu, RAO FU, Huihui Fang, Yu~Zhang, Yehui Yang, Haoyi Xiong, Huiying Liu,
  and Yanwu Xu.
\newblock Medsegdiff: Medical image segmentation with diffusion probabilistic
  model.
\newblock In \emph{Medical Imaging with Deep Learning}, 2023.
\newblock URL \url{https://openreview.net/forum?id=Jdw-cm2jG9}.

\bibitem[Yang and Wibisono(2022)]{yang2022convergence}
Kaylee~Yingxi Yang and Andre Wibisono.
\newblock Convergence in {KL} and {R}\'enyi divergence of the unadjusted
  langevin algorithm using estimated score.
\newblock \emph{NeurIPS Workshop on Score-Based Methods}, 2022.

\bibitem[Yang et~al.(2023)Yang, Zhang, Song, Hong, Xu, Zhao, Zhang, Cui, and
  Yang]{yang2023diffusion}
Ling Yang, Zhilong Zhang, Yang Song, Shenda Hong, Runsheng Xu, Yue Zhao, Wentao
  Zhang, Bin Cui, and Ming-Hsuan Yang.
\newblock Diffusion models: A comprehensive survey of methods and applications.
\newblock \emph{ACM Computing Surveys}, 56\penalty0 (4):\penalty0 1--39, 2023.

\bibitem[Yang et~al.(2024)Yang, Wang, Jiang, and Li]{yangconvergence}
Ruofeng Yang, Zhijie Wang, Bo~Jiang, and Shuai Li.
\newblock The convergence of variance exploding diffusion models under the
  manifold hypothesis.
\newblock \emph{OpenReview}, 2024.

\bibitem[Yoon et~al.(2021)Yoon, Hwang, and Lee]{yoon2021adversarial}
Jongmin Yoon, Sung~Ju Hwang, and Juho Lee.
\newblock Adversarial purification with score-based generative models.
\newblock In \emph{International Conference on Machine Learning}, pages
  12062--12072. PMLR, 2021.

\bibitem[Zhu et~al.(2024)Zhu, Chen, Theodorou, Chen, and Tao]{zhu2024quantum}
Yuchen Zhu, Tianrong Chen, Evangelos~A Theodorou, Xie Chen, and Molei Tao.
\newblock Quantum state generation with structure-preserving diffusion model.
\newblock \emph{arXiv preprint arXiv:2404.06336}, 2024.

\bibitem[Zou and Gu(2019)]{zou2019improved}
Difan Zou and Quanquan Gu.
\newblock An improved analysis of training over-parameterized deep neural
  networks.
\newblock \emph{Advances in neural information processing systems}, 32, 2019.

\bibitem[Zou et~al.(2020)Zou, Cao, Zhou, and Gu]{zou2020gradient}
Difan Zou, Yuan Cao, Dongruo Zhou, and Quanquan Gu.
\newblock Gradient descent optimizes over-parameterized deep relu networks.
\newblock \emph{Machine learning}, 109:\penalty0 467--492, 2020.

\end{thebibliography}

%%%%%%%%%%%%%%%%%%%%%%%%%%%%%%%%%%%%%%%%%%%%%%%%%%%%%%%%%%%%

\clearpage
\appendix

\section*{Appendix}

\section{Conclusions and limitations}
\label{app:conclusions_limitations}

\textbf{Conclusions.} In this paper, we provide a first full error analysis incorporating both optimization and sampling processes. For the training process, we provide a first result under a deep neural network and prove the exponential convergence into a neighborhood of minima. At the same time, we extend the current analysis to the variance exploding case for sampling. Moreover, based on the full error analysis, we establish a quantitative understanding of the error bound under the two schedules. Consequently, we conclude with a qualitative illustration of the "bell-shaped" weighting and the choices of schedules under well-trained and less-trained cases.

\textbf{Limitations.}  The network architecture we used in the model is a deep ReLU network. Although being so far the most complicated architecture for theoretical results, it is still far from what is used in practice like U-Nets and transformers. Moreover, regarding the full error analysis, we only focus on the optimization and sampling error and do not dissect the generalization error. When bridging the theoretical results with practical designs of diffusion models, our results are mostly qualitative and we only compare two existing schedules under two extreme cases, when the network is well-trained and less-trained. Thus, theoretical implications on practical designs remain to be explored. We will leave these perspectives for future exploration.

\section{Notations}

\begin{longtable}{p{.13\textwidth} p{.9\textwidth}}
    % \centering
    % \begin{tabular}{cc}
        $X_t$ & Solution of forward dynamics~\eqref{eqn:general_forward_SDE}  \\
        $Y_t$ & Solution of backward dynamics~\eqref{eqn:general_backward_SDE} \\
        $\bar{Y}_t$ & Solution of generation algorithm~\eqref{eqn:EI_backward_SDE}\\
        $\sigma_t$ & Diffusion coefficient of~\eqref{eqn:general_forward_SDE} and~\eqref{eqn:general_backward_SDE} \\
        $\bar{\sigma}_t$ & Standard deviation of $X_t$ \eqref{fct:conti_denoising}\\
        $\mathcal{L}_{\rm conti}$ & Continuous-time score-matching objective \eqref{fct:L_conti}\\
        $\bar{\mathcal{L}}$ & Discrete-time denoising score-matching objective (population version) \\
        $\lem$ & Discrete-time denoising score-matching objective (empirical version) \eqref{fct:lem}\\
        $C_t$ & Constant between score-matching and denoising score-matching loss at time $t$ \eqref{fct:conti_denoising}\\
        $\bar{C} $ & Constant between score-matching and denoising score-matching loss over all discrete times \eqref{eqn:barL+barC}\\
        $x_i$ & Sample from the initial data distribution $P_0$ \eqref{fct:lem}\\
        $X_{ij}$ & Sample from the distribution $P_t$ at time $t$ \eqref{fct:lem}\\
        $t_j$ & The $j$th time point for forward process \eqref{fct:bar_L}\\
        $\backt_j$ & The $j$th time point for backward process \eqref{eqn:EI_backward_SDE}\\
        $\delta$ & The first (last) time point of the forward (backward) dynamics, i.e., $t_0$ \eqref{eqn:VESDE_generation}\\
        $T$ & Stopping time of the forward dynamics \eqref{eqn:VESDE_generation}\\
        $\gamma_j$ & Difference between backward time points, $\backt_{j+1}-\backt_j$ \eqref{eqn:VESDE_generation}\\
        $p_t$ & Density of the solution of forward dynamics at time $t$ (and backward dynamics at time $T-t$) \eqref{eqn:VESDE_generation}\\
        $q_t$ & Density of the solution of the generation algorithm at time $t$ \eqref{eqn:VESDE_generation}\\
        $w(t)$ & Weighting function \eqref{fct:L_conti}\\
        $\beta_j$ & Total weighting, i.e. $w(t_j)(t_j-t_{j-1})/\bar{\sigma}_{t_j}$ \eqref{fct:lem}\\
        $\beta_{\rm EDM}$ & Total weighting used in EDM \citep{karras2022elucidating} \eqref{fct:EDM_training_objective}\\
        $\bar{\sigma}_{\max}\ (\bar{\sigma}_{\min})$ & Maximum (minimum) of $\bar{\sigma}_{t_j}$ (Table~\ref{tab:E_S_EDM_VE})\\
        $n$ & Number of samples from the initial distribution $P_0$ \eqref{fct:lem}\\
        $N$ & Number of time steps when discretizing the forward and backward dynamics \eqref{fct:bar_L}\\
        $d$ & Dimension of input, output data, and the solutions of the dynamics~\eqref{eqn:general_forward_SDE} and~\eqref{eqn:general_backward_SDE}\\
        $S$ & Deep ReLU network (parameterization of score function) \\
        $\theta$ & All the parameters in the network $S$ \eqref{fct:neural network}\\
        $W_\ell$ & The weight in the $i$th layer of the network $S$ \eqref{fct:neural network}\\
        $\theta^{(k)}$ & The $k$th iteration of the weights $\theta$ through GD \eqref{eq:agjvdouiy2g314bori4uybgo1afds}\\
        $W_\ell^{(k)}$ & The $k$th iteration of the weights $W_\ell$ through GD \eqref{eq:agjvdouiy2g314bori4uybgo1afds}\\
        $m$ & Width of the network \eqref{fct:neural network}\\
        $L$ & Depth of the network \eqref{fct:neural network}\\
        $(i^*(s), j^*(s))$ & Index of the largest loss and $w(t_j)(t_j-t_{j-1})\bar{\sigma}_{t_j}$ at the $s$th iteration (Theorem~\ref{thm:convergence_GD_no_interpolation})\\
        $\mathrm{m}_2^2$ & Second moment of the initial distribution $P_0$ \eqref{eqn:VESDE_generation}\\
        $E_I$ & Initialization error \eqref{eqn:VESDE_generation}\\
        $E_D$ & Discretization error \eqref{eqn:VESDE_generation}\\
        $E_S$ & Score error \eqref{eqn:VESDE_generation}\\
        $\epsilon_{\rm train}$ & Optimization error (Theorem~\ref{thm:convergence_GD_no_interpolation})\\
        $\epsilon_n$ & Statistical error (Theorem~\ref{cor:full_error_analysis})\\
        $\epsilon_{\rm est}$ & Estimation error (Theorem~\ref{cor:full_error_analysis})\\
        $\epsilon_{\rm approx}$ & Approximation error (Theorem~\ref{cor:full_error_analysis})\\
        $\theta^*$ & Minimum point of $\bar{\mathcal{L}}$ when $\lem=0$ (Theorem~\ref{cor:full_error_analysis})\\
        $\theta_\mathcal{F}$ & Optimal parameter in the function class (Theorem~\ref{cor:full_error_analysis})\\
        
    % \end{tabular}
    % \caption{Caption}
    % \label{tab:my_label}
\end{longtable}

\section{Derivation of denoising score matching objective}
\label{app:denoising_score_matching_objective}

In this section, we will derive the denoising score matching objective, i.e. show the equivalence of \eqref{fct:L_conti} and \eqref{fct:conti_denoising}. For simplicity, we denote $S_\theta$ to be the neural network we use $S(\theta;t,X_t)$. 

Consider
\begin{align}
    \label{appeqn:score_obj}\mathbb{E}_{X_t\sim P_t}\|S(\theta;t,X_t)-\nabla\log p_t\|^2=\mathbb{E}_{X_t}\left[ \|S_\theta\|^2-2\ip{S_\theta}{\nabla\log p_t} \right]+\mathbb{E}_{X_t}\|\nabla\log p_t\|^2,
\end{align}
where $p_t$ is the density of $X_t$.

Since $p_t(x)=\int p_0(y)q_t(x|y)dy$, where $q_t(\cdot)$ is the density of $X_t|X_0$, then we have
\begin{align*}
    \mathbb{E}_{X_t}\ip{S_\theta}{\nabla\log p_t}&=\int S_\theta^\top\nabla\log p_t\cdot p_t\, dx_t\\
    &=\int S_\theta^\top\nabla p_t\, dx_t\\
    &=\iint S_\theta^\top \nabla q_t(x|y)p_0(y)\,dxdy\\
    &=\iint S_\theta^\top \nabla \log q_t(x|y)p_0(y)q_t(x|y)\,dxdy\\
    &=\mathbb{E}_{X_0\sim P_0}\mathbb{E}_{X_t|X_0\sim Q_t}\ip{S_\theta}{\nabla\log q_t(x_t|x_0)}.
\end{align*}
Then
\begin{align*}
    \eqref{appeqn:score_obj}&=\mathbb{E}_{X_0\sim P_0}\mathbb{E}_{X_t|X_0\sim Q_t}\left[  \|S_\theta\|^2-2\ip{S_\theta}{\nabla\log q_t(x_t|x_0)}\right]+\mathbb{E}_{X_t}\|\nabla\log p_t\|^2\\
    &=\underbrace{\mathbb{E}_{X_0}\mathbb{E}_{X_t|X_0} \|S_\theta-\nabla\log q_t\|^2}_{(\Delta)}+C,
\end{align*}
where $C=\mathbb{E}_{X_t}\|\nabla\log p_t\|^2-\mathbb{E}_{X_0}\mathbb{E}_{X_t|X_0}\|\nabla\log q_t(x_t|x_0)\|^2$.

Moreover, $X_t|X_0\sim\mathcal{N}(e^{-\mu_t}X_0,\bar{\sigma}_t^2I)$, and its density function is
\begin{align*}
    q_t(x|y)=(2\pi\bar{\sigma}_t^2)^{-d/2}\exp\left(-\frac{\|x-e^{-\mu_t}y\|^2}{2\bar{\sigma}_t^2}\right).
\end{align*}
Then
\begin{align*}
    (\Delta)&=\mathbb{E}_{X_0}\mathbb{E}_{X_t|X_0} \|S_\theta-\nabla\log q_t\|^2\\
    &=\mathbb{E}_{X_0}\mathbb{E}_{X_t|X_0} \left\|S_\theta-\nabla_x\left(-\frac{\|X_t-e^{-\mu_t}X_0\|^2}{2\bar{\sigma}_t^2}\right)\right\|^2\\
    &=\mathbb{E}_{X_0}\mathbb{E}_{X_t|X_0} \left\|S_\theta+\frac{X_t-e^{-\mu_t}X_0}{\bar{\sigma}_t^2}\right\|^2\\
    &=\mathbb{E}_{X_0}\mathbb{E}_{\epsilon_t} \left\|S_\theta+\frac{\epsilon_t}{\bar{\sigma}_t^2}\right\|^2.
\end{align*}
Let $\xi=\frac{\epsilon_t}{\bar{\sigma}_t}\sim\mathcal{N}(0,I)$. Then
\begin{align*}
    (\Delta)&=\mathbb{E}_{X_0}\mathbb{E}_{\xi}\bar{\sigma}_t\cdot \frac{1}{\bar{\sigma}_t^2}\|\bar{\sigma}_t S_\theta+\xi\|^2\\
    &=\frac{1}{\bar{\sigma}_t}\mathbb{E}_{X_0}\mathbb{E}_{\xi}\|\bar{\sigma}_t S_\theta+\xi\|^2
\end{align*}

\section{Proofs for training}
\label{app:training_proof}

In this section, we will prove Theorem~\ref{thm:convergence_GD_no_interpolation}.

Before introducing the concrete proof, we first redefine the deep fully connected feedforward network
\begin{align*}
    &r_{ij,0}=W_0X_{ij}, q_{ij,0}=\sigma(r_{ij,0}), \\
    &r_{ij,\ell}=W_\ell q_{ij,\ell-1}, q_{ij,\ell}=\sigma(r_{ij,\ell}), \text{ for } \ell=1,\cdots,L\\
    &S(\theta;t_j,X_{ij})=W_{L+1}q_{ij,L}
\end{align*}
where $W_0\in\mathbb{R}^{m\times d},W_{L+1}\in\mathbb{R}^{d\times m}$ and $W_\ell\in\mathbb{R}^{m\times m}$; $\sigma$ is the ReLU activation. We also denote $q_{ij,-1}$ to be $X_{ij}$.

We also follow the notation in~\citet{allen2019convergence} and denote $D_{i,\ell}\in\mathbb{R}^{m\times m}$ to be a diagonal matrix and $(D_{i,\ell})_{kk}=\mathbbm{1}_{(W_\ell q_{ij,\ell-1})_k>0}$ for $k=1,\cdots,m$. Then
\begin{align*}
    q_{ij,\ell}=D_{ij,\ell}W_\ell q_{ij,\ell-1}
\end{align*}

For the objective~\eqref{fct:lem}, the gradient w.r.t. to the $k$th row of $W_\ell$ for $\ell=1,\cdots,L$ is the following
\begin{align*}
    \nabla_{(W_\ell)_k}\lem(\theta)&=\frac{1}{n}\sum_{i=1}^n\sum_{j=1}^{N}w(t_j)(t_j-t_{j-1})\\
    &\qquad[(\underbrace{W_{L+1}D_{ij,L}W_L\cdots D_{ij,\ell}W_{\ell+1}}_{R_{ij,\ell+1}})^\top(\bar{\sigma}_{t_j}W_{L+1}q_{ij,L}+\xi_{ij})]_k\,q_{ij,\ell-1}\, \mathbbm{1}_{(W_\ell q_{ij,\ell-1})_k>0}
\end{align*}
Throughout the proof, we use both $\lem(\theta)$ and $\lem(W)$ to represent the same value of the loss function, where $W=(W_1,\cdots,W_L)$, and we let $a=b=\frac{1}{2}$.

Next, we will prove Theorem~\ref{thm:convergence_GD_no_interpolation}.

\begin{proof}[Proof of Theorem~\ref{thm:convergence_GD_no_interpolation}]
    First by Lemma~\ref{lem:scale_Xij_qij},
    \begin{align*}
        \lem(W^{(0)})=\mathcal{O}(d^{2a}\sum_{j}w(t_j)(t_j-t_{j-1})/\bar{\sigma}_{t_j})
    \end{align*}
    Also, $\|\nabla\lem(\theta)\|\le\sqrt{L}\max_\ell \|\nabla_{W_\ell}\lem(\theta)\|$. Then we have
    \begin{align*}
    \|W^{(k)}-W^{(0)}\|&\le \sum_{i=0}^{k-1}h \|\nabla\lem(W^{(i)})\|\\
    &\le\mathcal{O}(\sqrt{md^{2a-1}NL\max_j w(t_j)(t_j-t_{j-1})\bar{\sigma}_{t_j}})h  k \max_i\sqrt{\lem(W^{(i)})}\\
    &\le\mathcal{O}(\sqrt{md^{2a-1}NL\max_j w(t_j)(t_j-t_{j-1})\bar{\sigma}_{t_j}}\ d^a)h  k \sqrt{\sum_{j}w(t_j)(t_j-t_{j-1})/\bar{\sigma}_{t_j}}:=\omega
\end{align*}

Let $h =\Theta(\frac{nN}{m\min_j w(t_j)(t_j-t_{j-1})\bar{\sigma}_{t_j}})$ and $k=\mathcal{O}(d^{\frac{1-a_0}{2}}n^2N\log(\frac{d}{\epsilon_{\rm train}}))$, where $\epsilon_{\rm train}>0$ is some small constant. 
% let $h =\Theta(\frac{1}{md^{2a+1+12c}n^7N^6L^6})$ and $k\le\mathcal{O}(\frac{n^2N}{h  m(\log d)^{3/2}d^{2a-1.5-4c}}\log(nN))=\mathcal{O}(\frac{d^{2.5+16c}n^9N^5L^6}{(\log d)^{3/2}}\log(nN))$. 
 Then $\omega=\mathcal{O}(\log (\frac{d}{\epsilon_{\rm train}})\frac{d^{1-\frac{a_0}{2}}n^3N^{5/2}L^{1/2}}{\sqrt{m}}\frac{\sqrt{\max_j w(t_j)(t_j-t_{j-1})\bar{\sigma}_{t_j}\sum_k w(t_j)(t_j-t_{j-1})/\bar{\sigma}_{t_j}}}{\min_j w(t_j)(t_j-t_{j-1})\bar{\sigma}_{t_j}})$ and by Lemma~\ref{lem:semi-smoothness}, with probability at least $1-\mathcal{O}(nN)\exp(-\Omega(d^{2a_0-1}))$,
    \begin{align*}
         &\lem(W^{(k+1)})\\
         &\le \lem(W^{(k)})-h \|{\nabla \lem(W^{(k)})}\|^2\\
         &+h \sqrt{\lem}\sqrt{\sum_{j}w(t_j)(t_j-t_{j-1})\bar{\sigma}_{t_j}}\mathcal{O}(\omega^{1/3}L^2\sqrt{m\log m}d^{a/2})\|\nabla \lem(W^{(k)})\|\\
        &+h ^2\sqrt{\lem}\sqrt{\sum_{j}w(t_j)(t_j-t_{j-1})\bar{\sigma}_{t_j}}\mathcal{O}(L^2\sqrt{m}d^{a})\|\nabla \lem(W^{(k)})\|^2\\
        &\le \left(1-h  w(t_{j^*})(t_{j^*}-t_{j^*-1})\bar{\sigma}_{t_{j^*}}\cdot\Omega\left(\frac{md^{\frac{a_0-1}{2}}}{n^3N^2} \right)\right)\lem(W^{(k)})\\
        &+h C\frac{m^{5/6}d^{7/12-a_0/6}}{N^{1/6}n^{2/3}(\log m)^{1/6}\sqrt{L}}\frac{\sqrt{\sum_{j}w(t_j)(t_j-t_{j-1})\bar{\sigma}_{t_j}}\min_j w(t_j)(t_j-t_{j-1})\bar{\sigma}_{t_j}}{\max_j w(t_j)(t_j-t_{j-1})\bar{\sigma}_{t_j}\sum_k w(t_j)(t_j-t_{j-1})/\bar{\sigma}_{t_j}}\lem(W^{(k)})\\
        &\le \left(1-h  w(t_{j^*})(t_{j^*}-t_{j^*-1})\bar{\sigma}_{t_{j^*}}\cdot\Omega\left(\frac{md^{\frac{a_0-1}{2}}}{n^3N^2} \right)\right)\lem(W^{(k)})\\
        % &+\underbrace{h \lem^*\frac{L^{3/2}\sqrt{\log m}n^2N^{3/2}}{m^{1/2-4c}}\sqrt{\max_j w(t_j)(t_j-t_{j-1})\bar{\sigma}_{t_j}}\sqrt{\sum_{j}w(t_j)(t_j-t_{j-1})\bar{\sigma}_{t_j}} \mathcal{O}\left(\frac{m\sqrt{\log d}d^{2b-1.5-4c}}{n^2N}\right)}_{h  g}
    \end{align*}
    where $C>0$ is some constant, $a_0\in(1/2,1)$; the second inequality follows from Lemma~\ref{lem:perturb_upper_lower_bound_grad} with
    \begin{align*}
        \|\nabla_{W_L} \lem({\theta^{(k)}})\|^2=\Omega\left(\frac{md^{\frac{a_0-1}{2}}}{n^3N^2} \ w(t_{j^*})(t_{j^*}-t_{j^*-1})\bar{\sigma}_{t_{j^*}}\right) \lem (\theta^{(k)}),
    \end{align*}
    which is obtained inductively; the last inequality follows from $m=\Omega\left(d^{13/2-2a_0/3}n^{14/3}N^{11}L^3(\log m)\left(\frac{\max_j w(t_j)(t_j-t_{j-1})\bar{\sigma}_{t_j}\sum_k w(t_j)(t_j-t_{j-1})/\bar{\sigma}_{t_j}}{\min_j w(t_j)(t_j-t_{j-1})\bar{\sigma}_{t_j}\sqrt{\sum_{j}w(t_j)(t_j-t_{j-1})\bar{\sigma}_{t_j}}}\right)^6\right)$
    
%     By discrete Gronwall's inequality,
%     we have
%     \begin{align*}
%         &\lem(W^{(k+1)})-\lem^*\\
%         &\le \prod_{s=0}^k\left(1-h \ w(t_{j^*(s)})(t_{j^*(s)+1}-t_{j^*(s)})\bar{\sigma}_{t_{j^*(s)}} \cdot\Omega\left(\frac{m\sqrt{\log d}d^{2b-1.5-4c}}{n^2N}\right)\right)\cdot(\lem(W^{(0)})-\lem^*)\\
%         &+h  g\sum_{s=0}^{k-1}\prod_{l=s+1}^k\left(1-h  w(t_{j^*(l)})(t_{j^*(l)+1}-t_{j^*(l)})\bar{\sigma}_{t_{j^*(l)}}\cdot\Omega\left(\frac{m\sqrt{\log d}d^{2b-1.5-4c}}{n^2N}\right)\right)\\
%         &\le\prod_{s=0}^k\left(1-h \ w(t_{j^*(s)})(t_{j^*(s)+1}-t_{j^*(s)})\bar{\sigma}_{t_{j^*(s)}} \cdot\Omega\left(\frac{m\sqrt{\log d}d^{2b-1.5-4c}}{n^2N}\right)\right)\cdot(\lem(W^{(0)})-\lem^*)\\
%         &+h  g\sum_{s=0}^{k-1}\left(1-h  \min_j w(t_{j})(t_j-t_{j-1})\bar{\sigma}_{t_{j}}\cdot\Omega\left(\frac{m\sqrt{\log d}d^{2b-1.5-4c}}{n^2N}\right)\right)^{k-s-1}\\
%         &\le\prod_{s=0}^k\left(1-h \ w(t_{j^*(s)})(t_{j^*(s)+1}-t_{j^*(s)})\bar{\sigma}_{t_{j^*(s)}} \cdot\Omega\left(\frac{m\sqrt{\log d}d^{2b-1.5-4c}}{n^2N}\right)\right)\cdot(\lem(W^{(0)})-\lem^*)\\
%         &+\mathcal{O}\left(\frac{L^{3/2}\sqrt{\log m}n^2N^{2}}{m^{1/2-4c}}
% \cdot\frac{\max_j w(t_{j})(t_j-t_{j-1})\bar{\sigma}_{t_{j}}}{\min_j w(t_{j})(t_j-t_{j-1})\bar{\sigma}_{t_{j}}}\right)
%     \end{align*}
\end{proof}

\subsection{Proof of lower bound of the gradient at the initialization}
\label{app:subsec:lower_bound_grad}
In this section, we will show the main part of the convergence analysis, which is the following lower bound of the gradient.
\begin{lemma}[Lower bound]
\label{lem:lower_bound}
    With probability $1-\mathcal{O}(nN)\exp(-\Omega(d^{2a_0-1}))$, we have
    \begin{align*}
        \|\nabla\lem({\theta^{(0)}})\|^2\ge C_6\left(\frac{md^{\frac{a_0-1}{2}}}{n^3N^2} \ w(t_{j^*})(t_{j^*}-t_{j^*-1})\bar{\sigma}_{t_{j^*}}\right)\lem(\theta^{(0)})
    \end{align*}
    where $(i^*,j^*)=\arg\max \left\|\sqrt{\frac{w(t_j)(t_j-t_{j-1})}{\bar{\sigma}_{t_j}}}(\bar{\sigma}_{t_j}W_{L+1}q_{ij,L}+\xi_{ij})\right\|$, $\frac{1}{2}<a_0<1$, and $C_6>0$ is some universal constant.
\end{lemma}

Below is the proof sketch of Lemma~\ref{lem:lower_bound}.

\begin{proof}[Proof sketch]

    We first decompose the gradient of the $k$th row of $W_L$ $\nabla_{(W_L)_k}\bar{\mathcal{L}}_{\rm em}(\theta)=\underbrace{\frac{1}{n}w(t_{j^*})(t_{j^*}-t_{j^*-1}){(W_{L+1})^k}^\top(\bar{\sigma}_{t_{j^*}}W_{L+1} q_{i^* j^*,L}+\xi_{i^* j^*})q_{i^* j^*,L-1} {1}_{(W_L q_{i^* j^*,L-1})_k>0}}_{\nabla_1}$
    $+\underbrace{\frac{1}{n}\sum_{(i,j)\ne (i^*,j^*)}w(t_j)(t_j-t_{j-1}){(W_{L+1})^k}^\top(\bar{\sigma}_{t_j}W_{L+1}q_{ij,L}+\xi_{ij})\,q_{ij,L-1} {1}_{(W_L q_{ij,L-1})_k>0}}_{\nabla_2}$

where $(i^*,j^*)$ indicates the sample index with the largest loss value.

Then we first fix $(q_{ij,L-1})_s=1$, and prove that the index set of both $(q_{i^*j^*,L})_s> 0$ and $\sum_{(i,j)\ne (i^*,j^*)}w(t_j)(t_j-t_{j-1})\bar{\sigma}_{t_j}{1}
_{(W_Lq_{ij,L-1})_k>0}(q_{ij,L})_s>0$ is order $m$ with high probability.

Next, we conditioned on the index set we've found, then we can decouple each element of $\nabla_{(W_L)_k}\bar{\mathcal{L}}_{\rm em}$ with high probability. We prove that with high probability
\begin{align*}
    \angle(W_{L+1}\bar{\sigma}_{t_{j^*}}q_{i^*j^*,L}+\xi_{i^*j^*},W_{L+1}\sumij \alpha_{ij} q_{ij,L}+\sumij \bar{\alpha}_{ij}\xi_{ij})\le \pi-c d^{\frac{a_0-1}{2}},
\end{align*}
for some constant $c>0$ and $\frac{1}{2}<a_0<1$. Based on this, we show that with probability at least $1-\mathcal{O}(nN)\exp(-\Omega(d))$,
\begin{align*}
    \mathbb{P}\left((\nabla_1)_s> 0,(\nabla_2)_s> 0\right)\ge cd^{\frac{a_0-1}{2}},
\end{align*}
for some $c>0$. Then we prove that with probability at least $1-\exp(-\Omega(md^{\frac{a_0-1}{2}}))$ 
\begin{align*}
    |\{k:(W_{L+1}^k)^\top v\ge 0,(W_{L+1}^k)^\top (u+\xi)\ge 0\}|=\Theta(md^{\frac{a_0-1}{2}}).
\end{align*}

with high probability, the event $(\nabla_1)_s> 0$ and $(\nabla_2)_s>0$ has probability at least of order $d^{(a_0-1)/2}$ where $a_0\in(1/2,1)$.

Now, we deal with $(q_{ij,L-1})_s$ and prove that if the above results hold for $(q_{ij,L-1})_s=1$, then there exists an index set with cardinality of order $m/(nN)$ such that $(\nabla_1)_s> 0$ and $(\nabla_2)_s>0$ also hold in this index set.

In the end, combining all the steps above yields the lower bound.
\end{proof}

Here is the complete proof.

\begin{proof}
    The main idea of the proof of lower bound is to decouple the elements in the gradient and incorporate geometric view. We focus on $\nabla_{W_L}\lem(\theta)$.

\textbf{Step 1}: Rewrite $\nabla_{(W_L)_k}\lem(\theta)$ to be the $(i^*,j^*)$th term $g_1$ plus the rest $nN-1$ terms $g_2$.

Let $(i^*,j^*)=\arg\max \left\|\sqrt{\frac{w(t_j)(t_j-t_{j-1})}{\bar{\sigma_{t_j}}}}(\bar{\sigma}_{t_j}W_{L+1}q_{ij,L}+\xi_{ij})\right\|$. Let $$g_{ij,L}=w(t_j)(t_j-t_{j-1}){(W_{L+1})^k}^\top(\bar{\sigma}_{t_j}W_{L+1}q_{ij,L}+\xi_{ij})\,q_{ij,L-1}.$$ Then
\begin{align*}
    &\nabla_{(W_L)_k}\lem(\theta)
    \\
    =&\underbrace{\frac{1}{n}w(t_{j^*})(t_{j^*}-t_{j^*-1}){(W_{L+1})^k}^\top(\bar{\sigma}_{t_{j^*}}W_{L+1}q_{i^*j^*,L}+\xi_{i^*j^*})\,q_{i^*j^*,L-1} \mathbbm{1}_{(W_L q_{i^*j^*,L-1})_k>0}}_{\nabla_1}\\
    &+\underbrace{\frac{1}{n}\sum_{(i,j)\ne (i^*,j^*)}w(t_j)(t_j-t_{j-1}){(W_{L+1})^k}^\top(\bar{\sigma}_{t_j}W_{L+1}q_{ij,L}+\xi_{ij})\,q_{ij,L-1} \mathbbm{1}_{(W_L q_{ij,L-1})_k>0}}_{\nabla_2}
\end{align*}
Also define
\begin{align*}
    \nabla_{1,s}=&\underbrace{\frac{1}{n}w(t_{j^*})(t_{j^*}-t_{j^*-1})\bar{\sigma}_{t_{j^*}}\ {(W_{L+1})^k}^\top W_{L+1}q_{i^*j^*,L}\, (q_{i^*j^*,L-1})_s\mathbbm{1}_{(W_L q_{i^*j^*,L-1})_k>0}}_{\nabla_{11,s}}\\
    &+\underbrace{\frac{1}{n}w(t_{j^*})(t_{j^*}-t_{j^*-1})\ {(W_{L+1})^k}^\top\xi_{i^*j^*}\, (q_{i^*j^*,L-1})_s\mathbbm{1}_{(W_L q_{i^*j^*,L-1})_k>0}}_{\nabla_{12,s}}
\end{align*}
\begin{align*}
    \nabla_{2,s}=&\underbrace{\frac{1}{n}\sum_{(i,j)\ne (i^*,j^*)}w(t_j)(t_j-t_{j-1})\bar{\sigma}_{t_j}{(W_{L+1})^k}^\top W_{L+1}q_{ij,L}\,(q_{ij,L-1})_s \mathbbm{1}_{(W_L q_{ij,L-1})_k>0}}_{\nabla_{21,s}}\\
    &+\underbrace{\frac{1}{n}\sum_{(i,j)\ne (i^*,j^*)}w(t_j)(t_j-t_{j-1}){(W_{L+1})^k}^\top\xi_{ij}\,q_{ij,L-1}(q_{ij,L-1})_s \mathbbm{1}_{(W_L q_{ij,L-1})_k>0}}_{\nabla_{22,s}}
\end{align*}
Our goal is to show that with high probability, there are at least $\mathcal{O}(\frac{md^{\frac{a_0-1}{2}}}{nN})$ number of rows $k$ such that $\nabla_{11,s}\ge 0, \nabla_{12,s}\ge 0,\nabla_{21,s}\ge 0,\nabla_{22,s}\ge 0$. Then we can lower bound $\|\nabla_{(W_L)_k}\lem(\theta)\|^2$ by $\|\nabla_1\|^2$, which can be eventually lower bounded by $\lem(\theta)$.

\textbf{Step 2}: Consider $[\nabla_{(W_L)_k}\lem(\theta)]_s$. For $(g_2)_s$, first take $(q_{ij,L-1})_s=1$ for all $(i,j)\ne(i^*,j^*)$. Then we only need to consider 
\begin{align*}
    \nabla_{2,s}'=\frac{1}{n}\sum_{(i,j)\ne (i^*,j^*)}w(t_j)(t_j-t_{j-1}){(W_{L+1})^k}^\top(\bar{\sigma}_{t_j}W_{L+1}q_{ij,L}+\xi_{ij})\,\mathbbm{1}_{(W_L q_{ij,L-1})_k>0}
\end{align*}
which is independent of $s$. For $\nabla_1$, since $q_{i^*,j^*,L-1}\ge 0$ which does not affect the sign of this term, we can also first take $(q_{i^*,j^*,L-1})_s=1$ for all $s$.

\textbf{Step 3}: We focus on $\nabla_{11}$ and $\nabla_{21}$ and we would like to pick the non-zero elements in this two terms. More precisely, let
\begin{align*}
    &N_1=\left\{s\ |\ (q_{i^*j^*,L})_s>0,s=1,\cdots,m\right\},\\
    &N_2=\left\{s\ \Big|\ \sumij w(t_j)(t_j-t_{j-1})\bar{\sigma}_{t_j}\mathbbm{1}_{(W_Lq_{ij,L-1})_k>0}(q_{ij,L})_s>0,s=1,\cdots,m \right\}
\end{align*}
Let $\alpha_{ij}=w(t_j)(t_j-t_{j-1})\bar{\sigma}_{t_j}\mathbbm{1}_{(W_Lq_{ij,L-1})_k>0}\ge 0$. Then 
\begin{align*}
    &\sumij w(t_j)(t_j-t_{j-1})\bar{\sigma}_{t_j}\mathbbm{1}_{(W_Lq_{ij,L-1})_k>0}(q_{ij,L})_s\\
    &=\sumij \alpha_{ij}(q_{ij,L})_s=\sumij \alpha_{ij}\sigma(W_Lq_{ij,L-1})_s.
\end{align*}
If 
\begin{align*}
    \sumij \alpha_{ij}(W_Lq_{ij,L-1})_s=(W_L)_s\sumij \alpha_{ij}q_{ij,L-1}>0,
\end{align*}
then there must be at least one pair of $(i,j)$ s.t. $\alpha_{ij}(W_Lq_{ij,L-1})_s=\alpha_{ij}\sigma(W_Lq_{ij,L-1})_s>0$, which implies $\sumij \alpha_{ij}(q_{ij,L})_s>0$. Therefore, it suffices to consider 
\begin{align*}
    &N_1=\left\{s\ |\ (q_{i^*j^*,L})_s=(W_L)_s q_{i^*j^*,L-1}>0,s=1,\cdots,m\right\},\\
    &N_2'=\left\{ s\ |\ (W_L)_s\sumij \alpha_{ij}q_{ij,L-1}>0 \right\}.
\end{align*}
Since $(q_{ij,L-1})_s\ge 0$, we have
\begin{align*}
    &\ip{q_{i^*j^*,L-1}}{\sumij\alpha_{ij}q_{ij,L-1}}\ge 0,\\
    &i.e., \angle \left(q_{i^*j^*,L-1}, \sumij\alpha_{ij}q_{ij,L-1}\right)\le \frac{\pi}{2}
\end{align*}
By Lemma~\ref{lem:prob_wv1>0_wv2>0} and Proposition~\ref{prop:gaussian_norm_direction_indep_uniform}, we have
\begin{align*}
    &\mathbb{P}\left((W_L)_s q_{i^*j^*,L-1}>0, (W_L)_s\sumij \alpha_{ij}q_{ij,L-1}>0\right)\\
    &=\mathbb{P}\left(\frac{(W_L)_s}{\|(W_L)_s\|} q_{i^*j^*,L-1}>0, \frac{(W_L)_s}{\|(W_L)_s\|} \sumij \alpha_{ij}q_{ij,L-1}>0\right)\\
    &\ge \frac{1}{4}.
\end{align*}
Also $(W_L)_s$'s are $i.i.d.$ multivariate Gaussian. By Chernoff bound,
\begin{align*}
    \mathbb{P}\left( |N_1\cap N_2'|\in (\delta_1 \frac{m }{4},\delta_2 \frac{m }{4})  \right)\le 1-2e^{-\Omega(m)}
\end{align*}
for some small $\delta_1\le \frac{1}{4}$ and $\delta_2\le4$, i.e., $|N_1\cap N_2'|=\Theta(m)$ with probability at least $1-2e^{-\Omega(m)}$.

\textbf{Step 4}: Next we condition on $N_1\cap N_2'$ and consider ${(W_{L+1})^k}^\top W_{L+1}\bar{\sigma}_{t_{j^*}}q_{i^*j^*,L}+{(W_{L+1})^k}^\top\xi_{i^*j^*}$ and ${(W_{L+1})^k}^\top W_{L+1}\sumij \alpha_{ij} q_{ij,L}+{(W_{L+1})^k}^\top\sumij \bar{\alpha}_{ij}\xi_{ij}$, where $\bar{\alpha}_{ij}=\alpha_{ij}/\bar{\sigma}_{ij}$.

We would like to prove that with high probability
\begin{align*}
    \angle(W_{L+1}\bar{\sigma}_{t_{j^*}}q_{i^*j^*,L}+\xi_{i^*j^*},W_{L+1}\sumij \alpha_{ij} q_{ij,L}+\sumij \bar{\alpha}_{ij}\xi_{ij})\le \pi-c d^{\frac{a_0-1}{2}},
\end{align*}
for some constant $c>0$ and $\frac{1}{2}<a_0<1$.

First, since $\xi_{ij}\sim\mathcal{N}(0,I_d)$, by Bernstein's inequality, with probability at least $1-\exp(-\Omega(d))$, we have $\|\xi_{ij}\|^2=\Theta(d)$. Similarly, since $(W_{L+1}q_{ij,L})_s\sim\mathcal{N}(0,\frac{2\|q_{ij,L}\|^2}{m})$, by Berstein's inequality and Lemma~\ref{lem:scale_Xij_qij}, with probability at least $1-\exp(-\Omega(d))$, we have $\|W_{L+1}q_{ij,L}\|^2=\Theta(d)$. By union bounds, the above holds for all $i,j$ with probability at least $1-2nN\exp(-\Omega(d))$.

Let $v=\frac{W_{L+1}\sumij \alpha_{ij} q_{ij,L}+\sumij \bar{\alpha}_{ij}\xi_{ij}}{\|W_{L+1}\sumij \alpha_{ij} q_{ij,L}+\sumij \bar{\alpha}_{ij}\xi_{ij}\|}$ and $u=W_{L+1}\bar{\sigma}_{t_{j^*}}q_{i^*j^*,L}$. For notational simplicity, we use $\xi$ to denote $\xi_{i^*j^*}$. Fix $v,u$ and consider the probability of event $A$ $$A=\{v^\top (u+\xi)\le - \sqrt{1-c_0d^{a_0-1}}\|u+\xi\|\}$$ for some $c_0>0$ and $\frac{1}{2}<a_0<1$.

Then consider the following event that has larger probability than $A$
\begin{align}
    &(v^\top (u+\xi))^2\ge (1-c_0d^{a_0-1})\|u+\xi\|^2\\
    \iff& (v^\top u)^2-(1-c_0d^{a_0-1})\|u\|^2+2(v^\top u\,v-(1-c_0d^{a_0-1})u)^\top\xi+(v^\top\xi)^2\ge (1-c_0d^{a_0-1})\|\xi\|^2\label{eqn:angle}
\end{align}
Since $(v^\top u)^2\le \|u\|^2$ where the equality holds when $v=\frac{u}{\|u\|}$, we have
\begin{align*}
    \mathrm{LHS}\le c_0d^{a_0-1}\|u\|^2+2(v^\top u\,v-(1-c_0d^{a_0-1})u)^\top\xi+(v^\top\xi)^2
\end{align*}
Also, since $\|u\|^2=\mathcal{O}(d)$ with probability at least $1-2nN\exp(-\Omega(d))$, we have
\begin{align*}
    \mathbb{P}\left(|2(v^\top u\,v-(1-c_0d^{a_0-1})u)^\top\xi|\ge d^{a_0} \right)\le 2\exp\left(-c\frac{d^{2a_0}}{\|u\|^2}\right)=2\exp(-\Omega(d^{2a_0-1}))
\end{align*}
for some constant $c>0$.

Therefore, with probability at least $1-\mathcal{O}(nN)\exp(-\Omega(d^{2a_0-1}))$
\begin{align*}
    \text{LHS of }\eqref{eqn:angle}\le cd^{a_0}+(v^\top\xi)^2
\end{align*}
for some constant $c>0$.

Then
\begin{align*}
    &\mathbb{P}(v^\top (u+\xi)\le - \sqrt{1-c_0d^{a_0-1}}\|u+\xi\|)\\
    &\le \mathbb{P}((v^\top (u+\xi))^2\ge (1-c_0d^{a_0-1})\|u+\xi\|^2)\\
    &\le \mathbb{P}((v^\top\xi)^2\ge (1-c'd^{a_0-1})\|\xi\|^2)\\
    &=\mathbb{P}\left(v^\top\frac{\xi}{\|\xi\|}\ge \sqrt{(1-c'd^{a_0-1})}\right)+\mathbb{P}\left(-v^\top\frac{\xi}{\|\xi\|}\ge \sqrt{(1-c'd^{a_0-1})}\right)\\
    &=\mathbb{P}\left(\angle(v,\frac{\xi}{\|\xi\|})\ge \arccos(-\sqrt{(1-c'd^{a_0-1})})\right)+\mathbb{P}\left(\angle(-v,\frac{\xi}{\|\xi\|})\ge \arccos(-\sqrt{(1-c'd^{a_0-1})})\right)\\
    &=2\mathbb{P}\left(\angle(v,\frac{\xi}{\|\xi\|})\ge \pi-c'' d^{\frac{a_0-1}{2}}\right)\\
    &=\frac{C}{(d^{\frac{1-a_0}{2}})^{d-1}\sqrt{d}}
\end{align*}
where the second equality follows from Lemma~\ref{prop:gaussian_norm_direction_indep_uniform}; the third equality follows from series expansion; the forth equality follows from \eqref{eqn:probability_angle_cap}; $c',c'',C>0$ are some constants.

Thus with probability at least $1-\frac{C}{(d^{\frac{1-a_0}{2}})^{d-1}\sqrt{d}}$,
\begin{align*}
    &v^\top (u+\xi)\ge - \sqrt{1-c_0d^{a_0-1}}\|u+\xi\|\\
    i.e.& \angle(v,u+\xi)\le \pi-cd^{\frac{a_0-1}{2}}
\end{align*}
for some $c>0$.

Then by Lemma~\ref{lem:prob_wv1>0_wv2>0}, with probability at least $1-\mathcal{O}(nN)\exp(-\Omega(d))$,
\begin{align*}
    \mathbb{P}\left((W_{L+1}^k)^\top v\ge 0,(W_{L+1}^k)^\top (u+\xi)\ge 0\right)\ge cd^{\frac{a_0-1}{2}},
\end{align*}
for some $c>0$.

Since $(W_{L+1}^k)$ are iid Guassian vectors for $k=1,\cdots,m$, by Chernoff bound on Bernoulli variable $\mathbbm{1}_{\{(W_{L+1}^k)^\top v\ge 0,(W_{L+1}^k)^\top (u+\xi)\ge 0\}}$, we have, with probability at least $1-\exp(-\Omega(md^{\frac{a_0-1}{2}}))$ 
\begin{align*}
    |\{k:(W_{L+1}^k)^\top v\ge 0,(W_{L+1}^k)^\top (u+\xi)\ge 0\}|=\Theta(md^{\frac{a_0-1}{2}}).
\end{align*}

\textbf{Step 5:} Combining the above 4 steps, we would like to obtain the lower bound of the gradient.

For each $k$, consider $(q_{ij,L-1})_s$ for $(i,j)\ne (i^*,j^*)$ and denote $q_s=\{(q_{ij,L-1})_s\}_{(i,j)\ne (i^*,j^*)}=\{\sigma((W_{L-1})_sq_{ij,L-2})\}_{(i,j)\ne (i^*,j^*)}$. Let $\bar{q}_s=\{(W_{L-1})_sq_{ij,L-2}\}_{(i,j)\ne (i^*,j^*)}$ and $\bar{q}_s\sim\mathcal{N}(0,QQ^\top)$, where each row of $Q$ is $q_{ij,L-2}^\top$ for $(i,j)\ne (i^*,j^*)$. Thus, $q_s$ is $\bar{q}_s$ projected to the nonnegative orthant.

Let $\mathbf{1}=(1,1,\cdots,1)\in\mathbb{R}^{nN-1}$. Therefore, if $\ip{\beta_k}{\mathbf{1}}\ge 0$ for some $\beta_k\in\mathbb{R}^{nN-1}$, then at least half of the nonnegative orthant is contained in $\{v\in\mathbb{R}^{nN-1}:\ip{\beta_k}{v}\ge 0\}$, i.e., there exists a constant $c_k>0$, s.t.
\begin{align*}
    \mathbb{P}(\ip{\beta_k}{q_s}\ge 0)\ge c_k\ge\min_{k=1,\cdots,m}c_k>0,\ \text{ for all }s=1,\cdots,m
    % \\
    % i.e., \mathbb{P}(\nabla_2\ge 0)\ge c_k.
\end{align*}
Then since $\beta_k\in\mathbb{R}^{nN-1}$ for $k=1,\cdots,m$ and $nN\ll m$, there exists a set of indices $\mathcal{K}\subseteq\{1,\cdots,m\}$ with $|\mathcal{K}|=\Theta(\frac{m}{nN})$ and a set of indices $\mathcal{S}\subseteq\{1,\cdots,m\}$ with $|\mathcal{S}|=\Theta(m)$, s.t., $\ip{\beta_k}{q_s}\ge 0$, for $k\in\mathcal{K},s\in\mathcal{S}$. 

Let $q_{ij,\ell}^\mathcal{K}=(q_{ij,\ell})_{k\in\mathcal{K}}$. Then by Bernstein's inequality, we can also obtain that $\|q_{ij,\ell}^\mathcal{K}\|^2=\Theta(d)$ with probability at least $1-nN\exp(-\Omega(d))$.

%Moreover, consider with $A_k$'s and $B_k$'s where they are independent with each other and apply the Chernoff bounds, we have, with probability at least $1-\exp(-\Omega(\frac{m\sqrt{\log d}}{d^{1/2+4c}})$ 
% \begin{align*}
%     |\{A_k \text{ and }B_k\}|\ge C_5\frac{m\sqrt{\log d}}{d^{1/2+4c}}
% \end{align*}
% for some constant $C_5>0$.

Combine all of the above and apply the Claim 9.5 in \citet{allen2019convergence}, we obtain, with probability at least $1-\mathcal{O}(nN)\exp(-\Omega(d^{2a_0-1}))$,
\begin{align*}
    \|\nabla_{W_L}\lem(\theta^{(0)})\|_F^2&\ge \frac{1}{n^2} C_6 w(t_{j^*})^2(t_{j^*}-t_{j^*-1})^2\frac{1}{d} \|\bar{\sigma}_{t_{j^*}}W_{L+1}q_{i^*j^*,L}+\xi_{i^*j^*}\|^2 \|q_{i^*j^*,L-1}^\mathcal{K}\|^2\frac{1}{nN}md^{\frac{a_0-1}{2}}\\
    &\ge C_6md^{\frac{a_0-1}{2}}w(t_{j^*})(t_{j^*}-t_{j^*-1})\bar{\sigma}_{t_{j^*}}\frac{1}{n^3N^2}\lem(\theta^{(0)}),
\end{align*}
where $C_6>0$ is some universal constant, $\frac{1}{2}<a_0<1$, and the second inequality follows from the definition of $i^*,j^*$.

\end{proof}

\subsubsection{Geometric ideas used in the proof}

\begin{proposition}
\label{prop:gaussian_norm_direction_indep_uniform}
    Consider $w\sim\mathcal{N}(0,I)$, where $w\in\mathbb{R}^n$. Then $\|w\|$ and $\frac{w}{\|w\|}$ are independent random variables and $\frac{w}{\|w\|}\sim\text{Unif}(\mathbb{S}^{n-1})$.
\end{proposition}

\begin{lemma}
\label{lem:prob_wv1>0_wv2>0}
    Let $w\sim \text{Unif}\ (\mathbb{S}^{n-1})$, where $\mathbb{S}^{n-1}=\{x\in\mathbb{R}^n|\|x\|=1\}$. Then for two vectors $v_1,v_2\in\mathbb{R}^n$,
    \begin{align*}
        \mathbb{P}(w^\top v_1\ge 0,w^\top v_2\ge 0)=\frac{\pi-\angle(v_1,v_2)}{2\pi}.
    \end{align*}
\end{lemma}
\begin{proof}
    Since $w\sim Unif(\mathbb{S}^{n-1})$, we only need to consider the area of the event. It is obvious that the set $\{w\in\mathbb{S}^{n-1}|w^\top v_i\}$ is a semi-hypersphere. Therefore, we only need to consider the intersection of two semi-hypersphere, i.e.,
    \begin{align*}
        \mathbb{P}(w^\top v_1\ge 0,w^\top v_2\ge 0)&=\frac{\text{area of }\{w\in\mathbb{S}^{n-1}|w^\top v_1\ge 0\}\cap\text{ area of }\{w\in\mathbb{S}^{n-1}|w^\top v_2\ge 0\}}{ \text{area of the hypersphere}}\\
        &=\frac{\pi-\angle(v_1,v_2)}{2\pi}.
    \end{align*}
\end{proof}

Next we 
% would like to show the probability when the lower bound of the two inner product are non-zero. We 
follow the notations and definitions in \citet{lee2014concise}. Consider the unit hypersphere in $\mathbb{R}^d$, $\mathbb{S}^{d-1}=\{x\in\mathbb{R}^d\,|\,\|x\|=1\}$. The area of $\mathbb{S}^{d-1}$ is 
\begin{align*}
    A_d(1)=\frac{2\pi^{d/2}}{\Gamma(d/2)}.
\end{align*}

\begin{lemma}
\label{lem:angle_Gaussian_3pi_4}
    Fix $\xi_1\in\mathbb{S}^{d-1}$ and let $\xi_2\sim \text{Unif}\,(\mathbb{S}^{d-1})$, where $\mathbb{S}^{d-1}=\{x\in\mathbb{R}^d|\|x\|=1\}$. Then with probability at least $1-\exp(-\Omega(d))$, we have $\angle (\xi_1,\xi_2)\le\frac{3\pi}{4}$.
\end{lemma}
\begin{proof}
    For any fixed $\xi_1$, all the $\xi_2$'s that satisfy $\angle(\xi_1,\xi_2)\ge\pi-\theta$ are on a hyperspherical cap. By \citet{lee2014concise}, the area of the hypersperical cap is
\begin{align*}
    A^\theta_d(1)=\frac{1}{2}A_d(1)I_{\sin^2\theta}\left( \frac{d-1}{2},\frac{1}{2} \right).
\end{align*}
Then
\begin{align}
\label{eqn:probability_angle_cap}
    \mathbb{P}(\angle(\xi_1,\xi_2)\ge\pi-\theta)=\frac{A_d^\theta(1)}{A_d(1)}=\frac{1}{2}I_{\sin^2\theta}\left( \frac{d-1}{2},\frac{1}{2} \right)\propto \frac{1}{2} \frac{\theta^{d-1}}{\sqrt{\pi}\sqrt{\frac{d-1}{2}}}.
\end{align}
Let $\theta=\frac{\pi}{4}<1$. Then with probability at least $1-\exp(-\Omega(d))$, we have $\angle(\xi_1,\xi_2)\le\frac{3\pi}{4}$.
\end{proof}

\subsection{Proofs related to random initialization}
Consider $W_i=W_i^{(0)}$ in this section.

\begin{lemma}
\label{lem:scale_Xij_qij}
    If $\epsilon\in(0,1)$, with probability at least $1-\mathcal{O}(nN)e^{-\Omega(\min(\epsilon^2d^{4b-1},\epsilon d^{2b}))}$, $\|X_{ij}\|^2\in [{\|e^{-\mu_{t_j}}x_i\|^2+\bar{\sigma}_{t_j}^2d}-\epsilon \bar{\sigma}_{t_j}^2d^{2b}, \|e^{-\mu_{t_j}}x_i\|^2+ {\bar{\sigma}_{t_j}^2d}+\epsilon \bar{\sigma}_{t_j}^2d^{2b}]$ for all $i=1,\cdots,n$ and $j=0,\cdots,N-1$. Moreover, with probability at least $1-\mathcal{O}(L)e^{-\Omega(m\epsilon^2/L)}$ over the randomness of $W_s$ for $s=0,\cdots,L$, we have $\|q_{ij,\ell}\|\in [\|X_{ij}\|(1-\epsilon),\|X_{ij}\|(1+\epsilon)]$ for fixed $i,j$. Therefore, with probability at least $1-\mathcal{O}(nNL)e^{-\Omega(\min(m\epsilon^2/L,\epsilon^2d^{4b-1},\epsilon d^{2b}))}$, we have $\Omega(d^b)=\|q_{ij,\ell}\|=\mathcal{O}(d^a)$.
\end{lemma}
\begin{proof}
    Consider $\frac{1}{\bar{\sigma}_{t_j}}X_{ij}=\frac{e^{-\mu_{t_j}}}{\bar{\sigma}_{t_j}}x_i+\xi_{ij}$. Since $\xi_{ij}\sim\mathcal{N}(0,I)$, $\|\frac{1}{\bar{\sigma}_{t_j}}X_{ij}\|^2$ follows from the noncentral $\chi^2$ distribution and $\mathbb{E}\|\frac{1}{\bar{\sigma}_{t_j}}X_{ij}\|^2=d+\|\frac{e^{-\mu_{t_j}}}{\bar{\sigma}_{t_j}}x_i\|^2$ (this includes the time variable at the $d$th dimension). By Berstein inequality,
\begin{align*}
    &\mathbb{P}\bigg(\bigg|\bigg\|\frac{1}{\bar{\sigma}_{t_j}}X_{ij}\bigg\|^2-\mathbb{E}\bigg\|\frac{1}{\bar{\sigma}_{t_j}}X_{ij}\bigg\|^2\bigg|\ge t\bigg)\le 2\exp\bigg(-c\min\bigg(\frac{t^2}{d},t\bigg)\bigg)\\
    &i.e., \mathbb{P}\bigg(\bigg|\|e^{-\mu_{t_j}}x_i+\bar{\sigma}_{t_j}\xi_{ij}\|^2-(\bar{\sigma}_{t_j}^2d+\|e^{-\mu_{t_j}}x_i\|^2)\bigg|\ge \bar{\sigma}_{t_j}^2t\bigg)\le 2\exp\bigg(-c\min\bigg(\frac{t^2}{d},t\bigg)\bigg)
\end{align*}
Therefore, with probability at least $1-\mathcal{O}(nN)e^{-\Omega(\min(\epsilon^2d^{4b-1},\epsilon d^{2b}))}$, $\|X_{ij}\|^2\in [{\|e^{-\mu_{t_j}}x_i\|^2+\bar{\sigma}_{t_j}^2d}-\epsilon \bar{\sigma}_{t_j}^2d^{2b}, \|e^{-\mu_{t_j}}x_i\|^2+ {\bar{\sigma}_{t_j}^2d}+\epsilon \bar{\sigma}_{t_j}^2d^{2b}]$ for all $i=1,\cdots,n$ and $j=0,\cdots,N-1$, where $\epsilon\in(0,1)$. The second part of the Lemma follows the similar proof in Lemma 7.1 of~\citet{allen2019convergence}. The last part follows from union bound and Assumption~\ref{assump:network_hyperparameters}.
\end{proof}

\begin{lemma}[Upper bound]
\label{lem:upper_bound}
    Under the random initialization of $W_i$ for $i=0,\cdots,L$, with probability at least $1-\mathcal{O}(nNL)e^{-\Omega(\min(m\epsilon^2/L,\epsilon^2d^{4b-1},\epsilon d^{2b}))}$, we have
    \begin{align*}
        \|\nabla_{W_\ell} \lem({\theta^{(0)}})\|^2=\mathcal{O}\left({md^{2a-1}} N \max_jw(t_{j})(t_j-t_{j-1})\bar{\sigma}_{t_j}\right)\lem(\theta^{(0)}).
    \end{align*}
\end{lemma}

\begin{proof}
For any $\ell=1,\cdot,L$, we have
    \begin{align*}
        &\|\nabla_{W_\ell}\lem(\theta)\|_F^2\\
        &=
        \sum_{k=1}^m\|\nabla_{(W_\ell)_k}\lem(\theta)\|^2\\
        &=\sum_{k=1}^m\bigg\|\frac{1}{n}\sum_{i=1}^n\sum_{j=1}^{N}w(t_j)(t_j-t_{j-1})\\
        &\qquad\times[( {W_{L+1}D_{ij,L}W_L\cdots D_{ij,\ell}W_{\ell+1}})^\top(\bar{\sigma}_{t_j}W_{L+1}q_{ij,L}+\xi_{ij})]_k\,q_{ij,\ell-1}\, \mathbbm{1}_{(W_\ell q_{ij,\ell-1})_k>0}\bigg\|^2\\
        &\le \frac{N}{n}\sum_{i=1}^n\sum_{j=1}^{N}w(t_j)^2(t_j-t_{j-1})^2 \sum_{k=1}^m\|({W_{L+1}D_{ij,L}W_L\cdots D_{ij,\ell}W_{\ell+1}})_k^\top(\bar{\sigma}_{t_j}W_{L+1}q_{ij,L}+\xi_{ij})\|^2\cdot\|q_{ij,\ell-1}\|^2\\
        &\le C_7 d^{2a}\frac{m}{d} \frac{N}{n}\sum_{i=1}^n\sum_{j=1}^{N}w(t_j)^2(t_j-t_{j-1})^2 \cdot\|\bar{\sigma}_{t_j}W_{L+1}q_{ij,L}+\xi_{ij}\|^2\\
        &\le C_7 d^{2a}\frac{mN}{d} \max_j w(t_j)(t_j-t_{j-1})\bar{\sigma}_{t_j}\lem(\theta)
    \end{align*}
    where the first inequality follows from Young's inequality; the second inequality follows from Lemma~\ref{lem:scale_Xij_qij} and Lemma 7.4 in~\citet{allen2019convergence}; $C_7>0$
\end{proof}

\subsection{Proofs related to perturbation}

Consider $W_i^\per=W^{(0)}_i+W'_i$ for $i=1,\cdots,L$ in this section. We follow the same idea in~\citet{allen2019convergence} to consider the network value of perturbed weights at each layer. We use the superscript ``$\rm per$'' to denotes the perturbed version, i.e.,
\begin{align*}
    &r_{ij,0}^\per=W_0 X_{ij}, q_{ij,0}^\per=\sigma(r_{ij,0}^\per), \\
    &r_{ij,\ell}^\per=W_\ell^\per  q_{ij,\ell-1}^\per, q_{ij,\ell}^\per=\sigma(r_{ij,\ell}^\per), \text{ for } \ell=1,\cdots,L\\
    &S(\theta^\per;t_j,X_{ij})=W_{L+1}q_{ij,L}^\per
\end{align*}
We also similarly define the diagonal matrix $D_{ij,\ell}^\per$ for the above network.

The following Lemma measures the perturbation of each layer. The lemma differs from Lemma 8.2 in \citet{allen2019convergence} by a scale of $d^a$. For sake of completeness, we state it in the following and the proof can be similarly obtained.
\begin{lemma}
\label{lem:forward_perturbation}
    Let $\omega\le\frac{1}{C_7L^{9/2}(\log m)^3 d^a}$ for some large $C>1$. With probability at least $1-\exp(-\Omega(d^a m\omega^{2/3}L))$, for any $\Delta W$ s.t. $\|\Delta W\|\le \omega$, we have
    \begin{enumerate}
        \item $r_{ij,\ell}^\per-r_{ij,\ell}$ can be decomposed to two part $r_{ij,\ell}^\per-r_{ij,\ell}=r'_{ij,\ell,1}+r'_{ij,\ell,2}$, where $\|r'_{ij,\ell,1}\|=\mathcal{O}(\omega L^{3/2}d^a)$ and $\|r'_{ij,\ell,2}\|_\infty=\mathcal{O}(\omega L^{5/2}\sqrt{\log m}d^am^{-1/2})$.
        \item $\|D_{ij,\ell}^\per-D_{ij,\ell}\|_0=\mathcal{O}(m\omega^{2/3}L)$ and $\|(D_{ij,\ell}^\per-D_{ij,\ell})r_{ij,\ell}^\per\|=\mathcal{O}(\omega L^{3/2} d^a)$.
        \item $\|r_{ij,\ell}^\per-r_{ij,\ell}\|$ and $\|q_{ij,\ell}^\per-q_{ij,\ell}\|$ are $\mathcal{O}(\omega L^{5/2}\sqrt{\log m}d^a)$.
    \end{enumerate}
\end{lemma}

\subsection{Proofs related to the evolution of the algorithm}
\begin{lemma}[Upper and lower bounds of gradient after perturbation]
\label{lem:perturb_upper_lower_bound_grad}
    Let 
    \begin{align*}
        \omega=\mathcal{O}\left(\frac{\lem^*}{L^9 (\log m)^2 n^3 N^3 d^{\frac{1-a_0}{2}}}\cdot\frac{\min_j w(t_{j})(t_j-t_{j-1})\bar{\sigma}_{t_j}}{\max \{\max_j w(t_{j})(t_j-t_{j-1})\bar{\sigma}_{t_j}, \sum_j (w(t_{j})(t_j-t_{j-1})\bar{\sigma}_{t_j})^2\}}\right).
    \end{align*}
    Consider $\theta^\per$ s.t. $\|\theta^\per-\theta\|\le \omega$, where $\theta$ follows from the Gaussian initialization. Then with probability at least $1-\mathcal{O}(nN)e^{-\Omega(d^{2a_0-1})}$,
    \begin{align*}
        &\|\nabla_{W_\ell} \lem({\theta^\per})\|^2=\mathcal{O}\left({md^{2a-1}} N \max_jw(t_{j})(t_j-t_{j-1})\bar{\sigma}_{t_j}\right)\lem(\theta^\per),\\
        &\|\nabla_{W_L} \lem({\theta^\per})\|^2=\Omega\left(\frac{md^{\frac{a_0-1}{2}}}{n^3N^2}  \ w(t_{j^*})(t_{j^*}-t_{j^*-1})\bar{\sigma}_{t_{j^*}}\right) \min\{\lem (\theta),\lem (\theta^\per)\},
    \end{align*}
    for $\ell=1,\cdots,L$.
\end{lemma}

\begin{proof}
    Consider the following terms
    \begin{align}
    &\nabla_{(W_\ell)_k}\lem(\theta)=\frac{1}{n}\sum_{i=1}^n\sum_{j=1}^{N}w(t_j)(t_j-t_{j-1})\notag\\
    &\qquad\times[({W_{L+1}D_{ij,L}W_L\cdots D_{ij,\ell}W_{\ell+1}})^\top(\bar{\sigma}_{t_j}W_{L+1}q_{ij,L}+\xi_{ij})]_k\,q_{ij,\ell-1}\, \mathbbm{1}_{(W_\ell q_{ij,\ell-1})_k>0}\label{fct:grad}\\
        &\nabla_{(W_\ell)_k}^\per\lem(\theta)=\frac{1}{n}\sum_{i=1}^n\sum_{j=1}^{N}w(t_j)(t_j-t_{j-1})\notag\\
        &\qquad\times[({W_{L+1}D_{ij,L}W_L\cdots D_{ij,\ell}W_{\ell+1}})^\top(\bar{\sigma}_{t_j}W_{L+1}q_{ij,L}^\per+\xi_{ij})]_k\,q_{ij,\ell-1}\, \mathbbm{1}_{(W_\ell q_{ij,\ell-1})_k>0},\label{fct:grad_per}\\
        &\nabla_{(W_\ell)_k}\lem(\theta^\per)=\frac{1}{n}\sum_{i=1}^n\sum_{j=1}^{N}w(t_j)(t_j-t_{j-1})\notag\\
        &\qquad\times[({W_{L+1}D_{ij,L}^\per W_L^\per\cdots D_{ij,\ell}^\per W_{\ell+1}^\per})^\top(\bar{\sigma}_{t_j}W_{L+1}q_{ij,L}^\per+\xi_{ij})]_k\,q_{ij,\ell-1}^\per\, \mathbbm{1}_{(W_\ell^\per q_{ij,\ell-1}^\per)_k>0}\label{fct:true_grad_per}
    \end{align}

    Then 
    \begin{align*}
        &\|\nabla_{W_\ell}\lem(\theta)-\nabla_{W_\ell}^\per\lem(\theta)\|^2_F\\
        &=\sum_{k=1}^m\|\nabla_{(W_\ell)_k}\lem(\theta)-\nabla_{(W_\ell)_k}^\per\lem(\theta)\|^2\\
        &\le\frac{N}{n}\sum_{i=1}^n\sum_{j=1}^{N}w(t_j)^2(t_j-t_{j-1})^2 \sum_{k=1}^m\|({W_{L+1}D_{ij,L}W_L\cdots D_{ij,\ell}W_{\ell+1}})_k^\top(\bar{\sigma}_{t_j}W_{L+1}(q_{ij,L}-q_{ij,L}^\per))\|^2\cdot\|q_{ij,\ell-1}\|^2\\
        &\le C_8 d^{2a}\frac{m}{d} \frac{N}{n}\sum_{i=1}^n\sum_{j=1}^{N}w(t_j)^2(t_j-t_{j-1})^2 \cdot\|\bar{\sigma}_{t_j}W_{L+1}(q_{ij,L}-q_{ij,L}^\per)\|^2\\
        &\le C_8' d^{2a}\frac{m}{d} {N}\sum_{j=1}^{N}w(t_j)^2(t_j-t_{j-1})^2(\omega L^{5/2}\sqrt{\log m}d^a)^2\\
        &\le \tilde{C}_8\left(\frac{md^{\frac{a_0-1}{2}}}{n^3N^2} \ w(t_{j^*})(t_{j^*}-t_{j^*-1})\bar{\sigma}_{t_{j^*}}\right) \lem^*
    \end{align*}
    where the first two inequalities follow the same as the proof of Lemma~\ref{lem:upper_bound}; the third inequality follows from Lemma~\ref{lem:forward_perturbation}; the last inequality follows from the definition of $\omega$.

    Also, we have
    \begin{align*}
        &\|\nabla_{(W_\ell)_k}^\per\lem(\theta)-\nabla_{(W_\ell)_k}\lem(\theta^\per)\|\\
        &\le \bigg\|\frac{1}{n}\sum_{i=1}^n\sum_{j=1}^{N}w(t_j)(t_j-t_{j-1})\\
        &\qquad\times[({W_{L+1}D_{ij,L}^\per W_L^\per\cdots D_{ij,\ell}^\per W_{\ell+1}^\per}-{W_{L+1}D_{ij,L}W_L\cdots D_{ij,\ell}W_{\ell+1}})^\top(\bar{\sigma}_{t_j}W_{L+1}q_{ij,L}^\per+\xi_{ij})]_k\\
        &\qquad\times\,q_{ij,\ell-1}^\per\, \mathbbm{1}_{(W_\ell^\per q_{ij,\ell-1}^\per)_k>0}\bigg\|\\
        &+\bigg\|\frac{1}{n}\sum_{i=1}^n\sum_{j=1}^{N}w(t_j)(t_j-t_{j-1})[({W_{L+1}D_{ij,L}W_L\cdots D_{ij,\ell}W_{\ell+1}})^\top(\bar{\sigma}_{t_j}W_{L+1}q_{ij,L}^\per+\xi_{ij})]_k\\
        &\qquad\times\left(q_{ij,\ell-1}\, \mathbbm{1}_{(W_\ell q_{ij,\ell-1})_k>0}-q_{ij,\ell-1}^\per\, \mathbbm{1}_{(W_\ell^\per q_{ij,\ell-1}^\per)_k>0}\right)\bigg\|
    \end{align*}
    Then
    \begin{align*}
        &\|\nabla_{W_\ell}^\per\lem(\theta)-\nabla_{W_\ell}\lem(\theta^\per)\|_F^2\\
        & \le 2\frac{N}{n}\sum_{i=1}^n\sum_{j=1}^{N}w(t_j)^2(t_j-t_{j-1})^2 \cdot\|q_{ij,\ell-1}^\per\|^2\\
        &\qquad\times\sum_{k=1}^m\|({W_{L+1}D_{ij,L}^\per W_L^\per\cdots D_{ij,\ell}^\per W_{\ell+1}^\per}-{W_{L+1}D_{ij,L}W_L\cdots D_{ij,\ell}W_{\ell+1}})_k^\top(\bar{\sigma}_{t_j}W_{L+1}q_{ij,L}^\per+\xi_{ij})\|^2\\
        &\quad+2\frac{N}{n}\sum_{i=1}^n\sum_{j=1}^{N}w(t_j)^2(t_j-t_{j-1})^2 \cdot\|q_{ij,\ell-1}\, \mathbbm{1}_{(W_\ell q_{ij,\ell-1})_k>0}-q_{ij,\ell-1}^\per\, \mathbbm{1}_{(W_\ell^\per q_{ij,\ell-1}^\per)_k>0}\|^2\\
        &\qquad\times\sum_{k=1}^m\|({W_{L+1}D_{ij,L}W_L\cdots D_{ij,\ell}W_{\ell+1}})_k^\top(\bar{\sigma}_{t_j}W_{L+1}q_{ij,L}^\per+\xi_{ij})\|^2\\
        &\le C_9 (\omega^2L^5 \log m d^{2a})(L^{3/2}\log m\, m^{1/2})\frac{m}{d}\frac{N}{n}\sum_{i=1}^n\sum_{j=1}^{N}w(t_j)^2(t_j-t_{j-1})^2 \cdot\|\bar{\sigma}_{t_j}W_{L+1}q_{ij,L}^\per+\xi_{ij}\|^2\\
        &\quad+C_{9}' (\omega^2L^5 \log m d^{2a})\frac{m}{d} \frac{N}{n}\sum_{i=1}^n\sum_{j=1}^{N}w(t_j)^2(t_j-t_{j-1})^2 \cdot\|\bar{\sigma}_{t_j}W_{L+1}q_{ij,L}^\per+\xi_{ij}\|^2\\
        &\le \tilde{C}_9\left(\frac{md^{\frac{a_0-1}{2}}}{n^3N^2}  \ w(t_{j^*})(t_{j^*}-t_{j^*-1})\bar{\sigma}_{t_{j^*}}\right) \lem (\theta^\per)
    \end{align*}
    where the first inequality follows from Young's inequality and the above decomposition; the second inequality follows from Lemma 7.4, 8.7 in \citet{allen2019convergence} (with $s=\mathcal{O}(d^am\omega^{2/3}L)$) and Lemma~\ref{lem:forward_perturbation}; the last inequality follows from the definition of $\omega$.

   For upper bound, we only need to consider $\nabla_{(W_\ell)_k}^\per\lem(\theta)$ and $\nabla_{(W_\ell)_k}\lem(\theta^\per)$. By similar argument as Lemma~\ref{lem:upper_bound}, with probability at least $1-\mathcal{O}(nNL)e^{-\Omega(\min(m\epsilon^2/L,\epsilon^2d^{4b-1},\epsilon d^{2b}))}$, we have
    \begin{align*}
        \|\nabla_{(W_\ell)_k}^\per\lem(\theta)\|^2=\mathcal{O}\left({md^{2a-1}} N \max_jw(t_{j})(t_j-t_{j-1})\bar{\sigma}_{t_j}\right)\lem(\theta^{\per}).
    \end{align*} 
    Then
    \begin{align*}
        &\|\nabla_{W_\ell} \lem({\theta^\per})\|^2\\
        &\le 2\|\nabla_{W_\ell}^\per\lem(\theta)\|^2_F+2\|\nabla_{W_\ell}^\per\lem(\theta)-\nabla_{W_\ell}\lem(\theta^\per)\|_F^2\\
        &=\mathcal{O}\left({md^{2a-1}} N \max_jw(t_{j})(t_j-t_{j-1})\bar{\sigma}_{t_j}\right)\lem(\theta^{\per}).
    \end{align*}
    Also, 
    \begin{align*}
         &\|\nabla_{W} \lem({\theta^\per})\|^2\\
         &\ge\|\nabla_{W_L} \lem({\theta^\per})\|^2\\
         &\ge  \frac{1}{3}\|\nabla_{W_L} \lem({\theta})\|^2-\|\nabla_{W_\ell}\lem(\theta)-\nabla_{W_\ell}^\per\lem(\theta)\|^2_F-\|\nabla_{W_\ell}^\per\lem(\theta)-\nabla_{W_\ell}\lem(\theta^\per)\|_F^2\\
         &=\Omega\left(\frac{md^{\frac{a_0-1}{2}}}{n^3N^2} \ w(t_{j^*})(t_{j^*}-t_{j^*-1})\bar{\sigma}_{t_{j^*}}\right) \min\{\lem (\theta),\lem (\theta^\per)\}.
    \end{align*}
    % where the last inequality follows from $\lem^*\le\lem(\theta^\per)\le \lem(\theta)$ (the right inequality is obtained via Taylor expansion of each iteration). 
    % The probability is obtained by taking $s=\mathcal{O}(m^2\omega^2L)$ in \citet{allen2019convergence}, which is still less than $\mathcal{O}(\frac{m}{L^3\log m})$.
\end{proof}
Note when interpolation is not achievable, this lower bound is always away from 0, which means the current technique can only evaluate the lower bound outside a neighbourhood of the minimizer. More advanced method is needed and we leave it for future investigation.

\begin{lemma}
    [semi-smoothness]
    \label{lem:semi-smoothness}
    Let $\omega=\Omega$. With probability at least $1-e^{-\Omega(\log m)}$ over the randomness of $\theta^{(0)}$, we have for all $\theta$ s.t. $\|\theta-\theta^{(0)}\|\le\omega$, and all $\theta^\per$ s.t. $\|\theta^\per-\theta\|\le\omega$,
    \begin{align*}
        \lem(\theta^\per)\le \lem(\theta)&+\ip{\nabla \lem(\theta)}{\theta^\per-\theta}\\
        &+\sqrt{\lem(\theta)}\sqrt{\sum_{j}w(t_j)(t_j-t_{j-1})\bar{\sigma}_{t_j}}\mathcal{O}(\omega^{1/3}L^2\sqrt{m\log m}d^{a/2})\|\theta^\per-\theta\|\\
        &+\sqrt{\lem(\theta)}\sqrt{\sum_{j}w(t_j)(t_j-t_{j-1})\bar{\sigma}_{t_j}}\mathcal{O}(L^2\sqrt{m}d^{a})\|\theta^\per-\theta\|^2
    \end{align*}
\end{lemma}
\begin{proof}
By definition,
    \begin{align*}
        &\lem(\theta^\per)-\lem(\theta)-\ip{\nabla\lem(\theta)}{\theta^\per-\theta}\\
        &=\frac{1}{n}\sum_{i=1}^n\sum_{j=1}^N w(t_j)(t_j-t_{j-1})(\bar{\sigma}_{t_j}W_{L+1}q_{ij,L}+\xi_{ij})^\top W_{L+1}\\
        &\qquad\times\left( q_{ij,L}^\per-q_{ij,L}-\sum_{\ell=1}^L D_{ij,L} W_{ij,L}\cdots W_{ij,\ell+1}D_{ij,\ell}(W^\per_{ij,\ell}-W_{ij,\ell})q_{ij,\ell} \right)\\
        &\qquad+\frac{1}{2\bar{\sigma}_{t_j}}w(t_j)(t_j-t_{j-1})\|\bar{\sigma}_{t_j}W_{L+1}(q_{ij,L}^\per-q_{ij,L})\|^2.
    \end{align*}
    Similar to the proof of Theorem 4 in~\citet{allen2019convergence}, we obtain the desired bound by using Cauchy-Schwartz inequality. Note, in our case, due to the order of input data, we choose $s=\mathcal{O}(d^am\omega^{2/3}L)$ in \citet{allen2019convergence} and therefore the bound is slightly different from theirs.
\end{proof}

\section{Proofs for sampling}\label{append:sampling proof}
\label{app:sampling_proof}
In this section, we prove Theorem~\ref{thm:VESDE_generation}. The proof includes two main steps: 1. decomposing $\KL(p_\delta|q_{T-\delta})$ into the initialization error, the score estimation errors and the discretization errors; 2. estimating the initialization error and the discretization error based on our assumptions. In the following context, we introduce the proof of these two steps separately. 

\begin{proof}[Proof of Theorem \ref{thm:VESDE_generation}]
\textbf{Step 1:} The error decomposition follows from the ideas in \citep{chen2023improved} of studying VPSDE-based diffusion models. According to the chain rule of KL divergence, we have
\begin{align*}
    \KL(p_\delta | q_{T-\delta} ) \le \KL(p_T|q_0)+ \mathbb{E}_{y\sim p_T} [ \KL(p_{\delta|T}(\cdot|y)|q_{T-\delta|0}(\cdot|y)) ] ,
\end{align*}
\iffalse
Let us first introduce an intermediate process, $\{\Tilde{Y}_t\}_{0\le t\le T-\delta}$, defined by piecewisely by the following SDEs: for any $t\in [\backt_j,\backt_{j+1})$,
\begin{align}\label{eq:intermediate generation SDE}
    d \Tilde{Y}_t =  2\sigma_{T-t}^2 s(\theta;T-\backt_j, \Tilde{Y}_{\backt_j})  dt +\sqrt{2\sigma_{T-t}^2} d\Tilde{W}_t,
\end{align}
with initialization $\Tilde{Y}_0\sim p_T$. Now we look at the path measures of the processes $\{\bar{Y}_t\}_{0\le t\le T-\delta}$ (defined in \eqref{eqn:EI_backward_SDE}), $\{Y_t\}_{0\le t\le T-\delta}$ (defined in \eqref{eqn:general_backward_SDE} with $f_t\equiv 0$) and $\{\Tilde{Y}_t\}_{0\le t\le T-\delta}$ (defined in \eqref{eq:intermediate generation SDE}), and denote them by $Q^{q_0}$, $P^{p_T}$ and $Q^{p_T}$, respectively. Then we have 
\begin{align*}
    \KL(p_\delta|q_{T-\delta}) &\le \int \ln \frac{d P^{p_T}}{d Q^{q_0}} d P^{p_T} = \int \ln \left(\frac{d P^{p_T}}{d Q^{p^T}}\frac{d Q^{p_T}}{d Q^{q_0}}\right) d P^{p_T} \\
    &= \KL(P^{p_T} | Q^{p_T} ) + \KL(p_T|q_0),
\end{align*}
where the inequality follows from the data processing inequality and the last identity follows from the fact that $\frac{d Q^{p_T}}{d Q^{q_0}}=\frac{d Q^{p_T}_{0}}{d Q^{q_0}_{0}}= \frac{d p_T}{d q_0} $, which is true because the processes $\{ \Bar{Y}_t \}_{0\le t\le T-\delta}$ and $\{\Tilde{Y}_t\}_{0\le t\le T-\delta}$ have the same transition kernel function.
\fi
Apply the chain rule again for at across the time schedule $(T-\backt_j)_{0\le j\le N-1}$, the second term can be written as
\begin{align*}
    &\quad \mathbb{E}_{y\sim p_T} [ \KL(p_{\delta|T}(\cdot|y)|q_{T-\delta|0}(\cdot|y)) ] \\
    &\le \sum_{j=0}^{N-1} \mathbb{E}_{y_j\sim p_{T-\backt_j}} [ \KL(p_{T-\backt_{j+1}|T-\backt_j}(\cdot|y_j)|q_{\backt_{j+1}|\backt_j}(\cdot|y_j)) ] \\
    &\le \frac{1}{2}\sum_{j=0}^{N-1} \int_{\backt_j}^{\backt_{j+1}} \sigma_{T-t}^2 \mathbb{E}[ \lVert s(\theta;T-\backt_j, Y_{\backt_j})-\nabla\log p_{T-t}(Y_t) \rVert^2 ] dt \\
    &\le \sum_{j=0}^{N-1} \int_{\backt_j}^{\backt_{j+1}} \sigma_{T-t}^2 \mathbb{E}[ \lVert s(\theta;T-\backt_j, Y_{\backt_j})-\nabla\log p_{T-\backt_j}(Y_{\backt_j}) \rVert^2 ] dt\\
    &\quad+ \sum_{j=0}^{N-1} \int_{\backt_j}^{\backt_{j+1}} \sigma_{T-t}^2 \mathbb{E}[ \lVert \nabla\log p_{T-\backt_j}(Y_{\backt_j})-\nabla\log p_{T-t}(Y_t) \rVert^2 ] dt,
\end{align*}
where the second inequality follows from Lemma~\ref{lem:Girsanov application}. Therefore, the error decomposition writes as 
\begin{align}\label{eq:error decomposition}
    \KL(p_\delta|q_{T-\delta}) &\lesssim \KL(p_T|q_0) + \sum_{j=0}^{N-1} \int_{\backt_j}^{\backt_{j+1}} \sigma_{T-t}^2 \mathbb{E}[ \lVert s(\theta;T-\backt_j, Y_{\backt_j})-\nabla\log p_{T-\backt_j}(Y_{\backt_j}) \rVert^2 ] dt \nonumber\\
    &\quad+ \sum_{j=0}^{N-1} \int_{\backt_j}^{\backt_{j+1}} \sigma_{T-t}^2 \mathbb{E}[ \lVert \nabla\log p_{T-\backt_j}(Y_{\backt_j})-\nabla\log p_{T-t}(Y_t) \rVert^2 ] dt
\end{align}
where the three terms in \eqref{eq:error decomposition} quantify the initialization error, the score estimation error and the discretization error, respectively.

\textbf{Step 2:} In this step, we estimate the three error terms in Step 1. First, recall that $p_T= p\ast \mathcal{N}(0,\bar{\sigma}_T^2 I_d)$ and $q_0=\mathcal{N}(0,\bar{\sigma}_T^2 I_d)$, hence the initialization error $\KL(p_T|q_0)$ can be estimated as follows,
\begin{align}\label{eq:initialization error bound}
    \KL(p_T|q_0)=\KL( p\ast \mathcal{N}(0,\bar{\sigma}_T^2 I_d)| \mathcal{N}(0,\bar{\sigma}_T^2 I_d) )\lesssim \tfrac{\mathrm{m}_2^2}{\bar{\sigma}_T^2},
\end{align}
where the inequality follows from Lemma~\ref{lem:initialization error VE}. Hence we recover the term $E_I$ in \eqref{eqn:VESDE_generation}.

Next, since $\sigma_t$ is non-decreasing in $t$, the score estimation error can be estimated as
\begin{align}\label{eq:score estimation error VE}
    &\quad \sum_{j=0}^{N-1} \int_{\backt_j}^{\backt_{j+1}} \sigma_{T-t}^2 \mathbb{E}[ \lVert s(\theta;T-\backt_j, Y_{\backt_j})-\nabla\log p_{T-\backt_j}(Y_{\backt_j}) \rVert^2 ] dt \nonumber\\
    &\le \sum_{j=0}^{N-1} \gamma_j \sigma_{T-\backt_j}^2 \mathbb{E}[ \lVert s(\theta;T-\backt_j, Y_{\backt_j})-\nabla\log p_{T-\backt_j}(Y_{\backt_j}) \rVert^2 ].
\end{align}
Hence, we recover the term $E_S$ in \eqref{eqn:VESDE_generation}.

Last, we estimated the discretization error term. Our approach is motivated by analyses of VPSDEs in \citep{benton2024nearly,he2024zeroth}. We defines a process $L_t:= \nabla\log p_{T-t}(Y_t)$. Then we can relate discretization error to quantities depending on $L_t$, and therefore bound the discretization error via properties of $\{L_t\}_{0\le t\le T}$. According to Lemma \ref{lem:discretization error representation}, we have
\begin{align*}
    &\quad \sum_{j=0}^{N-1} \int_{\backt_j}^{\backt_{j+1}} \sigma_{T-t}^2 \mathbb{E}[ \lVert \nabla\log p_{T-\backt_j}(Y_{\backt_j})-\nabla\log p_{T-t}(Y_t) \rVert^2 ] dt \\
    &\le \underbrace{2d \sum_{j=0}^{N-1} \int_{\backt_j}^{\backt_{j+1}} \int_{\backt_j}^t \sigma_{T-t}^2\sigma_{T-u}^2\bar{\sigma}_{T-u}^{-4} du dt}_{N_1} + \underbrace{\sum_{j=0}^{N-1} \int_{\backt_j}^{\backt_{j+1}} \sigma_{T-t}^2 dt \bar{\sigma}_{T-\backt_j}^{-4}\mathbb{E}[\Tr(\Sigma_{T-\backt_j}(X_{T-\backt_j}))]}_{N_2}\\
    &\quad \underbrace{-\sum_{j=0}^{N-1} \int_{\backt_j}^{\backt_{j+1}} \sigma_{T-t}^2 \bar{\sigma}_{T-t}^{-4}\mathbb{E}[\Tr(\Sigma_{T-t}(X_{T-t}))] dt}_{N_3} .
\end{align*}
Since $t\mapsto \sigma_t$ is non-decreasing and $t\mapsto \mathbb{E}[\Tr(\Sigma_{T-t}(X_{T-t}))]$ is non-increasing, we have
\begin{align*}
    N_1 &=2d \sum_{j=0}^{N-1} \int^{T-\backt_j}_{T-\backt_{j+1}} \int_{T-\backt_{j+1}}^{T-u} \sigma_{t}^2 dt \sigma_{u}^2\bar{\sigma}_{u}^{-4} du \le 2d \sum_{j=0}^{N-1} \gamma_j \int^{T-\backt_j}_{T-\backt_{j+1}}\sigma_{t}^4 \bar{\sigma}_{t}^{-4}dt, \\
    N_2&=  \sum_{j=0}^{N-1} \int_{\backt_j}^{\backt_{j+1}} \sigma_{T-t}^2 dt \bar{\sigma}_{T-\backt_j}^{-4} \mathbb{E}[\Tr(\Sigma_{T-\backt_j}(X_{T-\backt_j}))], \\
    N_3&\le -\sum_{j=0}^{N-1} \int_{\backt_j}^{\backt_{j+1}} \sigma_{T-t}^2 \bar{\sigma}_{T-t}^{-4} dt \mathbb{E}[\Tr(\Sigma_{T-\backt_{j+1}}(X_{T-\backt_{j+1}}))] .
\end{align*}
Therefore, we obtain
\begin{align}
      &\quad \sum_{j=0}^{N-1} \int_{\backt_j}^{\backt_{j+1}} \sigma_{T-t}^2 \mathbb{E}[ \lVert \nabla\log p_{T-\backt_j}(Y_{\backt_j})-\nabla\log p_{T-t}(Y_t) \rVert^2 ] dt \nonumber\\
      &\le 2d \sum_{j=0}^{N-1} \gamma_j \int^{T-\backt_j}_{T-\backt_{j+1}}\sigma_{t}^4 \bar{\sigma}_{t}^{-4}dt + \sum_{j=0}^{N-1} \int_{\backt_j}^{\backt_{j+1}} \sigma_{T-t}^2 dt \bar{\sigma}_{T-\backt_j}^{-4}\mathbb{E}[\Tr(\Sigma_{T-\backt_j}(X_{T-\backt_j}))]
      \nonumber\\
      &\quad -\sum_{j=0}^{N-1} \int_{\backt_j}^{\backt_{j+1}} \sigma_{T-t}^2 \bar{\sigma}_{T-t}^{-4} dt \mathbb{E}[\Tr(\Sigma_{T-\backt_{j+1}}(X_{T-\backt_{j+1}}))] \nonumber\\
      &=2d \sum_{j=0}^{N-1} \gamma_j \int^{T-\backt_j}_{T-\backt_{j+1}}\sigma_{t}^4 \bar{\sigma}_{t}^{-4}dt + \int_{0}^{\backt_{1}} \sigma_{T-t}^2 dt \bar{\sigma}_{T}^{-4}\mathbb{E}[\Tr(\Sigma_T(X_T))] \nonumber\\
      &\quad +\sum_{j=1}^{N-1} \big(\int_{\backt_j}^{\backt_{j+1}} \sigma_{T-t}^2 \bar{\sigma}_{T-\backt_j}^{-4}dt - \int_{\backt_{j-1}}^{\backt_{j}} \sigma_{T-t}^2\bar{\sigma}_{T-t}^{-4} dt\big) \mathbb{E}[\Tr(\Sigma_{T-\backt_j}(X_{T-\backt_j}))]. \label{eq:discretization bound intermediate}
\end{align}
The above bound depends on $\mathbb{E}[\Tr(\Sigma_{t}(X_t))]$, hence we estimate $\mathbb{E}[\Tr(\Sigma_{t}(X_t))]$ for different values of $t$. 

First, we have
\begin{align*}
    \mathbb{E}[\Tr(\Sigma_{t}(X_t))] = \mathbb{E}[ \mathbb{E}[\lVert X_0 \rVert^2|X_t]-\lVert \mathbb{E}[X_0|X_t] \rVert^2 ] \le \mathbb{E}[\lVert X_0 \rVert^2]=\mathrm{m}_2^2.
\end{align*}
Meanwhile,
\begin{align*}
    \mathbb{E}[\Tr(\Sigma_{t}(X_t))] &=\mathbb{E} [ \Tr (\text{Cov}(X_0-X_t|X_t))]=\mathbb{E}[\mathbb{E}[\lVert X_0-X_t \rVert^2|X_t]]-\lVert \mathbb{E}[X_0-X_t|X_t] \rVert^2\\\
    &\le \mathbb{E}[\lVert X_0-X_t \rVert^2] =\bar{\sigma}_t^2 d
\end{align*}
Therefore, $\mathbb{E}[\Tr(\Sigma_{t}(X_t))]\le \min ( \mathrm{m}_2^2, \bar{\sigma}_t^2 d )\lesssim (1-e^{-\bar{\sigma}_t^2})(\mathrm{m}_2^2 +d) $. Plug this estimation into \eqref{eq:discretization bound intermediate} and we get
\begin{align*}
    &\quad \sum_{j=0}^{N-1} \int_{\backt_j}^{\backt_{j+1}} \sigma_{T-t}^2 \mathbb{E}[ \lVert \nabla\log p_{T-\backt_j}(Y_{\backt_j})-\nabla\log p_{T-t}(Y_t) \rVert^2 ] dt \nonumber\\
      &\lesssim d \sum_{j=0}^{N-1} \gamma_j \int^{T-\backt_j}_{T-\backt_{j+1}}\sigma_{t}^4 \bar{\sigma}_{t}^{-4}dt + \int_{0}^{\backt_{1}} \sigma_{T-t}^2 dt \bar{\sigma}_{T}^{-4}\mathrm{m}_2^2 \nonumber\\
      &\quad +(\mathrm{m}_2^2+d)\sum_{j=1}^{N-1}(1-e^{-\bar{\sigma}_{T-\backt_j}^2}) \big(\int_{\backt_j}^{\backt_{j+1}} \sigma_{T-t}^2 \bar{\sigma}_{T-\backt_j}^{-4}dt - \int_{\backt_{j-1}}^{\backt_{j}} \sigma_{T-t}^2\bar{\sigma}_{T-t}^{-4} dt\big) \\
      &\lesssim d \sum_{j=0}^{N-1} \gamma_j \int^{T-\backt_j}_{T-\backt_{j+1}}\sigma_{t}^4 \bar{\sigma}_{t}^{-4}dt + \mathrm{m}_2^2 \frac{\int_0^{\backt_1}\sigma_{T-t}^2 d t }{\bar{\sigma}_T^{4}} + +(\mathrm{m}_2^2+d)\sum_{k=1}^{N-1}(1-e^{-\bar{\sigma}_{T-\backt_j}^2})\frac{\bar{\sigma}_{T-\backt_j}^4-\bar{\sigma}_{T-\backt_{j+1}}^2\bar{\sigma}_{T-\backt_{j-1}}^2}{\bar{\sigma}_{T-\backt_{j-1}}^2\bar{\sigma}_{T-\backt_j}^4},
\end{align*}
where the last inequality follows from the definition of $\bar{\sigma}_t$ and integration by parts. The proof of Theorem \ref{thm:VESDE_generation} is completed.
\end{proof}
\begin{lemma}\label{lem:Girsanov application} Let $\{Y_t\}_{0\le t\le T}$ be the solution to \eqref{eqn:general_backward_SDE} with $f_t\equiv 0$ and $p^{\leftarrow}_{t+s|s}(\cdot|y)$ be the conditional distribution of $Y_{s+t}$ given $\{Y_s=y\}$. Let $\{\bar{Y}_t\}_{0\le t\le T}$ be the solution to \eqref{eqn:VESDE_generation} $q_{t+s|s}(\cdot|y)$  be the conditional distribution of $\bar{Y}_{s+t}$ given $\{\bar{Y}_s=y\}$. Then for any fixed $t\in (0,\gamma_j]$, we have
\begin{align*}
    \mathbb{E}_{y\sim p^{\leftarrow}_{\backt_j}}\KL( p_{\backt_j+t|\backt_j}^{\leftarrow}(\cdot|y)|q_{\backt_j+t|\backt_j}(\cdot|y)  )\le \frac{1}{2} \sigma_{T-t}^2 \mathbb{E} [\lVert s(\theta; T-\backt_j, Y_{\backt_j})-\nabla \log p_{T-t}(Y_t)  \rVert^2]
\end{align*}
\end{lemma}
\begin{proof}[Proof of Lemma \ref{lem:Girsanov application}] According to \citep[Lemma 6]{chen2023improved}, we have
\begin{align*}
    &\quad \KL( p_{\backt_j+t|\backt_j}^{\leftarrow}(\cdot|y)|q_{\backt_j+t|\backt_j}(\cdot|y)  )\\
    &\le -2\sigma_{T-t}^2 \int  p_{\backt_j+t|\backt_j}^{\leftarrow}(x|y) \lVert \nabla \log \frac{p_{\backt_j+t|\backt_j}^{\leftarrow}(x|y)}{q_{\backt_j+t|\backt_j}(x|y)} \rVert^2 dx\\
    &\quad+2\sigma_{T-t}^2\mathbb{E}_{p_{\backt_j+t|\backt_j}^{\leftarrow}(x|y)}[\langle \nabla\log p_{T-t}(x)-s(\theta; T-\backt_j,Y_{\backt_j}),  \nabla \log \frac{p_{\backt_j+t|\backt_j}^{\leftarrow}(x|y)}{q_{\backt_j+t|\backt_j}(x|y)} \rangle] \\
    &\le \frac{1}{2} \sigma_{T-t}^2 \mathbb{E}_{p_{\backt_j+t|\backt_j}^{\leftarrow}(x|y)}[\lVert \nabla\log p_{T-t}(x)-s(\theta; T-\backt_j,Y_{\backt_j})\rVert^2],
\end{align*}
where the last inequality follows from Young's inequality. Therefore, Lemma \ref{lem:Girsanov application} is proved after taking another expectation.
    
\end{proof}
\begin{lemma}\label{lem:initialization error VE} For any probability distribution $p$ satisfying Assumption \ref{assump:finite_second_moment} and $q$ being a centered multivariate normal distribution with covariance matrix $\sigma^2 I_d$, we have
    \begin{align*}
        \KL(p\ast q | q) \le \frac{\mathrm{m}_2^2}{ 2\sigma^2}.
    \end{align*}
\end{lemma}
\begin{proof}[Proof of Lemma \ref{lem:initialization error VE}]
    \begin{align*}
        \KL(p\ast q|q) &\le \int \KL( q(\cdot-y)|q(\cdot) ) p(dy) = \int \KL( \mathcal{N}(y,\sigma^2 I_d)| \mathcal{N}(0,\sigma^2 I_d) ) p(dy) \\
        &= \frac{1}{2} \int \ln(1)-d+\Tr(I_d) + \lVert y \rVert^2 \sigma^{-2} p(dy) = \frac{\mathrm{m}_2^2}{2\sigma^2},
    \end{align*}
    where the inequality follows from convexity of $\KL(\cdot|q)$ and the second identity follows from KL-divergence between multivariate normal distributions.
\end{proof}
\begin{lemma}\label{lem:localization} Let $\{X_t\}_{0\le t\le T}$ be the solution to \eqref{eqn:general_forward_SDE} with $f_t\equiv 0$ and $p_{0|t}(\cdot|x)$ be the conditional distribution of $X_0$ given $\{X_t=x\}$. Define
\begin{align}\label{eq:localization mean and variance}
    m_t(X_t):= \mathbb{E}_{X\sim p_{0|t}(\cdot|X_t)} [X] ,\quad \Sigma_t(X_t) = \text{Cov}_{X\sim p_{0|t}(\cdot|X_t)}(X).
\end{align}
Let $\{Y_t\}_{0\le t\le T}$ be the solution to \eqref{eqn:general_backward_SDE} with $f_t\equiv 0$ and $q_{0|t}(\cdot|x)$ be the conditional distribution of $Y_0$ given $\{Y_t=x\}$. Define
\begin{align}\label{eq:localization mean and variance reverse}
    \bar{m}_t(Y_t):= \mathbb{E}_{X\sim q_{0|t}(\cdot|Y_t)} [X] ,\quad \bar{\Sigma}_t(Y_t) = \text{Cov}_{X\sim q_{0|t}(\cdot|Y_t)}(X).
\end{align}
    Then we have for all $t\in (0,T)$,
    \begin{align*}
        d \bar{m}_{t}(Y_t) & =\sqrt{2}\sigma_{T-t}\bar{\sigma}_{T-t}^{-2} \bar{\Sigma}_{t}(Y_t)  d \Tilde{W}_t , \\
        \text{and}\qquad \frac{d}{dt} \mathbb{E} [ \Sigma_t(X_t) ] &=2\sigma_t^2 \bar{\sigma}_t^{-4} \mathbb{E} [\Sigma_t(X_t)^2].
    \end{align*}
\end{lemma}
\begin{proof}[Proof of Lemma \ref{lem:localization}] We first represent $\nabla log p_t(X_t)$ and $\nabla^2\log p_t(X_t)$ via $m_t(X_t)$ and $\Sigma_t(X_t)$. Since $\{X_t\}_{0\le t\le T}$ solves \eqref{eqn:general_forward_SDE}, $X_t=X_0+\bar{\sigma}_t \xi$ with $(X_0,\xi)\sim p \otimes \mathcal{N}(0,I_d)$. Therefore, according to Bayes rule, we have
\begin{align}\label{eq:score along FD}
    \nabla \log p_t(X_t) &= \frac{1}{p_t(X_t)} \int \nabla \log p_{t|0}(X_t|x) p_{0,t}(x,X_t) dx \nonumber \\
    &= \mathbb{E}_{x\sim p_{0|t}(\cdot|X_t)}[ \bar{\sigma}_t^{-2}(X_t-x) ] \nonumber\\
    &= -\bar{\sigma}_t^{-2} (X_t-m_t(X_t)),
\end{align}
    where the second identity follows from the fact that $p_{t|0}(\cdot|x)=\mathcal{N}(x,\bar{\sigma}_t^2 I_d)$. The last identity follows from the definition of $m_t(X_t)$ in Lemma \ref{lem:localization}. Similarly, according to Bayes rule, we can compute
    \begin{align}\label{eq:hessian score along FD}
        &\quad \nabla^2 \log p_t(X_t) \nonumber\\
        &=\frac{1}{p_t(X_t)} \int \nabla^2 \log p_{t|0}(X_t|x) p_{0,t}(x,X_t) dx \nonumber \\ 
        &\quad + \frac{1}{p_t(X_t)} \int \big(\nabla \log p_{t|0}(X_t|x)\big)\big(\nabla \log p_{t|0}(X_t|x)\big)^\intercal p_{0,t}(x,X_t) dx \nonumber \\ 
        &\quad -\frac{1}{p_t(X_t)^2} \big(\int \nabla \log p_{t|0}(X_t|x) p_{0,t}(x,X_t) dx\big)\big(\int \nabla \log p_{t|0}(X_t|x) p_{0,t}(x,X_t) dx\big)^\intercal \nonumber\\
        &= -\bar{\sigma}_t^{-2} I_d + \bar{\sigma}_t^{-4} \Sigma_t(X_t) ,
    \end{align}
    where the second identity follows from the fact that $p_{t|0}(\cdot|x)=\mathcal{N}(x,\bar{\sigma}_t^2 I_d)$ and the definition of $\Sigma_t(X_t)$ in Lemma \ref{lem:localization}.

    According to Bayes rule, we have
    \begin{align*}
        p_{0|t}(dx | X_t) \propto \exp( -\frac{1}{2}\frac{\lVert X_t-x\rVert^2}{ \bar{\sigma}_t^2 } ) p(dx)
    \end{align*}
    and 
    \begin{align}\label{eq:conditional q}
        q_{0|t}(dx|Y_t) &= Z^{-1} \exp( -\frac{1}{2} \frac{\lVert Y_t-x \rVert^2}{\bar{\sigma}_{T-t}^2} ) p(dx) \nonumber\\
        &= Z_t^{-1}\exp( -\frac{1}{2}\bar{\sigma}_{T-t}^{-2} \lVert x \rVert^2 +\bar{\sigma}_{T-t}^{-2}\langle x,Y_t \rangle ) p(dx)\nonumber \\
        &:=Z_t^{-1} \exp(h_t(x)) p(dx),
    \end{align}
    where $Z_t=\int \exp(h_t(x))p(dx)$ is a (random) normalization constant. From the above computations, we can see that $q_{0|t}(dx|Y_t)=p_{0|T-t}(dx|X_{T-t})$ for all $t\in [0,T]$. Therefore, we have
    \begin{align*}
        \bar{m}_t(Y_t)= \mathbb{E}_{X\sim q_{0|t}(\cdot|Y_t)} [X]=m_{T-t}(X_{T-t}) ,\quad \bar{\Sigma}_t(Y_t) = \text{Cov}_{X\sim q_{0|t}(\cdot|Y_t)}(X)=\Sigma_{T-t}(X_{T-t}),
    \end{align*}
    where the identities hold in  distribution. Therefore, to prove the first statement, it suffices to compute $d \bar{m}_t(Y_t)$. To do so, we first compute $dh_t(x)$, $d[h(x),h(x)]_t$, $d Z_t$ and $d\log Z_t $.
    \begin{align}
       & d h_t(x) = \bar{\sigma}_{T-t}^{-3} \dot{\bar{\sigma}}_{T-t}\lVert x \rVert^2 dt -2\bar{\sigma}_{T-t}^{-3} \dot{\bar{\sigma}}_{T-t}\bar{\sigma}\langle x,Y_t\rangle dt+\bar{\sigma}_{T-t}^{-2} \langle x,dY_t \rangle .\label{eq:diff h_t}
       %& d[h(x),h(x)]_t = \bar{\sigma}_{T-t}^{-4} |x|^2 [d Y, dY]_t .\label{eq:quadratic variation h_t}
    \end{align}
    According to the definition of $Y_t$ and \eqref{eq:score along FD}, we have
    \begin{align*}
        dY_t &= 2\sigma_{T-t}^2 \nabla\log p_{T-t}(Y_t)dt +\sqrt{2\sigma_{T-t}^2}d \Tilde{W}_t \\
        &= -2\sigma_{T-t}^2 \bar{\sigma}_{T-t}^{-2}(Y_t-\bar{m}_t(Y_t)) dt +\sqrt{2\sigma_{T-t}^2}d\Tilde{W}_t.
    \end{align*}
    Therefore 
    \begin{align}
        d[h(x),h(x)]_t = \bar{\sigma}_{T-t}^{-4} |x|^2 [d Y, dY]_t =2 \sigma_{T-t}^2\bar{\sigma}_{T-t}^{-4} \lVert x\rVert^2.\label{eq:quadratic variation h_t}
    \end{align}
    Apply \eqref{eq:diff h_t} and \eqref{eq:quadratic variation h_t} and we get
    \begin{align}
       d Z_t & = \int \exp(h_t(x)) \big( d h_t(x)+\frac{1}{2}d [h(x),h(x)]_t\big) p(dx) \nonumber\\
       & = \bar{\sigma}_{T-t}^{-3} \dot{\bar{\sigma}}_{T-t} \mathbb{E}_{q_{0|t}(\cdot|Y_t)}[\lVert x \rVert^2] Z_t dt -2\bar{\sigma}_{T-t}^{-3} \dot{\bar{\sigma}}_{T-t}\langle Y_t, \bar{m}_t(Y_t) \rangle Z_t dt \nonumber \\
       &\quad + \bar{\sigma}_{T-t}^{-2} \langle \bar{m}_t(Y_t), dY_t \rangle Z_t + \sigma_{T-t}^2 \bar{\sigma}_{T-t}^{-4}\mathbb{E}_{q_{0|t}(\cdot|Y_t)}[\lVert x \rVert^2] Z_t dt, \label{eq: diff Z_t}
    \end{align}
    and 
    \begin{align}
        d\log Z_t & = Z_t^{-1} d Z_t -\frac{1}{2}Z_t^{-2} d[Z,Z]_t \nonumber \\
        &= -2\bar{\sigma}_{T-t}^{-3} \dot{\bar{\sigma}}_{T-t}\langle Y_t, \bar{m}_t(Y_t) \rangle dt + \bar{\sigma}_{T-t}^{-2}\langle \bar{m}_t(Y_t), dY_t \rangle-\sigma_{T-t}^2\bar{\sigma}_{T-t}^{-4} \lVert \bar{m}_t(Y_t) \rVert^2 dt. \label{eq:diff log Z_t}
    \end{align}
    If we further define $R_t(Y_t):= \tfrac{q_{0|t}(dx|Y_t) }{p(dx)}=Z_t^{-1}\exp(h_t(x))$. We have
    \begin{align}
        d R_t(Y_t) & = d \exp ( \log R_t(Y_t)) = R_t(Y_t)d (\log R_t(Y_t)) +\frac{1}{2}R_t(Y_t)d [\log R_t(Y_t),\log R_t(Y_t) ] \nonumber\\
        &= -R_t(Y_t)d (\log Z_t)+R_t(Y_t)dh_t(x)+\frac{1}{2}R_t(Y_t)d [ h_t(x)-\log Z_t, h_t(x)-\log Z_t] \label{eq:diff R_t}
    \end{align}
    With \eqref{eq:diff h_t}, \eqref{eq:quadratic variation h_t}, \eqref{eq: diff Z_t}, \eqref{eq:diff log Z_t} and \eqref{eq:diff R_t}, we have
    \begin{align}
        d\bar{m}_t(Y_t) & = d \int x R_t(Y_t) p(dx) \nonumber\\
        &= \int x \big( -d (\log Z_t)+dh_t(x)+\frac{1}{2}d [ h_t(x)-\log Z_t, h_t(x)-\log Z_t] \big) q_{0|t}(dx|Y_t)\nonumber \\
       &= \sqrt{2}\sigma_{T-t}\bar{\sigma}_{T-t}^{-2} \bar{\Sigma}_t(Y_t) d\Tilde{W}_t,\label{eq: diff mean local}
    \end{align}
    where most terms cancel in the last identity. Therefore, the first statement is proved. Next, we prove the second statement. We have
    \begin{align*}
        \frac{d}{dt} \mathbb{E} [\Sigma_{T-t}(X_{T-t})] &= \frac{d}{dt} \mathbb{E} [\bar{\Sigma}_{t}(Y_t)] =\frac{d}{dt} \mathbb{E} [\Sigma_{T-t}(X_{T-t})] \\
        &= \frac{d}{dt} \mathbb{E}_{Y_t\sim p_{T-t}} [\mathbb{E}_{q_{0|t}(\cdot|Y_t)}[x^\otimes 2]-\bar{m}_t(Y_t)^\otimes 2] \\
        &=\frac{d}{dt} \mathbb{E}_{q_0}[x^\otimes 2]-\frac{d}{dt} \mathbb{E}[\bar{m}_t(Y_t)^\otimes 2] \\
        &= -\mathbb{E}[ -2\bar{m}_t(Y_t)d \bar{m}_t(Y_t)^\intercal +d[\bar{m}_t(Y_t),\bar{m}_t(Y_t)^\intercal] ] \\
        &=2\bar{\sigma}_{T-t}^{-3} \dot{\bar{\sigma}}_{T-t} \mathbb{E}[\bar{\Sigma}_t(Y_t)^2] dt\\
        &=-2\sigma_{T-t}^2\bar{\sigma}_{T-t}^{-4}  \mathbb{E}[\Sigma_t(X_{T-t})^2],
    \end{align*}
    where the second last identity follows from \eqref{eq: diff mean local} and the last identity follows from the definition of $\bar{\sigma}_t$. Last, we reverse the time and get
    \begin{align*}
        \frac{d}{dt} \mathbb{E} [\Sigma_{t}(X_{t})]&= 2 \sigma_t^2\bar{\sigma}_{t}^{-4}  \mathbb{E} [\Sigma_{t}(X_{t})^2]. 
    \end{align*}
    The proof is completed.
\end{proof}

\begin{lemma}\label{lem:discretization error representation} Under the conditions in Lemma \ref{lem:localization}, let $\{Y_t\}_{0\le t\le T}$ be the solution to \eqref{eqn:general_backward_SDE} with $f_t\equiv 0$. Define $L_t:=\nabla \log p_{T-t}(Y_t)$, then for any $t\in [\backt_j,\backt_{j+1})$, we have
    \begin{align*}
        \mathbb{E}[ \lVert L_t-L_{\backt_j} \rVert^2 ] = 2d\int_{\backt_j}^t \sigma_{T-u}^2 \bar{\sigma}_{T-u}^{-4} du + \bar{\sigma}_{T-\backt_j}^{-4}\mathbb{E}[\Tr(\Sigma_{T-\backt_j}(X_{T-\backt_j}))]-\bar{\sigma}_{T-t}^{-4}\mathbb{E}[\Tr(\Sigma_{T-t}(X_{T-t}))]
    \end{align*}
\end{lemma}
\begin{proof}[Proof of Lemma \ref{lem:discretization error representation}]
First, according to the definition of $L_t$ and $Y_t$, it follows from It\^o's lemma that
\begin{align}\label{eq:diff Lt}
    d L_t &= \nabla^2\log p_{T-t}(Y_t)\big( 2\sigma_{T-t}^2\nabla \log p_{T-t}(Y_t)dt+\sqrt{2\sigma_{T-t}} d\Tilde{W}_t \big) \\
    &\quad + \Delta \big( \nabla \log p_{T-t}(Y_t) \big) \sigma_{T-t}^2 dt + \frac{d ( \nabla \log p_{T-t}) }{dt}(Y_t) dt  \\
    &= \sqrt{2\sigma_{T-t}^2} \nabla^2 \log p_{T-t}(Y_t) d\Tilde{W}_t,
\end{align}
where the last step follows from applying the Fokker Planck equation of \eqref{eqn:general_forward_SDE} with $f_t\equiv 0$, i.e., $ \partial _t p_t = \sigma_t^2 \Delta p_t$. Most of the terms are cancelled after applying the Fokker Planck equation. Now, for fixed $s>0$ and $t>s$, define $E_{s,t}:= \mathbb{E}[\lVert L_t-L_s\rVert^2]$. Apply It\^o's lemma and \eqref{eq:diff Lt}, we have
\begin{align}\label{eq:diff Est}
    d E_{s,t} &= 2\mathbb{E}[\langle L_t-L_s, dL_t \rangle] + d [L]_t \nonumber \\
    &= 2\mathbb{E}[\langle L_t-L_s, \sqrt{2\sigma_{T-t}^2}\nabla \log p_{T-t}(Y_t)d\Tilde{W}_t \rangle] + 2\sigma_{T-t}^2\mathbb{E}[\lVert \nabla^2 \log p_{T-t}(Y_t) \rVert_{\F}^2] dt \nonumber\\
    &=2\sigma_{T-t}^2\mathbb{E}[\lVert \nabla^2 \log p_{T-t}(Y_t) \rVert_{\F}^2] dt,
\end{align}
where $\lVert A \rVert_{\F}$ denotes the Frobenius norm of any matrix $A$. According to \eqref{eq:hessian score along FD}, we have
\begin{align*}
    \frac{d E_{s,t}}{dt} &= 2\sigma_{T-t}^2\mathbb{E}[\lVert \nabla^2 \log p_{T-t}(Y_t) \rVert_{\F}^2]  = 2\sigma_{T-t}^2\mathbb{E}[\lVert \nabla^2 \log p_{T-t}(X_{T-t}) \rVert_{\F}^2]  \\
    &= 2\sigma_{T-t}^2 \mathbb{E} [ \lVert -\bar{\sigma}_{T-t}^{-2} I_d + \bar{\sigma}_{T-t}^{-4} \Sigma_{T-t}(X_{T-t})  \rVert_{\F}^2 ]  \\
    &=2d \sigma_{T-t}^2 \bar{\sigma}_{T-t}^{-4} -4 \sigma_{T-t}^2\bar{\sigma}_{T-t}^{-6}\mathbb{E}[\Tr(\Sigma_{T-t}(X_{T-t}))] + 2\sigma_{T-t}^2 \bar{\sigma}_{T-t}^{-8}\mathbb{E}[\Tr(\Sigma_{T-t}(X_{T-t})^2)]\\
    &=2d \sigma_{T-t}^2 \bar{\sigma}_{T-t}^{-4} -4 \sigma_{T-t}^2\bar{\sigma}_{T-t}^{-6}\mathbb{E}[\Tr(\Sigma_{T-t}(X_{T-t}))] -  \bar{\sigma}_{T-t}^{-4}\frac{d}{dt}\mathbb{E}[\Tr(\Sigma_{T-t}(X_{T-t}))],
\end{align*}
where the last identity follows from the proof of Lemma \ref{lem:localization}. Therefore, for any $t\in [\backt_j,\backt_{j+1})$, we have
\begin{align*}
    E_{\backt_j,t} &= 2d\int_{\backt_j}^t \sigma_{T-u}^2 \bar{\sigma}_{T-u}^{-4} du - 4\int_{\backt_j}^t \sigma_{T-u}^2\bar{\sigma}_{T-u}^{-6}\mathbb{E}[\Tr(\Sigma_{T-u}(X_{T-u}))] du \\
    &\quad -  \int_{\backt_j}^t\bar{\sigma}_{T-u}^{-4}\frac{d}{du}\mathbb{E}[\Tr(\Sigma_{T-u}(X_{T-u}))]\\
    &=2d\int_{\backt_j}^t \sigma_{T-u}^2 \bar{\sigma}_{T-u}^{-4} du - 4\int_{\backt_j}^t \sigma_{T-u}^2\bar{\sigma}_{T-u}^{-6}\mathbb{E}[\Tr(\Sigma_{T-u}(X_{T-u}))] du \\
    &\quad - \bar{\sigma}_{T-t}^{-4}\mathbb{E}[\Tr(\Sigma_{T-t}(X_{T-t}))] +\bar{\sigma}_{T-\backt_j}^{-4}\mathbb{E}[\Tr(\Sigma_{T-\backt_j}(X_{T-\backt_j}))]\\
    &\quad +4\int_{\backt_j}^t\sigma_{T-u}^2\bar{\sigma}_{T-u}^{-6}\mathbb{E}[\Tr(\Sigma_{T-u}(X_{T-u}))]\\
    &=2d\int_{\backt_j}^t \sigma_{T-u}^2 \bar{\sigma}_{T-u}^{-4} du  - \bar{\sigma}_{T-t}^{-4}\mathbb{E}[\Tr(\Sigma_{T-t}(X_{T-t}))] +\bar{\sigma}_{T-\backt_j}^{-4}\mathbb{E}[\Tr(\Sigma_{T-\backt_j}(X_{T-\backt_j}))]
\end{align*}
The proof is completed.
\end{proof}

\section{Sampling error for Gaussian data distributions}\label{sec:gaussian data distribution}

In this section, we consider a special case when the data distribution is a mixture of Gaussians, i.e., 
\begin{align}\label{eq:gaussian}
    p(x)= \mathcal{N}(x; m, \sigma^2 I_d), 
\end{align}
where $\mathcal{N}(x; m, \sigma^2 I_d)$ is the density of Gaussian random vector with mean $m$ and covariance $\sigma^2 I_d$. In this case, the score function $\nabla\log p_t(x)$ can be explicitly calculated from any $t>0$, see Lemma \ref{lem:explicit score}. Therefore, the sampling process \eqref{eqn:VESDE_generation} can be implemented with zero score estimation error via the following piecewise SDE: for any $t\in [\backt_j,\backt_{j+1})$,
\begin{align}\label{eqn:exact backward_SDE}
    d \Bar{Y}_t =  2\sigma_{T-t}^2 \nabla\log p_{T-\backt_j}(\Bar{Y}_{\backt_j})  dt +\sqrt{2\sigma_{T-t}^2} d\Bar{W}_t.
\end{align}
The iterates, $(\Bar{Y}_{\backt_j})$, are all Gaussians with explicit means and covariance matrices, see Lemma \ref{lem:gaussian iterates}. As a consequence, we can quantify the quantity, $\KL(p_\delta|q_{T-\delta})$ in Theorem \ref{thm:VESDE_generation} explicitly since both distributions are Gaussians.
\begin{lemma}[KL-divergence error for Gaussian data distribution]\label{lem:exact error}
    Assume the data distribution has density given by \eqref{eq:gaussian}. Let $(\Bar{Y}_{\backt_j})_{j=0}^N$ be defined by \eqref{eqn:exact backward_SDE} with initial condition $\Bar{Y}_0\sim \mathcal{N}(0,\bar{\sigma}_T^2 I_d)$. Denote $q_{t}=\text{Law}(\Bar{Y}_t)$ for all $0\le t\le T-\delta$. Then
    \begin{align}\label{eq:exact error}
        \KL(p_\delta|q_{T-\delta}) = \frac{d}{2}( E_\sigma -1-\log E_\sigma ) + \lVert m \rVert^2 \frac{(\sigma^2+\bar{\sigma}_T^2)^2}{\sigma^2+\bar{\sigma}_\delta^2} E_\sigma ,
    \end{align}
    where $E_\sigma$ is a positive constant depending on the variance schedule $(\bar{\sigma}_{T-\backt_j})_{j=0}^{N}$, given by
    \begin{align*}
        E_\sigma^{-1} = \frac{(\sigma^2+\bar{\sigma}_{{\delta}}^2)\bar{\sigma}_T^2}{(\sigma^2+\bar{\sigma}_{T}^2)^2}+ \sum_{j=0}^{N-1} \frac{(\sigma^2+\bar{\sigma}_{\delta}^2)(\bar{\sigma}_{T-\backt_j}^2-\bar{\sigma}_{T-\backt_{j+1}}^2)}{(\sigma^2+\bar{\sigma}_{T-\backt_{j+1}}^2)^2}.
    \end{align*}
\end{lemma}
\begin{proof}[Proof of Lemma \ref{lem:exact error}]
    $p_\delta=p\ast \mathcal{N}(0,\bar{\sigma}_\delta^2)=\mathcal{N}(m,(\sigma^2+\bar{\sigma}_\delta^2)I_d)$ and $q_{T-\delta}=\text{Law}(\Bar{Y}_{\backt_N})=\mathcal{N}(m_N,\Sigma_N)$ with $(m_N,\Sigma_N)$ given in Lemma \ref{lem:gaussian iterates}. Therefore, we have
\begin{align*}
    \KL(p_\delta|q_{T-\delta}) & = \KL (\mathcal{N}(m,(\sigma^2+\bar{\sigma}_\delta^2)I_d )| \mathcal{N}(m_N,\Sigma_N))  \\
    &=\frac{1}{2}\log \frac{\det(\Sigma_N)}{\det((\sigma^2+\bar{\sigma}_\delta^2)I_d)}-\frac{d}{2}+\frac{1}{2}\Tr((\sigma^2+\bar{\sigma}_\delta^2)\Sigma_N^{-1})+(m_N-m)^\intercal\Sigma_N^{-1}(m-m_N) \\
    &= \frac{d}{2}\log\bigg( \frac{(\sigma^2+\bar{\sigma}_{{\delta}}^2)\bar{\sigma}_T^2}{(\sigma^2+\bar{\sigma}_{T}^2)^2}+ \sum_{j=0}^{N-1} \frac{(\sigma^2+\bar{\sigma}_{\delta}^2)(\bar{\sigma}_{T-\backt_j}^2-\bar{\sigma}_{T-\backt_{j+1}}^2)}{(\sigma^2+\bar{\sigma}_{T-\backt_{j+1}}^2)^2} \bigg)-\frac{d}{2}\\
    &\quad + \frac{d}{2}\bigg( \frac{(\sigma^2+\bar{\sigma}_{{\delta}}^2)\bar{\sigma}_T^2}{(\sigma^2+\bar{\sigma}_{T}^2)^2}+ \sum_{j=0}^{N-1} \frac{(\sigma^2+\bar{\sigma}_{\delta}^2)(\bar{\sigma}_{T-\backt_j}^2-\bar{\sigma}_{T-\backt_{j+1}}^2)}{(\sigma^2+\bar{\sigma}_{T-\backt_{j+1}}^2)^2} \bigg)^{-1}\\
    &\quad +  \bigg( {\bar{\sigma}_T^2}+ {(\sigma^2+\bar{\sigma}_{T}^2)^2}\sum_{j=0}^{N-1} \frac{(\bar{\sigma}_{T-\backt_j}^2-\bar{\sigma}_{T-\backt_{j+1}}^2)}{(\sigma^2+\bar{\sigma}_{T-\backt_{j+1}}^2)^2} \bigg)^{-1} \lVert m \rVert^2.
\end{align*}
\end{proof}

\begin{lemma}[Explicit score function for mixture of Gaussian target]\label{lem:explicit score} Assume the data distribution has density given by \eqref{eq:gaussian}, then the score function is given by
    \begin{align}\label{eq:explicit score}
        \nabla\log p_t(x)= - \frac{x-m}{\sigma^2+\bar{\sigma}_t^2} .
    \end{align}
\end{lemma}
\begin{proof}
    Since the forward process \eqref{eqn:general_forward_SDE} with $f_t\equiv 0$ is the just a process that keeps adding noise, the density $p_t$ along the process is a convolution between data density and a Gaussian density with mean zero and covariance $\bar{\sigma}_t^2 I_d$:
    \begin{align*}
        p_t(x) = p\ast \mathcal{N}(\cdot\ ; 0, \bar{\sigma}_t^2 I_d)(x)=\mathcal{N}(x;m,(\sigma^2+\bar{\sigma}_t^2)I_d) .
    \end{align*}
    Therefore, we have
    \begin{align*}
        \nabla p_t(x) & = (2\pi(\sigma^2+\bar{\sigma}_t^2))^{-d/2}\exp\big(-\frac{\lVert x-m \rVert^2}{2(\sigma^2+\bar{\sigma}_t^2)}\big)\big( -\frac{x-m}{\sigma^2+\bar{\sigma}_t^2} \big) \\
        &=- \frac{x-m}{\sigma^2+\bar{\sigma}_t^2} \mathcal{N}(x;m,(\sigma^2+\bar{\sigma}_t^2)I_d) .   
    \end{align*}
\eqref{eq:explicit score} follows directly from the above computations.
\end{proof}

\begin{lemma}[Gaussian iterates along the trajectory]\label{lem:gaussian iterates} Assume the data distribution has density given by \eqref{eq:gaussian}. Let $(\Bar{Y}_{\backt_j})$ be defined by \eqref{eqn:exact backward_SDE} with initial condition $\Bar{Y}_0\sim \mathcal{N}(0,\bar{\sigma}_T^2 I_d)$. Then for all $0\le j\le N$, $\Bar{Y}_{\backt_j}\sim \mathcal{N}(m_j, \Sigma_j )$ with 
\begin{align}
   m_j &= \frac{\bar{\sigma}_{T-\backt_{0}}^2+\bar{\sigma}_{T-\backt_{j}}^2}{\sigma^2+\bar{\sigma}_{T-\backt_0}^2} m , \label{eq:mean} \\
   \Sigma_j &= \bigg(\frac{(\sigma^2+\bar{\sigma}_{T-\backt_{j}}^2)^2\bar{\sigma}_T^2}{(\sigma^2+\bar{\sigma}_{T-\backt_{0}}^2)^2}+ \sum_{l=0}^{j-1} \frac{(\sigma^2+\bar{\sigma}_{T-\backt_{j}}^2)^2(\bar{\sigma}_{T-\backt_l}^2-\bar{\sigma}_{T-\backt_{l+1}}^2)}{(\sigma^2+\bar{\sigma}_{T-\backt_{l+1}}^2)^2}\bigg)I_d.\label{eq:covariance}
\end{align}
\end{lemma}
\begin{proof}[Proof of Lemma \ref{lem:gaussian iterates}] According to Lemma \ref{lem:explicit score}, \eqref{eqn:exact backward_SDE} can be written as
\begin{align*}
     d \Bar{Y}_t = -2\sigma_{T-t}^2 \frac{\Bar{Y}_{\backt_j}-m}{\sigma^2+\bar{\sigma}_{T-\backt_j}^2}  dt +\sqrt{2\sigma_{T-t}^2} d\Bar{W}_t,
\end{align*}
which implies that for any $j\in 0,1,\cdots,N-1$:
\begin{align}
    \Bar{Y}_{\backt_{j+1}}-\Bar{Y}_{\backt_j}& = -2\frac{\int_{\backt_j}^{\backt_{j+1}} \sigma_{T-t}^2 dt}{\sigma^2+\bar{\sigma}_{T-\backt_j}^2}\big(\Bar{Y}_{\backt_j}-m\big)+\sqrt{2\int_{\backt_j}^{\backt_{j+1}} \sigma_{T-t}^2 dt}  U_{j+1} \nonumber\\
    &= -\frac{\bar{\sigma}_{T-\backt_j}^2-\bar{\sigma}_{T-\backt_{j+1}}^2}{\sigma^2+\bar{\sigma}_{T-\backt_j}^2}\big(\Bar{Y}_{\backt_j}-m\big)+\sqrt{\bar{\sigma}_{T-\backt_j}^2-\bar{\sigma}_{T-\backt_{j+1}}^2} U_{j+1},\nonumber\\
   \implies \qquad  \Bar{Y}_{\backt_{j+1}} &= \big(1-\frac{\bar{\sigma}_{T-\backt_j}^2-\bar{\sigma}_{T-\backt_{j+1}}^2}{\sigma^2+\bar{\sigma}_{T-\backt_j}^2}\big)\Bar{Y}_{\backt_{j}}+ \frac{\bar{\sigma}_{T-\backt_j}^2-\bar{\sigma}_{T-\backt_{j+1}}^2}{\sigma^2+\bar{\sigma}_{\backt_j}^2} m +\sqrt{\bar{\sigma}_{T-\backt_j}^2-\bar{\sigma}_{T-\backt_{j+1}}^2} U_{j+1} \label{eq:induction}
\end{align}
where $(U_j)_{j=1}^N$ are i.i.d. standard Gaussian vectors in $\mathbb{R}^d$. Since $\Bar{Y}_{\backt_{0}}$ is Gaussian, by induction, we prove that $\Bar{Y}_{\backt_{j}}$ is Gaussian for all $j=1,\cdots,N$. Denote $\Bar{Y}_{\backt_{j}}\sim \mathcal{N}(m_j,\Sigma_j)$. According to \eqref{eq:induction} and the independence between $U_{j+1}$ and $\Bar{Y}_{\backt_{j}}$, we have
\begin{align*}
    & m_{j+1} = \big(1-\frac{\bar{\sigma}_{T-\backt_j}^2-\bar{\sigma}_{T-\backt_{j+1}}^2}{\sigma^2+\bar{\sigma}_{T-\backt_j}^2}\big)m_j+ \frac{\bar{\sigma}_{T-\backt_j}^2-\bar{\sigma}_{T-\backt_{j+1}}^2}{\sigma^2+\bar{\sigma}_{\backt_j}^2} m , \\
    \implies\qquad & m_{j+1}-m = \frac{\sigma^2+\bar{\sigma}_{T-\backt_{j+1}}^2}{\sigma^2+\bar{\sigma}_{T-\backt_j}^2} (m_j-m), \\
    \implies \qquad & m_{j} = \frac{\sigma^2+\bar{\sigma}_{T-\backt_{j}}^2}{\sigma^2+\bar{\sigma}_{T-\backt_0}^2} (m_0-m)+m = \frac{\bar{\sigma}_{T-\backt_{0}}^2+\bar{\sigma}_{T-\backt_{j}}^2}{\sigma^2+\bar{\sigma}_{T-\backt_0}^2} m.
\end{align*}
Again, according to \eqref{eq:induction} and the independence between $U_{j+1}$ and $\Bar{Y}_{\backt_{j}}$, we get a relation between consecutive covariance matrices:
\begin{align*}
    & \Sigma_{j+1} = \big(1-\frac{\bar{\sigma}_{T-\backt_j}^2-\bar{\sigma}_{T-\backt_{j+1}}^2}{\sigma^2+\bar{\sigma}_{T-\backt_j}^2}\big)^2 \Sigma_j+ (\bar{\sigma}_{T-\backt_j}^2-\bar{\sigma}_{T-\backt_{j+1}}^2)  I_d , \\
    \implies\qquad & \frac{\Sigma_{j+1}}{(\sigma^2+\bar{\sigma}_{T-\backt_{j+1}}^2)^2} = \frac{\Sigma_{j}}{(\sigma^2+\bar{\sigma}_{T-\backt_{j}}^2)^2} + \frac{\bar{\sigma}_{T-\backt_j}^2-\bar{\sigma}_{T-\backt_{j+1}}^2}{(\sigma^2+\bar{\sigma}_{T-\backt_{j+1}}^2)^2} I_d, \\
    \implies \qquad & \frac{\Sigma_{j}}{(\sigma^2+\bar{\sigma}_{T-\backt_{j}}^2)^2} = \frac{\Sigma_{0}}{(\sigma^2+\bar{\sigma}_{T-\backt_{0}}^2)^2}+\sum_{l=0}^{j-1} \frac{\bar{\sigma}_{T-\backt_l}^2-\bar{\sigma}_{T-\backt_{l+1}}^2}{(\sigma^2+\bar{\sigma}_{T-\backt_{l+1}}^2)^2}I_d ,\\
    \implies \qquad &\Sigma_j = \bigg(\frac{(\sigma^2+\bar{\sigma}_{T-\backt_{j}}^2)^2\bar{\sigma}_T^2}{(\sigma^2+\bar{\sigma}_{T-\backt_{0}}^2)^2}+ \sum_{l=0}^{j-1} \frac{(\sigma^2+\bar{\sigma}_{T-\backt_{j}}^2)^2(\bar{\sigma}_{T-\backt_l}^2-\bar{\sigma}_{T-\backt_{l+1}}^2)}{(\sigma^2+\bar{\sigma}_{T-\backt_{l+1}}^2)^2}\bigg)I_d.
\end{align*} 
\end{proof}

\section{Full error analysis}
\begin{proof}[Proof of Theorem~\ref{cor:full_error_analysis}]
We only need to deal with $E_S$. By applying the same schedules to training objective, we obtain
\begin{align*}
    E_S&=\sum_{j=0}^{N-1}\frac{\sigma^2_{t_{N-j}}}{w(t_{N-j})}\cdot w(t_{N-j})(t_{N-j}-t_{N-j-1})\frac{1}{\bar{\sigma}_{t_{N-j}}}\mathbb{E}_{X_0}\mathbb{E}_{\xi}\|\bar{\sigma}_{t_{N-j}} s(\theta;t_{N-j},X_{t_{N-j}})+\xi\|^2\\
    &\qquad+\sum_{j=0}^{N-1}\frac{\sigma^2_{t_{N-j}}}{w(t_{N-j})}\cdot w(t_{N-j})(t_{N-j}-t_{N-j-1})\cdot C\\
    &\le \max_j \frac{\sigma^2_{t_{N-j}}}{w(t_{N-j})}\cdot (\bar{\mathcal{L}}(W)+\bar{C}).
\end{align*}
Together with
\begin{align*}
    \bar{\mathcal{L}}(W^{(k)})+\bar{C}&\le |\bar{\mathcal{L}}(W^{(K)})-\lem(W^{(K)})+\lem(\theta^*)-\bar{\mathcal{L}}(\theta^*)|+ |\lem(W^{(K)})-\lem(\theta^*)|\\
    &\qquad+|\bar{\mathcal{L}}(\theta^*)-\bar{\mathcal{L}}(\theta_\mathcal{F})|+|\bar{\mathcal{L}}(\theta_\mathcal{F})+\bar{C}|,
\end{align*}
    we have the result.
\end{proof}

\section{Proofs for Section~\ref{subsec:weighting}}

\subsection{Proof of ``bell-shaped'' curve}
\begin{proof}[Proof of Proposition~\ref{prop:inverse_bell_shaped_loss}]

Fix $x_i,\xi_{ij},\bar{\sigma}_{t_N}$. By definition of the network $S(\theta;t_j,X_{ij})$, it is continuous with respect to $X_{ij}$.

For 1, $X_{ij}=x_i+\bar{\sigma}_{t_j}\xi_{ij}$, and thus $S(\theta;t_j,X_{ij})$ is also continuous w.r.t. $\bar{\sigma}_{t_j}$. Also, since $\bar{\sigma}_{t_j}\in[0,\bar{\sigma}_{t_N}]$, there exists $M_0>0$, s.t., 
\begin{align*}
    S(\theta;t_j,x_i+\bar{\sigma}_{t_j}\xi_{ij})\in[-M_0,M_0]^d.
\end{align*}
Then for any $\epsilon_1>0$, there exists $\delta=\frac{\epsilon_1}{\sqrt{d}M_0}>0$, s.t., $\forall\  0\le \bar{\sigma}_{t_j}< \delta_1$, we have
\begin{align*}
    \|\bar{\sigma}_{t_j}S(\theta;t_j,x_i+\bar{\sigma}_{t_j}\xi_{ij})+\xi_{ij}\|&\ge \|\xi_{ij}\|-\|\bar{\sigma}_{t_j}S(\theta;t_j,x_i+\bar{\sigma}_{t_j}\xi_{ij})\|\\
    &\ge \|\xi_{ij}\|-\sqrt{d}M_0\bar{\sigma}_{t_j}\\
    &\ge \|\xi_{ij}\|-\epsilon_1.
\end{align*}

For 2, by the positive homogeniety of ReLU,
\begin{align*}
    S(\theta;t_j,x_i+\bar{\sigma}_{t_j}\xi_{ij})=\bar{\sigma}_{t_j}S
    \left(\theta;t_j,\frac{x_i}{\bar{\sigma}_{t_j}}+\xi_{ij}\right).
\end{align*}
Consider $\bar{\sigma}_{t_j}\ge M$, for some $M>0$. Then
\begin{align}
\label{eqn:prop_2_M}
    \|\bar{\sigma}_{t_j}S(\theta;t_j,x_i+\bar{\sigma}_{t_j}\xi_{ij})+\xi_{ij}\|&=\left\|\bar{\sigma}_{t_j}^2S
    \left(\theta;t_j,\frac{x_i}{\bar{\sigma}_{t_j}}+\xi_{ij}\right)+\xi_{ij}\right\|\notag\\
    &\ge \left\|\bar{\sigma}_{t_j}^2S
    \left(\theta;t_j,\frac{x_i}{\bar{\sigma}_{t_j}}+\xi_{ij}\right)\right\|-\left\|\xi_{ij}\right\|\notag\\
    &\ge M^2\left\|S
    \left(\theta;t_j,\frac{x_i}{\bar{\sigma}_{t_j}}+\xi_{ij}\right)\right\|-\left\|\xi_{ij}\right\|.
\end{align}

For any $y\in\mathcal{D}(M)$, where $\mathcal{D}(M)=\{y\in\mathbb{R}^d:y_s\in[(\xi_{ij})_s-|(x_i)_s|/M,(\xi_{ij})_s+|(x_i)_s|/M],\ \forall\ s=1,\cdots,d\}$, 
\begin{align}
\label{eqn:prop_2_S}
    \|S(\theta;t_j,y)\|\ge \|S(\theta;t_j,\xi_{ij})\|-\|S(\theta;t_j,y)-S(\theta;t_j,\xi_{ij})\|.
\end{align}
Since $S$ is differentiable a.e., by the fundamental theorem of calculus, 
\begin{align*}
    S(\theta;t_j,y)-S(\theta;t_j,\xi_{ij})=\int_{\xi_{ij}}^y S_x'(\theta;t_j,x)dx. 
\end{align*}
Then
\begin{align}
\label{eqn:prop_2_difference}
    \| S(\theta;t_j,y)-S(\theta;t_j,\xi_{ij})\|\le\frac{1}{M}\cdot M_1,
\end{align}
where $M_1=\max_s (x_i)_s\cdot {\rm ess}\sup_{x\in\mathcal{D}(M_2)}\|S_x'(\theta;t_j,x)\|<+\infty$ for some fixed $0<M_2<M$.

Combining \eqref{eqn:prop_2_M},\eqref{eqn:prop_2_S}, and \eqref{eqn:prop_2_difference}, we have
\begin{align*}
    \|\bar{\sigma}_{t_j}S(\theta;t_j,x_i+\bar{\sigma}_{t_j}\xi_{ij})+\xi_{ij}\|&\ge M^2\left( \|S(\theta;t_j,\xi_{ij})\|-\frac{M_1}{M}  \right)-\|\xi_{ij}\|\\
    &= M^2\left( \|S(\theta;t_j,\xi_{ij})\|-\frac{M_1}{M}-\frac{\|\xi_{ij}\|}{M^2}  \right).
\end{align*}
Then $\forall\ \epsilon_2>0$, there exists $M=\max\left\{\frac{M_1+\sqrt{M_1^2+4\epsilon_2\|\xi_{ij}\|}}{2\epsilon_2},M_2\right\}>0$, s.t., when $\bar{\sigma}_{t_j}>M$, we have
\begin{align*}
    \|\bar{\sigma}_{t_j}S(\theta;t_j,x_i+\bar{\sigma}_{t_j}\xi_{ij})+\xi_{ij}\|\ge M^2(\|S(\theta;t_j,\xi_{ij})\|-\epsilon_2).
\end{align*}

    % For 1 and 2, the proof is simply the $\epsilon$-$\delta$ and $\epsilon$-$N$ version of the definition of limit, respectively. For 2, the continuity and positive homogeneity of ReLU function is also needed.
    % For 3, consider the data set with input data $X_{\rm in}=\{e^{-\mu_{t_j}}x_i+\bar{\sigma}_{t_j}\xi_{ij}\}_{ij}$, and the output data $X_{\rm out}=\{-\frac{\xi_{ij}}{\bar{\sigma}_{t_j}}\}_{ij}$ for $i=1,\cdots,n$ and $j=N_1,\cdots,N_2$, where $N_1>0$ and $N_2<T$. The $N_1$ and $N_2$ is chosen so that $C_3 \|\xi_{ij}\|\le\|\frac{\xi_{ij}}{\bar{\sigma}_{t_j}}\|\le C_4 \|\xi_{ij}\|$, where $\frac{1}{C_3}=\mathcal{O}(1),C_4=\mathcal{O}(1)$. By implicit function theorem, there exists some $C_3,C_4$ (and consequently $N_1,N_2$) and a continuously differentiable function $G$, s.t. $G(e^{-\mu_{t_j}}x_i+\bar{\sigma}_{t_j}\xi_{ij})=-\frac{\xi_{ij}}{\bar{\sigma}_{t_j}}$ for $i=1,\cdots,n$ and $j=N_1,\cdots,N_2$. Then by the universal approximation theorem with $m=\Theta(d)$~\citep{kidger2020universal}, for any ${C_3}{\epsilon}>0$, there exists a neural network $F$, s.t.
    % \begin{align*}
    %     \sup_{x\in X_{\rm in}}\|F(x)-G(x)\|\le {C_3}{\epsilon}\\
    %    \left \|F(e^{-\mu_{t_j}}x_i+\bar{\sigma}_{t_j}\xi_{ij})+\frac{\xi_{ij}}{\bar{\sigma}_{t_j}}\right\|\le {C_3}{\epsilon}\\
    %     \left \|\bar{\sigma}_{t_j}F(e^{-\mu_{t_j}}x_i+\bar{\sigma}_{t_j}\xi_{ij})+{\xi_{ij}}{}\right\|\le \max_j\bar{\sigma}_{t_j}\cdot{C_3}{\epsilon} \le\epsilon
    % \end{align*}
\end{proof}

\subsection{Proof of optimal rate}
\begin{proof}[Proof of Corollary~\ref{cor:optimal_rate}]
  If $|f(\theta^{(k)};i,j)-f(\theta^{(k)};l,s)|\le \epsilon$ for all $i,j,l,s$ and $k>K$, then by Lemma~\ref{lem:lower_bound} and~\ref{lem:perturb_upper_lower_bound_grad}, we choose the maximum $f(\theta^{(k)};i,j)$ for the lower bound, which is of order $\mathcal{O}(\epsilon)$ away from the other $f(\theta^{(k)};i,j)'s$. Therefore, we can take $j^*(k)=\arg\max_jf(\theta^{(k)};i,j)$ and absorb the $\mathcal{O}(\epsilon)$ error in constant factors. Then the result naturally follows. 
\end{proof}

\subsection{Proof of comparisons of $E_S$}
Recall that the training objective of EDM is defined in the following
\begin{align*}
    \mathbb{E}_{\bar{\sigma}\sim p_{\rm train}}\mathbb{E}_{y,n}\lambda(\bar{\sigma})\|D_\theta(y+n;\bar{\sigma})-y\|^2=\frac{1}{Z_1}\int \frac{1}{\bar{\sigma}}e^{-\frac{(\log \bar{\sigma}-P_{\rm mean})^2}{2P_{\rm std}^2}}\cdot\frac{\bar{\sigma}^2+\sigma_{\rm data}^2}{\bar{\sigma}^2\sigma_{\rm data}^2}\cdot\bar{\sigma}^2\mathbb{E}_{X_0,\xi}\|\bar{\sigma} s(\theta;t,X_{t})+\xi\|^2\, d\bar{\sigma}.
\end{align*}
Let $\beta_j=C_1\beta_{\rm EDM}$, i.e.,
\begin{align*}
    \frac{w(t_j)}{\bar{\sigma}_{t_j}}(t_j-t_{j-1})&=C_1\cdot e^{-\frac{(\log \bar{\sigma}_{t_j}-P_{\rm mean})^2}{2P_{\rm std}^2}}\cdot\frac{\bar{\sigma}^2_{t_j}+\sigma_{\rm data}^2}{\bar{\sigma}^2_{t_j}\sigma_{\rm data}^2}\cdot\bar{\sigma}_{t_j}\\
    w(t_j)&=C_1\cdot\frac{\bar{\sigma}_{t_j}}{t_j-t_{j-1}} \cdot e^{-\frac{(\log \bar{\sigma}_{t_j}-P_{\rm mean})^2}{2P_{\rm std}^2}}\cdot\frac{\bar{\sigma}^2_{t_j}+\sigma_{\rm data}^2}{\bar{\sigma}^2_{t_j}\sigma_{\rm data}^2}\cdot\bar{\sigma}_{t_j}
\end{align*}

\textbf{EDM.} Consider $\bar{\sigma}_{t}=t$ and $t_j=\left( \bar{\sigma}_{\max}^{1/\rho}-(\bar{\sigma}_{\max}^{1/\rho}-\bar{\sigma}_{\min}^{1/\rho})\frac{N-j}{N} \right)^{\rho}$ for $j=0,\cdots,N$. Then
\begin{align*}
    {w(t_j)}&=C_1\cdot \frac{t_j}{t_j-t_{j-1}}\cdot e^{-\frac{(\log {t_k}-P_{\rm mean})^2}{2P_{\rm std}^2}}\cdot\frac{{t_j}^2+\sigma_{\rm data}^2}{{t_j}\sigma_{\rm data}^2}\\
    &=C_1\cdot \frac{1}{\left( \bar{\sigma}_{\max}^{1/\rho}-(\bar{\sigma}_{\max}^{1/\rho}-\bar{\sigma}_{\min}^{1/\rho})\frac{N-j}{N} \right)^{\rho}-\left( \bar{\sigma}_{\max}^{1/\rho}-(\bar{\sigma}_{\max}^{1/\rho}-\bar{\sigma}_{\min}^{1/\rho})\frac{N-j+1}{N} \right)^{\rho}}\\
    &\qquad\cdot e^{-\frac{\left(\log {\left( \bar{\sigma}_{\max}^{1/\rho}-(\bar{\sigma}_{\max}^{1/\rho}-\bar{\sigma}_{\min}^{1/\rho})\frac{N-j}{N} \right)^{\rho}}-P_{\rm mean}\right)^2}{2P_{\rm std}^2}}\cdot\frac{{\left( \bar{\sigma}_{\max}^{1/\rho}-(\bar{\sigma}_{\max}^{1/\rho}-\bar{\sigma}_{\min}^{1/\rho})\frac{N-j}{N} \right)^{2\rho}}+\sigma_{\rm data}^2}{\sigma_{\rm data}^2}
\end{align*}
Then the maximum of $\frac{\sigma_{t_j}^2}{w(t_j)}= \frac{{t_j}}{w(t_j)}$ appears at $j=N$
\begin{align*}
    \max_j\frac{\sigma_{t_j}^2}{w(t_j)}= \max_j\frac{{t_j}}{w(t_j)}=\frac{\bar{\sigma}_{\max}\sigma _{\text{data}}^2 e^{\frac{\left(P_{\text{mean}}-\log  \bar{\sigma }_{\max }\right){}^2}{2 P_{\text{std}}^2}}}{C_1(\bar{\sigma }_{\max }^2+\sigma _{\text{data}}^2)}\cdot\left( \bar{\sigma }_{\max }-\left(\bar{\sigma }_{\max }^{1/\rho }-\frac{\bar{\sigma }_{\max }^{1/\rho }-\bar{\sigma }_{\min }^{1/\rho }}{N}\right)^{\rho }\right)
\end{align*}

\textbf{\citet{song2020score}.} 
% Its objective is
% \begin{align*}
%     \mathbb{E}_{\bar{\sigma}\sim p_{\rm train}}\mathbb{E}_{X_0,\xi}\|\bar{\sigma} s(\theta;t,X_{t})+\xi\|^2=\frac{1}{Z_2}\int_{\bar{\sigma}_{\min}}^{\bar{\sigma}_{\max}} \frac{1}{\bar{\sigma}}\mathbb{E}_{X_0,\xi}\|\bar{\sigma} s(\theta;t,X_{t})+\xi\|^2\, d\bar{\sigma}
% \end{align*}
% Let
% \begin{align*}
%     \frac{w(t_k)}{\bar{\sigma}_{t_k}}(t_k-t_{k-1})=C_2\cdot \frac{1}{\bar{\sigma}_{t_k}}
% \end{align*}
Consider $\bar{\sigma}_{t}=\sqrt{t}$ and $t_j=\bar{\sigma}_{\max}^2\left(\frac{\bar{\sigma}_{\min}^2}{\bar{\sigma}_{\max}^2}\right)^{\frac{N-j}{N}}$ for $j=0,\cdots,N$. Then
\begin{align*}
    w(t_j)&=C_1\cdot\frac{\sqrt{t_j}}{t_j-t_{j-1}}  \cdot e^{-\frac{(\log \sqrt{t_j}-P_{\rm mean})^2}{2P_{\rm std}^2}}\cdot\frac{{t_j}+\sigma_{\rm data}^2}{\sqrt{t_j}\sigma_{\rm data}^2}\\
    &=C_1\cdot\frac{1}{\bar{\sigma}_{\max}^2\left(\frac{\bar{\sigma}_{\min}^2}{\bar{\sigma}_{\max}^2}\right)^{\frac{N-j}{N}}-\bar{\sigma}_{\max}^2\left(\frac{\bar{\sigma}_{\min}^2}{\bar{\sigma}_{\max}^2}\right)^{\frac{N-j+1}{N}}} \cdot e^{-\frac{\left(\log \bar{\sigma}_{\max}\left(\frac{\bar{\sigma}_{\min}}{\bar{\sigma}_{\max}}\right)^{\frac{N-j}{N}}-P_{\rm mean}\right)^2}{2P_{\rm std}^2}}\cdot\frac{\bar{\sigma}_{\max}^2\left(\frac{\bar{\sigma}_{\min}^2}{\bar{\sigma}_{\max}^2}\right)^{\frac{N-j}{N}}+\sigma_{\rm data}^2}{\sigma_{\rm data}^2}
\end{align*}
Then
\begin{align*}
   \max_j \frac{\sigma_{t_j}^2}{w(t_j)}=\max_j\frac{1}{2w(t_j)}=\frac{\bar{\sigma}_{\max}\sigma _{\text{data}}^2 e^{\frac{\left(P_{\text{mean}}-\log  \bar{\sigma }_{\max }\right){}^2}{2 P_{\text{std}}^2}}}{C_1(\bar{\sigma }_{\max }^2+\sigma _{\text{data}}^2)}\cdot \frac{1}{2}\left(\bar{\sigma }_{\max }-\bar{\sigma }_{\max } \left(\frac{\bar{\sigma }_{\min }^2}{\bar{\sigma }_{\max }^2}\right){}^{1/N}\right)
\end{align*}

\section{Proofs for Section~\ref{subsec:time variance schedule}}

\subsection{Proof when $E_I+E_D$ dominates.}\label{append:EI+ED dominates}
Under the EDM choice of variance, $\bar{\sigma}_t=t$  for all $t\in [0,T]$, and study the optimal time schedule when $E_D+E_I$ dominates. First, it follows from Theorem~\ref{thm:VESDE_generation} that
\begin{align*}
    E_I+E_D \lesssim  & \frac{\mathrm{m}_2^2}{T^2}+d\sum_{j=0}^{N-1} \frac{\gamma_j^2}{(T-\backt_j)^2} \\
    &+ (\mathrm{m}_2^2+d)\big(\sum_{T-\backt_j\ge 1}  \frac{\gamma_j^2}{(T-\backt_j)^4}+\frac{\gamma_j^3}{(T-\backt_j)^5} + \sum_{T-\backt_j< 1}  \frac{\gamma_j^2}{(T-\backt_{j})^2}+\frac{\gamma_j^3}{(T-\backt_j)^3} \big) 
\end{align*}
Based on the above time schedule dependent error bound, we quantify the errors under polynomial time schedule and exponential time schedule.

\textbf{Polynomial time schedule.} we consider $T-\backt_j=(\delta^{1/a}+(N-j)h)^a$ with $h=\tfrac{T^{1/a}-\delta^{1/a}}{N}$ and $a>1$, $\gamma_j=a(\delta^{1/a}+(N-j-\vartheta)h)^{a-1}h$ for some $\vartheta\in (0,1)$. We have $\gamma_j/h\sim a(T-\backt_j)^{\frac{a-1}{a}}$ and 
\begin{align*}
    E_I+E_D \lesssim \frac{\mathrm{m}_2^2}{T^2} +\frac{d a^2 T^{\frac{1}{a}} }{\delta^{\frac{1}{a}}N}+(\mathrm{m}_2^2+d)\big( \frac{a^2T^{\frac{1}{a}}}{\delta^{\frac{1}{a}} N}+\frac{a^3T^{\frac{2}{a}}}{\delta^{\frac{2}{a}}N^2}\big)
\end{align*}
Therefore, to obtain $E_I+E_D\lesssim \varepsilon$, it suffices to require $T=\Theta(\tfrac{\mathrm{m}_2}{\varepsilon^{1/2}})$ and the iteration complexity 
\begin{align*}  N=\Omega\big(a^2\big(\frac{\mathrm{m}_2}{\delta \varepsilon^{\frac{1}{2}}}\big)^{\frac{1}{a}}\frac{\mathrm{m}_2^2+d}{\varepsilon}\big)
\end{align*}
For fixed $\mathrm{m}_2,\delta$ and $\varepsilon$, optimal value of $a$ that minimizes the iteration complexity $N$ is $a=\frac{1}{2}\ln(\tfrac{\mathrm{m}_2}{\delta\varepsilon^{1/2}})$. Once we let $\delta=\bar{\sigma}_{\min}$, $T=\bar{\sigma}_{\max}=\Theta(\tfrac{\mathrm{m}_2}{\varepsilon^{1/2}})$ and $a=\rho$, the iteration complexity is 
\[
N=\Omega\big( \frac{\mathrm{m}_2^2 \vee d}{d}\rho^2\big(\frac{\bar{\sigma}_{\max}}{\bar{\sigma}_{\min}}\big)^{1/\rho}\bar{\sigma}_{\max}^2\big),
\]
and it is easy to see that our theoretical result supports what's empirically observed in EDM that there is an optimal value of $\rho$ that minimizes the FID. 

\textbf{Exponential time schedule.} we consider $\gamma_j=\kappa (T-\backt_j) $ with $\kappa=\tfrac{\ln(T/\delta)}{N}$, we have
\begin{align*}
    E_I+E_D \lesssim \frac{\mathrm{m}_2^2}{T^2}+ \frac{d \ln(T/\delta)^2 }{N}+(\mathrm{m}_2^2+d) \big(\frac{\ln (T/\delta)^2}{N} +\frac{\ln (T/\delta)^3}{N^2}\big)
\end{align*}
Therefore, to obtain $E_I+E_D\lesssim \varepsilon$, it suffices to require $T=\Theta(\tfrac{\mathrm{m}_2}{\varepsilon^{\frac{1}{2}}})$ and the iteration complexity 
\begin{align*}
    N=\Omega\big(\frac{\mathrm{m}_2^2+d}{\varepsilon}\ln(\frac{\mathrm{m}_2}{\delta \varepsilon^{\frac{1}{2}}})^2\big)
\end{align*}
When $\mathrm{m}_2\le O(\sqrt{d})$, the exponential time schedule is asymptotic optimal, hence it is better than the polynomial time schedule when the initilization error and discretization error dominate. Once we let $\delta=\bar{\sigma}_{\min}$, $T=\bar{\sigma}_{\max}=\Theta(\tfrac{\mathrm{m}_2}{\varepsilon^{1/2}})$, the iteration complexity is 
\[
N=\Omega\big( \frac{\mathrm{m}_2^2 \vee d}{d}\ln \big(\frac{\bar{\sigma}_{\max}}{\bar{\sigma}_{\min}}\big)^{2}\bar{\sigma}_{\max}^2\big).
\]
Now we adopt the variance schedule in \citep{song2020score}, $\bar{\sigma}_t=\sqrt{t}$ for all $t\in [0,T]$, it follows from Theorem~\ref{thm:VESDE_generation} that
\begin{align*}
    E_I+E_D \lesssim \frac{\mathrm{m}_2^2}{T} + d\sum_{j=0}^{N-1} \frac{\gamma_j^2}{(T-\backt_j)^2}+(\mathrm{m}_2^2+d) \big(\sum_{T-\backt_j\ge 1}  \frac{\gamma_j^2}{(T-\backt_j)^3}+ \sum_{T-\backt_j< 1}  \frac{\gamma_j^2}{(T-\backt_{j})^2}\big) 
\end{align*}
\textbf{Polynomial time schedule.} we consider $T-\backt_j=(\delta^{1/a}+(N-j)h)^a$ with $h=\tfrac{T^{1/a}-\delta^{1/a}}{N}$ and $a>1$, $\gamma_j=a(\delta^{1/a}+(N-j-\vartheta)h)^{a-1}h$ for some $\vartheta\in (0,1)$. We have $\gamma_j/h\sim a(T-\backt_j)^{\frac{a-1}{a}}$ and 
\begin{align*}
    E_I+E_D \lesssim \frac{\mathrm{m}_2^2}{T} +\frac{d a^2 T^{\frac{1}{a}} }{\delta^{\frac{1}{a}}N}+(\mathrm{m}_2^2+d) \frac{a^2T^{\frac{1}{a}}}{\delta^{\frac{1}{a}} N}
\end{align*}
Therefore, to obtain $E_I+E_D\lesssim \varepsilon$, it suffices to require $T=\Theta(\tfrac{\mathrm{m}_2^2}{\varepsilon})$ and the iteration complexity 
\begin{align*}  
N=\Omega\big(a^2\big(\frac{\mathrm{m}_2^2}{\delta \varepsilon}\big)^{\frac{1}{a}}\frac{\mathrm{m}_2^2+d}{\varepsilon}\big)
\end{align*}
 Once we let $\delta=\bar{\sigma}_{\min}^2$, $T=\bar{\sigma}_{\max}^2=\Theta(\tfrac{\mathrm{m}_2^2}{\varepsilon})$ and $a=\rho$, the iteration complexity is 
\[
N=\Omega\big( \frac{\mathrm{m}_2^2 \vee d}{d}\rho^2\big(\frac{\bar{\sigma}_{\max}}{\bar{\sigma}_{\min}}\big)^{2/\rho}\bar{\sigma}_{\max}^2\big).
\]
Compared to exponential time schedule with the EDM choice of variance schedule, this iteration complexity is worse up to a factor $\big(\frac{\bar{\sigma}_{\max}}{\bar{\sigma}_{\min}}\big)^{1/\rho}$.

\textbf{Exponential time schedule.} we consider $\gamma_j=\kappa (T-\backt_j) $ with $\kappa=\tfrac{\ln(T/\delta)}{N}$, we have
\begin{align*}
    E_I+E_D \lesssim \frac{\mathrm{m}_2^2}{T}+ \frac{d \ln(T/\delta)^2 }{N}+(\mathrm{m}_2^2+d)\frac{\ln (T/\delta)^2}{N} 
\end{align*}
Therefore, to obtain $E_I+E_D\lesssim \varepsilon$, it suffices to require $T=\Theta(\tfrac{\mathrm{m}_2^2}{\varepsilon})$ and the iteration complexity 
\begin{align*}
    N=\Omega\big(\frac{\mathrm{m}_2^2+d}{\varepsilon}\ln(\frac{\mathrm{m}_2^2}{\delta \varepsilon})^2\big)
\end{align*}
 Once we let $\delta=\bar{\sigma}_{\min}^2$, $T=\bar{\sigma}_{\max}^2=\Theta(\tfrac{\mathrm{m}_2^2}{\varepsilon})$ and $a=\rho$, the iteration complexity is 
\[
N=\Omega\big( \frac{\mathrm{m}_2^2 \vee d}{d}\ln\big(\frac{\bar{\sigma}_{\max}}{\bar{\sigma}_{\min}}\big)^{2}\bar{\sigma}_{\max}^2\big).
\]
Compared to exponential time schedule with the EDM choice of variance schedule, this iteration complexity has the same dependence on dimension parameters $\mathrm{m}_2,d$ and the minimal/maximal variance $\bar{\sigma}_{\min},\bar{\sigma}_{\max}$.

\textbf{Optimality of Exponential time schedule.} For simplicity, we assume $\mathrm{m}_2^2=\mathcal{O}(d)$. Then under both schedules in \citep{karras2022elucidating} and \citep{song2020score},  $E_I$s only dependent on $T$, and are independent of the time schedule. Both $E_D$s satisfy
\[
E_D \lesssim d\sum_{j=0}^{N-1} \frac{\gamma_j^2}{(T-\backt_j)^2} \lesssim \varepsilon
\]
Let $\tau_j= \ln \big(\frac{T-\backt_{j}}{T-\backt_{j+1}}\big)\in (0,\infty)$. Then $\frac{\gamma_j}{T-\backt_j}=1-e^{-\tau_j}$ and $\sum_{\delta<T-\backt_j<T} \tau_j = \ln (T/\delta)$ is fixed. Since $x\mapsto (1-e^{-x})^2$ is convex on the domain $x\in (0,\infty)$, according the Jensen's inequality, $\sum_{\delta<T-\backt_j<T} \frac{\gamma_{j}^2}{(T-\backt_{j})^2}$ reaches its minimum when $\tau_j$ are constant-valued for all $j$, which implies the exponential schedule is optimal to minimize $E_D$, hence optimal to minimize $E_D+E_I$.

%%%%%%%%%%%%%%%%%%%%%%%%%%%%%%%%%%%%%%%%%%%%%%%%%%%%%%%%%%%%

\end{document}